\documentclass{article}
\usepackage{fullpage}
\usepackage{comment}
\usepackage{caption}
\usepackage{subcaption}
\usepackage{xcolor}
\usepackage{multirow}
\usepackage{dsfont}
\usepackage{breakcites}
\usepackage{placeins}

\usepackage{amsmath,amssymb,amsthm,amsfonts}
\usepackage{caption}
\usepackage{paralist}
\usepackage{subcaption}
\usepackage{latexsym,bbm,graphicx,float,mathtools}

\usepackage[colorlinks,linkcolor=blue,citecolor=blue]{hyperref}
\usepackage{url}
\usepackage[capitalize,nameinlink]{cleveref}
\usepackage[hang,flushmargin]{footmisc} 

\def\eqref#1{(\ref{#1})}

\def\1{\bm{1}}

\def\eps{{\epsilon}}

\DeclareMathAlphabet{\mathsfit}{\encodingdefault}{\sfdefault}{m}{sl}
\SetMathAlphabet{\mathsfit}{bold}{\encodingdefault}{\sfdefault}{bx}{n}

\newcommand{\R}{\mathbb{R}}

\DeclareMathOperator*{\argmax}{arg\,max}

\DeclareMathOperator{\sign}{sign}

\newcommand{\Ex}{\mathop{{\bf E}\/}}
\renewcommand{\Pr}{\operatorname{{\bf Pr}}}
\newcommand{\Prx}{\mathop{{\bf Pr}\/}}

\newcommand{\Ber}{\operatorname{Ber}}

\newcommand{\N}{\mathbbm N}

\newcommand{\ind}{{\bf 1}}

\newcommand{\bA}{\mathbf{A}}

\newcommand{\bD}{\mathbf{D}}

\newcommand{\bI}{\mathbf{I}}

\newcommand{\bS}{\mathbf{S}}

\newcommand{\bW}{\mathbf{W}}
\newcommand{\bX}{\mathbf{X}}

\newcommand{\ba}{\boldsymbol{a}}

\newcommand{\bg}{\boldsymbol{g}}
\newcommand{\bh}{\boldsymbol{h}}

\newcommand{\boldr}{\boldsymbol{r}} 

\newcommand{\bw}{\boldsymbol{w}}
\newcommand{\bx}{\boldsymbol{x}}

\newcommand{\bz}{\boldsymbol{z}}

\newcommand{\bnu}{\boldsymbol{\nu}}
\newcommand{\bmu}{\boldsymbol{\mu}}

\newcommand{\calA}{\mathcal{A}}

\newcommand{\calE}{\mathcal{E}}

\newcommand{\zo}{\{0,1\}}
\newcommand{\calbE}{\boldsymbol{\calE}}
\newcommand{\CSBM}{\textsf{CSBM}}
\newcommand{\LeakyRelu}{\mathrm{LeakyRelu}}
\newcommand{\indi}{\boldsymbol{1}}
\newcommand{\eqdef}{\stackrel{\rm def}{=}}

\newcommand{\BlackBox}{\rule{1.5ex}{1.5ex}}  
\ifdefined\proof
    \renewenvironment{proof}{\par\noindent{\bf Proof\ }}{\hfill\BlackBox\\[2mm]}
\else
    \newenvironment{proof}{\par\noindent{\bf Proof\ }}{\hfill\BlackBox\\[2mm]}
\fi
 
\newtheorem{theorem}{Theorem}
\newtheorem{lemma}[theorem]{Lemma} 
\newtheorem{proposition}[theorem]{Proposition} 

\newtheorem{corollary}[theorem]{Corollary}
\newtheorem{definition}[theorem]{Definition}
\newtheorem{conjecture}[theorem]{Conjecture}

\newtheorem{assumption}{Assumption}
\newtheorem{claim}[theorem]{Claim}

\newtheorem{observation}[theorem]{Observation}

\newtheorem*{theorem*}{Theorem}
\newtheorem*{lemma*}{Lemma}
\newtheorem*{proposition*}{Proposition}
\newtheorem*{corollary*}{Corollary}

\newcommand\blfootnote[1]{%
  \begingroup
  \renewcommand\thefootnote{}\footnote{#1}%
  \addtocounter{footnote}{-1}%
  \endgroup
}
\usepackage[utf8]{inputenc}
\def\FullBox{\hbox{\vrule width 8pt height 8pt depth 0pt}}
\newcommand{\QED}{\;\;\;\FullBox}
\renewenvironment{proof}{\noindent{\bf Proof:~}}{\hfill\QED}

\makeatletter
\def\blfootnote{\gdef\@thefnmark{}\@footnotetext}
\makeatother

\title{Graph Attention Retrospective}
\author{
    Kimon Fountoulakis$^{*\dagger}$\qquad\qquad
    Amit Levi$^{*\ddagger}$\qquad\qquad
    Shenghao Yang$^{*\dagger}$\\ \\
    Aseem Baranwal$^{\dagger}$\qquad\qquad
    Aukosh Jagannath$^{\mathsection}$
}

\date{}

\begin{document}
\maketitle
\def\thefootnote{*}\footnotetext{Equal contribution.}
\def\thefootnote{$\dagger$}\footnotetext{David R. Cheriton School of Computer Science, University of Waterloo, Waterloo, Canada.}
\def\thefootnote{$\ddagger$}\footnotetext{Huawei Noah’s Ark Lab, Montreal, Canada.}
\def\thefootnote{$\mathsection$}\footnotetext{Department of Statistics and Actuarial Science, Department of Applied Mathematics, University of Waterloo, Waterloo, Canada.}
\def\thefootnote{\arabic{footnote}}
\blfootnote{Emails: \href{mailto:kimon.fountoulakis@uwaterloo.ca}{kimon.fountoulakis@uwaterloo.ca}, \href{mailto:amit.levi@huawei.com}{amit.levi@huawei.com}, \href{mailto:shenghao.yang@uwaterloo.ca}{shenghao.yang@uwaterloo.ca},\\
\href{mailto:aseem.baranwal@uwaterloo.ca}{aseem.baranwal@uwaterloo.ca}, \href{mailto:a.jagannath@uwaterloo.ca}{a.jagannath@uwaterloo.ca}}

\frenchspacing

\begin{abstract}
Graph-based learning is a rapidly growing sub-field of machine learning with applications in social networks, citation networks, and bioinformatics. One of the most popular models is graph attention networks. They were introduced to allow a node to aggregate information from features of neighbor nodes in a non-uniform way, in contrast to simple graph convolution which does not distinguish the neighbors of a node. In this paper, we theoretically study the behaviour of graph attention networks. We prove multiple results on the performance of the graph attention mechanism for the problem of node classification for a contextual stochastic block model. Here, the node features are obtained from a mixture of Gaussians and the edges from a stochastic block model. We show that in an ``easy'' regime, where the distance between the means of the Gaussians is large enough, graph attention is able to distinguish inter-class from intra-class edges. Thus it maintains the weights of important edges and significantly reduces the weights of unimportant edges. Consequently, we show that this implies perfect node classification. In the ``hard'' regime, we show that every attention mechanism fails to distinguish intra-class from inter-class edges. In addition, we show that graph attention convolution cannot (almost) perfectly classify the nodes even if intra-class edges could be separated from inter-class edges. Beyond perfect node classification, we provide a positive result on graph attention's robustness against structural noise in the graph. In particular, our robustness result implies that graph attention can be strictly better than both the simple graph convolution and the best linear classifier of node features. We evaluate our theoretical results on synthetic and real-world data.
\end{abstract}

\section{Introduction}
\label{sec:intro}
Graph learning has received a lot of attention recently due to breakthrough learning models~\cite{GMS05,scarselli:gnn,BZSL14,DMAGHAA15,HBL15,AT16,DBV16,HYL17,kipf:gcn} that are able to exploit multi-modal data that consist of nodes and their edges as well as the features of the nodes. One of the most important problems in graph learning is the problem of \emph{classification}, where the goal is to classify the nodes or edges of a graph given the graph and the features of the nodes. Two of the most fundamental mechanisms for classification, and graph learning in general, are graph convolution and graph attention. Graph convolution, usually defined using its spatial version, corresponds to averaging the features of a node with the features of its neighbors~\cite{kipf:gcn}.\footnote{Although the model in~\cite{kipf:gcn} is related to spectral convolutions, it is mainly a spatial convolution since messages are propagated along graph edges. More broadly, graph convolution can refer to different variants arising from different (approximations of) graph spectral filters. We provide details in Section~\ref{sec:prelim}.} Graph attention~\cite{Velickovic2018GraphAN} mechanisms augment this convolution by appropriately weighting the edges of a graph before spatially convolving the data. Graph attention is able to do this by using information from the given features for each node. Despite its wide adoption by practitioners~\cite{FL2019,wang2019dgl,HFZDRLCL20} and its large academic impact as well, the number of works that rigorously study its effectiveness is quite limited.

One of the motivations for using a graph attention mechanism as opposed to a simple convolution is the expectation that the attention mechanism is able to distinguish inter-class edges from intra-class edges and consequently weights inter-class edges and intra-class edges differently before performing the convolution step. This ability essentially maintains the weights of important edges and significantly reduces the weights of unimportant edges, and thus it allows graph convolution to aggregate features from a subset of neighbor nodes that would help node classification tasks. In this work, we explore the regimes in which this heuristic picture holds in simple node classification tasks, namely classifying the nodes in a contextual stochastic block model (CSBM)~\cite{BVR17,DSM18}. The CSBM is a coupling of the stochastic block model (SBM) with a Gaussian mixture model, where the features of the nodes within a class are drawn from the same component of the mixture model. For a more precise definition, see Section~\ref{sec:prelim}. We focus on the case of two classes where the answer to the above question is sufficiently precise to understand the performance of graph attention and build useful intuition about it. We briefly and informally summarize our contributions as follows:
\begin{enumerate}
    \item In the ``easy regime", i.e., when the distance between the means of the Gaussian mixture model is much larger than the standard deviation, we show that there exists a choice of attention architecture that distinguishes inter-class edges from intra-class edges with high probability (Theorem~\ref{thm:edge_separation_easy}). In particular, we show that the attention coefficients for one class of edges are much higher than the other class of edges (Corollary~\ref{cor:gamma_easy}). Furthermore, we show that these attention coefficients lead to a perfect node classification result (Corollary~\ref{cor:node_separation_easy}). However, in the same regime, we also show that the graph is not needed to perfectly classify the data (Proposition~\ref{prop:linear_easy}).
    \item In the ``hard regime", i.e.,  when the distance between the means is small compared to the standard deviation, we show that \emph{every} attention architecture is unable to distinguish inter-class from intra-class edges with high probability (Theorem~\ref{thm:edge_separation_hard}). Moreover, we show that using the original Graph Attention Network (GAT) architecture~\cite{Velickovic2018GraphAN}, with high probability, most of the attention coefficients are going to have uniform weights, similar to those of simple graph convolution~\cite{kipf:gcn} (Theorem~\ref{thm:gamma_hard}). We also consider a setting where graph attention is provided with independent and perfect edge information, and we show that even in this case, if the distance between the means is sufficiently small, then graph attention would fail to (almost) perfectly classify the nodes with high probability (Theorem~\ref{thm:good_psi_negative}).
    \item Beyond perfect node classification, we show that graph attention is able to assign a significantly higher weight to self-loops, irrespective of the parameters of the Gaussian mixture model that generates node features or the stochastic block model that generates the graph. Consequently, we show that with a high probability, graph attention convolution is at least as good as the best linear classifier for node features (Theorem~\ref{thm:gat_linear}). In a high structural noise regime when the graph is not helpful for node classification, i.e., when intra-class edge probability $p$ is close to inter-class edge probability $q$, this means that graph attention is strictly better than simple graph convolution because it is able to essentially ignore the graph structural noise (Corollary~\ref{cor:model_rank}). In addition, we obtain lower bounds on graph attention's performance for almost perfect node classification and partial node classification (Corollary~\ref{cor:gat_linear_recovery}).
    \item We provide an extensive set of experiments both on synthetic data and on three popular real-world datasets that validates our theoretical results.
\end{enumerate}

\subsection{Limitations of our theoretical setting}

In this work, to study the benefits and the limitations of the graph attention mechanism, we focus on node classification tasks using the CSBM generative model with two classes and Gaussian node features. This theoretical setting has a few limitations. First, we note that many real-world networks are much more complex and may exhibit different structures than the ones obtainable from a stochastic block model. Furthermore, real-world node attributes such as categorical features may not even follow a continuous probability distribution. Apparently, there are clear gaps between CSBM and the actual generative processes of real-world data. Nonetheless, we note that a good understanding of graph attention's behavior over CSBM would already provide us with useful insights. For example, it has been shown empirically that synthetic graph datasets based on CSBM constitute good benchmarks for evaluating various GNN architectures~\cite{graphworld2022}. To that end, in Section~\ref{sec:implications}, we summarize some practical implications of our theoretical results which could be useful for practitioners working with GNNs.

Second, there are different levels of granularity for machine learning problems on graphs. Besides node-level tasks such as node classification, other important problems include edge-level tasks such as link prediction and graph-level tasks such as graph classification. While, empirically, variants of graph attention networks have been successfully applied to solve problems at all levels of granularity, our theoretical results mostly pertain to  the effects of graph attention for node classification. This is a limitation on the scope of our study. On the other hand, we note that edge-level and graph-level tasks are typically solved by adding a pooling layer which combines node representations from previous graph (attention) convolution layers. Consequently, the quality of graph attention convolution's output for a node not only affects node classification but also has a major influence on link prediction and graph classification. Therefore, our results may be extended to study the effects of graph attention on link prediction and graph classification under the CSBM generative model. For example, link prediction is closely related to edge classification which we extensively study in Section~\ref{sec:results}. For graph classification, one may consider the problem of classifying graphs generated from different parameter settings of the CSBM. In this case, our results on graph attention's impact on node representations may be used to establish results on graph attention's impact on graph representations after certain pooling layers. In particular, if graph attention convolution generates good node representations that are indicative of node class memberships, then one can show that the graph representation obtained from appropriately aggregating node representations would be indicative of the clustering structure of the graph, and hence the graph representation would be useful for graph classification tasks.

\subsection{Relevant work}\label{subsec:related_work}
Recently the concept of attention for neural networks~\cite{BCB15,VSPUJGKP17} was transferred to graph neural networks~\cite{LZBT16,BL18,Velickovic2018GraphAN,LRKAK19,PBL20}. 
A few papers have attempted to understand the attention mechanism in~\cite{Velickovic2018GraphAN}. One work relevant to ours is~\cite{BAY21}. In this paper, the authors show that a node may fail to assign large edge weight to its most important neighbors due to a global ranking of nodes generated by the attention mechanism in~\cite{Velickovic2018GraphAN}. Another related work is~\cite{KTA19}, which presents an empirical study of the ability of graph attention to generalize on larger, complex, and noisy graphs. In addition, in~\cite{hou2019measuring} the authors propose a different metric to generate the attention coefficients and show empirically that it has an advantage over the original GAT architecture. 

Other related work to ours, which does not focus on graph attention, comes from the field of statistical learning on random data models. Random graphs and the stochastic block model have been traditionally used in clustering and community detection~\cite{Abbe2018,AFTPPVLLQS,moore2017csphysics}. Moreover, the works by~\cite{BVR17,DSM18}, which also rely on CSBM are focused on the fundamental limits of unsupervised learning. Of particular relevance is the work by~\cite{BFJ2021}, which studies the performance of graph convolution on CSBM as a semi-supervised learning problem. Within the context of random graphs, \cite{KBV21} studies the approximation power of graph neural networks on random graphs. In~\cite{MLLK22} the authors derive generalization error of graph neural networks for graph classification and regression tasks. In our paper we are interested in understanding graph attention's capability for edge/node classification with respect to different parameter regimes of CSBM.

Finally, there are a few related theoretical works on understanding the generalization and representation power of graph neural networks~\cite{CLB19,Chien:2020:joint,Zhu:2020:generalizing,XHLJ19,GJJ20,A2020,ALoukas2020}. For a recent survey in this direction see \cite{J22}. Our work takes a statistical perspective which allows us to characterize the precise performance of graph attention compared to graph convolution and no convolution for CSBM, with the goal of answering the particular questions that we imposed above.

\section{Preliminaries}\label{sec:prelim}

In this section, we describe the \emph{Contextual Stochastic Block Model (CSBM)}~\cite{DSM18} which serves as our data model, and the \emph{Graph Attention} mechanism~\cite{Velickovic2018GraphAN}. We also define notations and terms that are frequently used throughout the paper.

For a vector $x \in \R^n$ and $n \in \N$, the norm $\|x\|$ denotes the Euclidean norm of $x$, i.e. $\|x\| \eqdef \sum_{i \in [n]}x_i^2$. We write $[n] \eqdef \{1,2,\ldots,n\}$. We use $\Ber(p)$ to denote the Bernoulli distribution, so $x \sim \Ber(p)$ means the random variable $x$ takes value 1 with probability $p$ and 0 with probability $1-p$. Let $d, n \in \N$, and $\eps_1,\ldots,\eps_n\sim \Ber(1/2)$. Define two classes as $C_k=\{j\in [n] \mid \eps_j=k\}$ for $k\in \zo$. For each index $i\in [n]$, we set the feature vector $\bX_i\in \R^d$ as $\bX_i\sim N((2\eps_i-1)\bmu, \sigma^2\bI)$, where $\bmu\in \R^d$, $\sigma\in \R$ and $\bI\in \zo^{d\times d}$ is the identity matrix.\footnote{The means of the mixture of Gaussians are $\pm \bmu$. Our results can be easily generalized to general means. The current setting makes our analysis simpler without loss of generality.} For a given pair $p,q\in [0,1]$ we consider the stochastic adjacency matrix $\bA\in\zo^{n\times n}$ defined as follows. For $i,j\in [n]$ in the same class (i.e., \emph{intra-class edge}), we set $a_{ij}\sim \Ber(p)$, and if $i,j$ are in different classes (i.e., \emph{inter-class edge}), we set $a_{ij}\sim\Ber(q)$. We denote by $(\bX,\bA)\sim \CSBM(n,p,q,\bmu,\sigma^2)$ a sample obtained according to the above random process. 

An advantage of CSBM is that it allows us to control the noise by controlling the parameters of the distributions of the model. In particular, CSBM allows us to control the distance of the means and the variance of the Gaussians, which are important for controlling the separability of the Gaussians. For example, fixing the variance, then the closer the means are the more difficult the separation of the node Gaussian features becomes. Another notable advantage of CSBM is that it allows us to control the structural noise and homophily level of the graph. The level of noise in the graph is controlled by increasing or reducing the gap between the intra-class edge probability $p$ and the inter-class edge probability $q$, and the level of homophily in the graph is controlled by the relative magnitude between $p$ and $q$. For example, when $p$ is much larger than $q$, then a node is more likely to be connected with a node from the same class, and hence we obtain a homophilous graph; when $q$ is much larger than $p$, then a node is more likely to be connected with a node from a different class, and hence we obtain a heterophilous graph. There are several recent works exploring the behavior of GNNs in heterophilous graphs~\cite{lim2021large,yan2022two,bodnar2022neural,luan2022revisiting}. Interestingly, we note that our results for graph attention's behavior over the CSBM data model do not depend on whether the graph is homophilous or heterophilous.

Given node representations $\bh_i \in \R^{F'}$ for $i \in [n]$, a (spatial) graph convolution for node $i$ has output
\[
    \bh_i' = \sum_{j \in [n]}\bA_{ij}c_{ij}\bW\bh_j, \quad c_{ij} = \bigg(\sum_{\ell \in [n]} \bA_{i\ell}\bigg)^{-1},
\]
where $\bW \in \R^{F \times F'}$ is a learnable matrix. Throughout this paper, we will refer to this operation as {\em simple graph convolution} or {\em standard graph convolution}. Our definition of graph convolution is essentially the mean aggregation step in a general GNN layer~\cite{HYL17}. The normalization constant $c_{ij}$ in our definition is closely related to the symmetric normalization $c_{ij} = (\sum_{\ell \in [n]}\bA_{i\ell})^{-1/2}(\sum_{\ell \in [n]}\bA_{j\ell})^{-1/2}$ used in the original Graph Convolutional Network (GCN) introduced by \cite{kipf:gcn}. Our definition does not affect the discussions or implications we have for GCN with symmetric normalization. More broadly, there are other forms of graph convolution in the literature~\cite{bronstein2021geometric,DBV16,levie2018cayleynets}, which we do not compare within this work.

A \emph{single-head} graph attention applies some weight function on the edges based on their node features (or a mapping thereof). Given two representations $\bh_i,\bh_j\in \R^{F'}$ for two nodes $i, j\in[n]$, the \emph{attention model/mechanism} is defined as the mapping
\[
	\Psi(\bh_i,\bh_j)\eqdef\alpha(\bW\bh_i,\bW\bh_j)
\]
where $\alpha:\R^F\times \R^F\to \R$ and $\bW\in \R^{F\times F'}$ is a learnable matrix. The \emph{attention coefficient} for a node $i$ and its neighbor $j$ is defined as
\begin{align}\label{def:attention_coeff}
    \gamma_{ij}\eqdef \frac{\exp(\Psi(\bh_i,\bh_j))}{\sum_{\ell\in N_i}\exp(\Psi(\bh_i,\bh_\ell))},
\end{align}
where $N_i$ is the set of neighbors of node $i$ that also includes node $i$ itself. Let $f$ be some element-wise activation function (which is usually nonlinear), the graph attention convolution output for a node $i \in [n]$ is given by
\begin{equation}\label{eq:gat_output}
\begin{split}
{\bh}'_i &=\sum_{j\in [n]}\bA_{ij}\gamma_{ij}\bW \bh_j,\\
\tilde\bh_i&= f(\bh'_i).
\end{split}
\end{equation}
A \emph{multi-head} graph attention~\cite{Velickovic2018GraphAN} uses $K\in \N$ weight matrices $\bW^1$, $\ldots$, $\bW^K\in \R^{F\times F'}$ and averages their individual (single-head) outputs.
We consider the most simplified case of a single graph attention layer (i.e., $F'=d$ and $F=1$) where $\alpha$ is realized by an MLP using the LeakyRelu activation function. The LeakyRelu activation function is defined as $\LeakyRelu(x) = x$ if $x \ge 0$ and $\LeakyRelu(x) = \beta x$ for some constant $\beta\in[0,1)$ if $x < 0$.

The CSBM model induces node features $\bX$ which are correlated through the graph $G=([n],E)$, represented by an adjacency matrix $\bA$. A natural requirement of attention architectures is to maintain important edges in the graph and ignore unimportant edges. For example, important edges could be the set of intra-class edges and unimportant edges could be the set of inter-class edges. In this case, if graph attention mains all intra-class edges and ignores all inter-class edges, then a node from a class will be connected only to nodes from its own class. More specifically, a node $v$ will be connected to neighbor nodes whose associated node features come from the \emph{same distribution} as node features of $v$. Given two sets $A$ and $B$, we denote $A \times B \eqdef \{(i,j) : i \in A, j \in B\}$ and $A^2 \eqdef A \times A$. 

\section{Separation of edges and its implications for graph attention's ability for perfect node classification}\label{sec:results}

In this section, we study the ability of graph attention to separate important edges from unimportant edges. In addition, we study the implications of graph attention's behavior for node classification. In particular, we are interested in whether or not the graph attention convolution is able to perfectly classify the nodes in the graph. Depending on the parameter regimes of the CSBM, we separate the results into two parts. In Section~\ref{subsec:easy}, we study graph attention's behavior when the distance between the means of the node features is large. We construct a specific attention architecture and demonstrate its ability to separate intra-class edges from inter-class edges. Consequently, we show that the corresponding graph attention convolution is able to perfectly classify the nodes. In Section~\ref{subsec:hard}, we study graph attention's behavior when the distance between the means of the node features is small. We show that with high probability no attention architecture is able to separate intra-class edges from inter-class edges. Consequently, we conjecture that no graph attention is able to perfectly classify the nodes. We provide evidence of this conjecture in Section~\ref{subsec:indep_edge}, where we assume the existence of a strong attention mechanism that is able to perfectly classify the edges independently from node features. We show that using this ``idealistic'' attention mechanism still fails to (almost) perfectly classify the nodes, provided that the distance between the means of the node features is sufficiently small.

The two-parameter regimes of the CSBM that we consider in Section~\ref{subsec:easy} and Section~\ref{subsec:hard} are as follows. The first (``easy regime") is where $\|\bmu\|=\omega(\sigma\sqrt{\log n})$, and the second (``hard regime") is where $\|\bmu\|= \kappa\sigma$ for some $0<\kappa\le O(\sqrt{\log n})$. We start by defining edge separability and node separability.

\begin{definition}[Edge separability]\label{def:sep_edge}
Given an attention model $\Psi$, we say that the model separates the edges, if the outputs satisfy
\[
	\sign(\Psi(\bX_i,\bX_j))= -\sign(\Psi(\bX_k,\bX_\ell))
\]
whenever $(i,j)$ is an intra-class edge, i.e. $(i,j) \in (C_{1}^2\cup C_0^2)\cap E$, and $(k,\ell)$ is an inter-class edge, i.e. $(k,\ell) \in  E\setminus (C_{1}^2\cup C_0^2)$.
\end{definition}

\begin{definition}[Node separability]\label{def:sep_node}
Given a classification model which outputs a scalar representation $h'_i$ for node $i$, we say that the model separates the nodes if $h'_i>0$ whenever $i\in C_1$ and $h'_i<0$ whenever $i\in C_0$.
\end{definition}

Some of our results in this section rely on a mild assumption that lower bounds the sparsity of the graph generated by the CSBM model. This assumption says that the expected degree of a node in the graph is larger than $\log^2 n$. It is required to obtain a concentration of node degrees property used in the proofs. While this assumption may limit a direct application of our results to very sparse graphs, for example, graphs whose node degrees are bounded by a constant, it is mainly an artifact of the proof techniques that we use. We expect that similar results still hold for sparser graphs, but to rigorously remove the assumption on graph density one may have to adopt a different proof technique.

\begin{assumption}\label{ass:p_q}
$p,q=\Omega(\log^2 n /n)$.
\end{assumption}

\subsection{``Easy Regime"}\label{subsec:easy}

In this regime $\left(\|\bmu\|=\omega(\sigma\sqrt{\log n})\right)$ we show that a two-layer MLP attention is able to separate the edges with high probability. Consequently, we show that the corresponding graph attention convolution is able to separate the nodes with high probability. At a high level, we transform the problem of classifying an edge $(i,j) \in E$ into the problem of classifying a point $[\tilde \bw^T\bX_i, \tilde \bw^T\bX_j]$ in $\R^2$, where $\tilde \bw = \bmu/\|\bmu\|$ is a unit vector that maximizes the total pairwise distances among the four means given below. When we consider the set of points $[\tilde\bw^T\bX_i,\tilde\bw^T\bX_j]$ for all $(i,j) \in E$, we can think of each point as a two-dimensional Gaussian vector whose mean is one of the following: $[\tilde\bw^T\bmu,\tilde\bw^T\bmu]$, $[-\tilde\bw^T\bmu,\tilde\bw^T\bmu]$, $[\tilde\bw^T\bmu,-\tilde\bw^T\bmu]$, $[-\tilde\bw^T\bmu,-\tilde\bw^T\bmu]$. The set of intra-class edges corresponds to the set of bivariate Gaussian vectors whose mean is either $[\tilde\bw^T\bmu,\tilde\bw^T\bmu]$ or $[-\tilde\bw^T\bmu,-\tilde\bw^T\bmu]$, while the set of inter-class edges corresponds to the set of bivariate Gaussian vectors whose mean is either $[-\tilde\bw^T\bmu,\tilde\bw^T\bmu]$ or $[\tilde\bw^T\bmu,-\tilde\bw^T\bmu]$. Therefore, in order to correctly classify the edges, we need to correctly classify the data corresponding to means $[\tilde\bw^T\bmu,\tilde\bw^T\bmu]$ and $[-\tilde\bw^T\bmu,-\tilde\bw^T\bmu]$ as 1, and classify the data corresponding to the other means as $-1$. This problem is known in the literature as the \emph{``XOR problem"}~\cite{minsky1969perceptron}. To achieve this we consider a two-layer MLP architecture $\Psi$ which separates the first and the third quadrants of the two-dimensional space from the second and the fourth quadrants. In particular, we consider the following specification of $\Psi(\bX_i,\bX_j)$,
\begin{equation}\label{eq:psi_ansatz}
\Psi(\bX_i, \bX_j) \eqdef \boldr^T \LeakyRelu\left(\bS \begin{bmatrix} \tilde \bw^T \bX_i\\ \tilde \bw^T \bX_j\end{bmatrix} \right),
\end{equation}
where
\begin{equation}\label{eq:psi_ansatz_parameters}
\tilde \bw \eqdef \frac{\bmu}{\|\bmu\|},
\quad
\bS\eqdef  
\begin{bmatrix}  
\hspace{3mm}1 & \hspace{3mm}1  \\
-1 & -1 \\
\hspace{3mm}1 & -1\\
-1 & \hspace{3mm}1
\end{bmatrix},
\quad 
\boldr\eqdef R \cdot \begin{bmatrix} \hspace{3mm}1 \\ \hspace{3mm}1 \\ -1 \\ -1 \end{bmatrix},
\end{equation}
where $R > 0$ is an arbitrary scaling parameter. The particular function $\Psi$ has been chosen such that it is able to classify the means of the XOR problem correctly, that is,
\begin{align*}
\sign(\Psi(\Ex[\bX_i], \Ex[\bX_j])) = \left\{\begin{array}{ll} \hspace{3mm}1, & \mbox{if $(i,j)$ is an intra-class edge}, \\ -1, & \mbox{if $(i,j)$ is an inter-class edge}. \end{array}\right.
\end{align*}
To see this, we can take look at the following simplified expression of $\Psi$, which is obtained by substituting the specifications of $\bS$ and $\boldr$ from \eqref{eq:psi_ansatz_parameters} to \eqref{eq:psi_ansatz},
\begin{equation}\label{eq:psi_simplified}
    \Psi(\bX_i,\bX_j) =
    \left\{
    \begin{array}{ll}
    -2R(1-\beta)\tilde\bw^T\bX_i, & \mbox{if}~ \tilde\bw^T\bX_j \le -\left|\tilde\bw^T\bX_i\right|,\\
    \hspace{3mm}2R(1-\beta)\sign(\tilde\bw^T\bX_i)\tilde\bw^T\bX_j, & \mbox{if}~ -\left|\tilde\bw^T\bX_i\right| < \tilde\bw^T\bX_j < \left|\tilde\bw^T\bX_i\right|,\\
    \hspace{3mm}2R(1-\beta)\tilde\bw^T\bX_i, & \mbox{if}~ \tilde\bw^T\bX_j \ge \left|\tilde\bw^T\bX_i\right|.
    \end{array}
    \right.
\end{equation}
Then, by substituting $\tilde\bw = \bmu/\|\bmu\|$ into \eqref{eq:psi_simplified}, one easily verifies that
\begin{equation}\label{eq:psi_expect}
    \Psi(\Ex[\bX_i], \Ex[\bX_j]) =  \left\{\begin{array}{ll} \hspace{3mm}2R(1-\beta)\|\bmu\|, & \mbox{if $(i,j)$ is an intra-class edge}, \\ -2R(1-\beta)\|\bmu\|, & \mbox{if $(i,j)$ is an inter-class edge}. \end{array}\right.
\end{equation}

On the other hand, our assumption that $\|\bmu\|=\omega(\sigma \sqrt{\log n})$ guarantees that the distance between the means of the XOR problem is much larger than the standard deviation of the Gaussians, and thus with high probability there is no overlap between the distributions. More specifically, one can show that with a high probability,
\begin{equation}\label{eq:psi_concentration}
     \Psi(\bX_i, \bX_j) =  \left\{\begin{array}{ll} \hspace{3mm}2R(1-\beta)\|\bmu\|(1\pm o(1)), & \mbox{if $(i,j)$ is an intra-class edge}, \\ -2R(1-\beta)\|\bmu\|(1\pm o(1)), & \mbox{if $(i,j)$ is an inter-class edge}. \end{array}\right.
\end{equation}
This property guarantees that with high probability,
\[
    \sign(\Psi(\bX_i,\bX_j)) = \sign(\Psi(\Ex[\bX_i], \Ex[\bX_j])), \ \mbox{for all} \ (i,j) \in E,
\]
which implies perfect separability of the edges. We state this result below in \cref{thm:edge_separation_easy} and provide detailed arguments in the appendix.

\begin{theorem}\label{thm:edge_separation_easy} 
Suppose that $\|\bmu\|=\omega(\sigma \sqrt{\log n})$. Then with probability at least $1-o(1)$ over the data $(\bX,\bA) \sim \CSBM(n,p,q,\bmu,\sigma^2)$, the two-layer MLP attention architecture $\Psi$ given in \eqref{eq:psi_ansatz} and \eqref{eq:psi_ansatz_parameters} separates intra-class edges from inter-class edges.
\end{theorem}

Our analysis of edge separability by the attention architecture $\Psi$ has two important implications. First, edge separability by $\Psi$ implies a nice concentration result for the attention coefficients $\gamma_{ij}$ defined in \eqref{def:attention_coeff}. In particular, one can show that the attention coefficients have the desirable property to maintain important edges and drop unimportant edges. Second, the nice distinguishing power of the attention coefficients in turn implies the exact recoverability of the node memberships using the graph attention convolution. We state these two results below in Corollary~\ref{cor:gamma_easy} and Corollary~\ref{cor:node_separation_easy}, respectively. Denote
\[
    \Psi' \eqdef (\mathbf{1}_{p\ge q} - \mathbf{1}_{p < q}) \cdot \Psi,
\]
that is, $\Psi' = \Psi$ if $p \ge q$ and $\Psi' = -\Psi$ if $p < q$, where $\Psi$ is the two-layer MLP attention architecture given in \eqref{eq:psi_ansatz} and \eqref{eq:psi_ansatz_parameters}. As it will become clear later, the attention architecture $\Psi'$ allows attention coefficients to assign more weights to either intra-class or inter-class edges, depending on if $p \ge q$ or $p < q$, and this will help us obtain a perfect separation of the nodes.

\begin{corollary}\label{cor:gamma_easy} 
Suppose that $\|\bmu\|=\omega(\sigma \sqrt{\log n})$ and that Assumption~\ref{ass:p_q} holds. Then with probability at least $1-o(1)$ over the data $(\bX,\bA) \sim \CSBM(n,p,q,\bmu,\sigma^2)$, the attention architecture $\Psi'$ yields attention coefficients $\gamma_{ij}$ such that
\begin{enumerate}
\item If $p \ge q$, then $\gamma_{ij}=\frac{2}{np}(1\pm o(1))$ if $(i,j)$ is an intra-class edge and $\gamma_{ij}=o(\frac{1}{n(p+q)})$ otherwise;
\item If $p < q$, then $\gamma_{ij}=\frac{2}{nq}(1\pm o(1))$ if $(i,j)$ is an inter-class edge and $\gamma_{ij}=o(\frac{1}{n(p+q)})$ otherwise.
\end{enumerate}
\end{corollary}

\begin{corollary}\label{cor:node_separation_easy}
Suppose that $\|\bmu\|=\omega(\sigma \sqrt{\log n})$ and that Assumption~\ref{ass:p_q} holds. Then with probability at least $1-o(1)$ over the data $(\bX,\bA) \sim \CSBM(n,p,q,\bmu,\sigma^2)$, using the attention architecture $\Psi'$ with the graph attention convolution given in \eqref{eq:gat_output}, where $f$ is set to be the identify function, the model separates the nodes.
\end{corollary}

Corollary~\ref{cor:gamma_easy} shows the desired behavior of the attention mechanism, namely it is able to assign significantly large weights to important edges while it drops unimportant edges. When $p\ge q$, the attention mechanism maintains intra-class edges and essentially ignores all inter-class edges; when $p < q$, it maintains inter-class edges and essentially ignores all intra-class edges. We now explain the high-level idea of the proof and leave detailed arguments to Appendix~\ref{subsec:cor4proof}. Corollary~\ref{cor:gamma_easy} follows from the concentration of $\Psi(\bX_i,\bX_j)$ around $\Psi(\Ex[\bX_i],\Ex[\bX_j])$ given by \eqref{eq:psi_concentration}. Assume for a moment that $p \ge q$ so that $\mathbf{1}_{p\ge q} - \mathbf{1}_{p < q} = 1$. Then $\Psi' = \Psi$, and \eqref{eq:psi_concentration} implies that the value of $\exp(\Psi'(\bX_i,\bX_j))$ when $(i,j)$ is an intra-class edge is exponentially larger than the value of $\exp(\Psi'(\bX_i,\bX_j))$ when $(i,j)$ is an inter-class edge. Therefore, by the definition of the attention coefficients in \eqref{def:attention_coeff}, the denominator of $\gamma_{ij}$ is dominated by terms $(i,k)$ where $k$ is in the same class as $i$. Moreover, using concentration of node degrees guaranteed by Assumption~\ref{ass:p_q}, each node $i$ is connected to $\frac{2}{np}(1\pm o(1))$ many intra-class nodes. By appropriately setting the scaling parameter $R$ in \eqref{eq:psi_ansatz_parameters}, the values of $\Psi'(\bX_i,\bX_k)$ for all intra-class edges $(i,k)$ can be made within $o(1)$ multiplicative factor from each other. Therefore we get $\gamma_{ij} = \frac{2}{np}(1\pm o(1))$ when $(i,j)$ is an intra-class edge. A similar reasoning applies to inter-class edges and yields the vanishing value of $\gamma_{ij}$ when $(i,j)$ is an inter-class edge. Finally, the argument for $p < q$ follows analogously.

The concentration result of attention coefficients in Corollary~\ref{cor:gamma_easy} implies the node classification result in Corollary~\ref{cor:node_separation_easy}, which holds for any $p,q$ satisfying Assumption~\ref{ass:p_q}. That is, even when the graph structure is noisy, e.g., when $p \approx q$, it is still possible to exactly recover the node class memberships. Essentially, by applying Corollary~\ref{cor:gamma_easy} and carrying out some straightforward algebraic simplifications, one can show that with high probability for all $i \in [n]$,
\[
    h'_i = \sum_{j \in N_i}\gamma_{ij}\tilde\bw^T\bX_j = \left\{\begin{array}{ll}-(\mathbf{1}_{p\ge q} - \mathbf{1}_{p < q})\|\bmu\|(1\pm o(1)), & \mbox{if $i \in C_0$}, \\ \hspace{3mm}(\mathbf{1}_{p\ge q} - \mathbf{1}_{p < q})\|\bmu\|(1 \pm o(1)), & \mbox{if $i \in C_1$},\end{array}\right.
\]
and hence the model separates the nodes with high probability. We provide details in the appendix. 

While Corollary~\ref{cor:node_separation_easy} provides a positive result for graph attention, it can be shown that a simple linear classifier which does not use the graph at all achieves perfect node separability with high probability. In particular, the Bayes optimal classifier for the node features without the graph is able to separate the nodes with high probability. This means that in this regime, using the additional graph information is unnecessary, as it does not provide additional power compared to a simple linear classifier for the node classification task.

\begin{lemma}[Section~6.4 in \cite{anderson1962introduction}]\label{lem:bayes} 
Let $(\bX,\bA)\sim\CSBM(n,p,q,\bmu,\sigma^2)$. Then the optimal Bayes classifier for $\bX$ is realized by the linear classifier
\begin{equation}\label{eq:bayes_classifier}
h(\bX_i)= \begin{cases}
0 & \text{if } \bmu^T\bX_i\le 0\\
1 & \text{if } \bmu^T\bX_i>0
\end{cases}.
\end{equation}
\end{lemma}

\begin{proposition}\label{prop:linear_easy} 
Suppose $\|\bmu\|=\omega(\sigma \sqrt{\log n})$. Then with probability at least $1-o(1)$ over the data $(\bX,\bA) \sim \CSBM(n,p,q,\bmu,\sigma^2)$, the linear classifier given in \eqref{eq:bayes_classifier} separates the nodes.
\end{proposition}

The proof of Proposition~\ref{prop:linear_easy} is elementary. To see the claim one may show that the probability that the classifier in \eqref{eq:bayes_classifier} misclassifies a node $i \in [n]$ is $o(1)$. To do this, let us fix $i\in [n]$ and write $\bX_i=(2\eps_i-1)\bmu+\sigma \bg_i$ where $\bg_i\sim N(0,\bI)$. Assume for a moment $\eps_i=0$. Then the probability of misclassification is 
\[
	\Prx\left[\bmu^T\bX_i> 0\right]=\Prx\left[\frac{\bmu^T\bg_i}{\|\bmu\|}> \frac{\|\bmu\|}{\sigma}\right]=1-\Phi\left(\frac{\|\bmu\|}{\sigma}\right),
\]
where $\Phi(\cdot)$ is the cumulative distribution function of $N(0,1)$ and the last equality follows from the fact that $\frac{\bmu^T\bg_i}{\|\bmu\|}\sim N(0,1)$. The assumption that $\|\bmu\|=\omega(\sigma \sqrt{\log n})$ implies $\|\bmu\| \ge \sigma\sqrt{2\log n}$ for large enough $n$. Therefore, using standard tail bounds for normal distribution~\cite{vershynin2018high} we have that
\[
1-\Phi\left(\frac{\|\bmu\|}{\sigma}\right)\le \frac{\sigma}{\sqrt{2\pi}\|\bmu\|}\exp\left(-\frac{\|\bmu\|^2}{2\sigma^2}\right)\le \frac{n^{-1}}{\sqrt{4\pi\log n}}.
\]
This means that the probability that there exists $i\in C_0$ which is misclassified is at most $\frac{1}{2\sqrt{4\pi\log n}}=o(1)$. A similar argument can be applied to the case where $\eps_i=1$, and an application of a union bound on the events that there is $i\in [n]$ which is misclassified finishes the proof of Proposition~\ref{prop:linear_easy}.

\subsection{``Hard Regime"}\label{subsec:hard}

In this regime ($\|\bmu\|=\kappa\sigma$ for $\kappa \le O(\sqrt{\log n})$), we show that {\em every} attention architecture $\Psi$ fails to separate the edges if $\kappa < \sqrt{2\log n}$, and we conjecture that no graph attention convolution is able to separate the nodes. The conjecture is based on our node separability result in Section~\ref{subsec:indep_edge} which says that, even if we assume that there is an attention architecture which is able to separate the edges independently from node features, the corresponding graph attention convolution still fails to (almost) perfectly classify the nodes with high probability.

The goal of the attention mechanism is to decide whether an edge $(i,j)$ is an inter-class edge or an intra-class edge based on the node feature vectors $\bX_i$ and $\bX_j$. Let $\bX'_{ij}$ denote the vector obtained from concatenating $\bX_i$ and $\bX_j$, that is, 
\begin{equation}\label{eq:edge_feats}
	\bX'_{ij} \eqdef \begin{pmatrix} \bX_i \\ \bX_j \end{pmatrix}.
\end{equation}
We would like to analyze every classifier $h'$ which takes as input $\bX'_{ij}$ and tries to separate inter-class edges and intra-class edges. An ideal classifier would have the property
\begin{equation}\label{eq:ideal_edge_classifier}
	y = h'(\bX'_{ij}) = \left\{ \begin{array}{ll} 0, & \mbox{if $(i,j)$ is an inter-class edge}, \\ 1, & \mbox{if $(i,j)$ is an intra-class edge}. \end{array}\right.
\end{equation}
To understand the limitations of all such classifiers in this regime, it suffices to consider the Bayes optimal classifier for this data model, whose probability of misclassifying of an arbitrary edge lower bounds that of every attention architecture which takes as input $(\bX_i,\bX_j)$. Consequently, by deriving a misclassification rate for the Bayes classifier, we obtain a lower bound on the misclassification rate for every attention mechanism $\Psi$ for classifying intra-class and inter-class edges. The following Lemma~\ref{lem:bayes-pairs} describes the Bayes classifier for this classification task.

\begin{lemma}\label{lem:bayes-pairs}
Let $(\bX, \bA) \sim \CSBM(n, p, q, \bmu, \sigma^2)$ and let $\bX'_{ij}$ be defined as in \eqref{eq:edge_feats}. The Bayes optimal classifier for $\bX'_{ij}$ is realized by the following function,
\begin{equation}\label{eq:bayes-pairs}
h^*(\bx)= \left\{
\begin{array}{ll}
0, & \text{if} \ p\cosh\left({\frac{\bx^T\bmu'}{\sigma^2}}\right) \le q\cosh\left({\frac{\bx^T\bnu'}{\sigma^2}}\right), \\
1, & \text{otherwise},
\end{array}
\right.
\end{equation}
where $\bmu' \eqdef \begin{pmatrix} \bmu \\ \bmu \end{pmatrix}$ and $\bnu' \eqdef \begin{pmatrix} \bmu \\ -\bmu \end{pmatrix}$.
\end{lemma}

Using Lemma~\ref{lem:bayes-pairs}, we can lower bound the rate of misclassification of edges that every attention mechanism $\Psi$ exhibits. Below we define $\Phi_{\mathrm{c}}\eqdef 1-\Phi$, where $\Phi$ is the cumulative distribution function of $N(0,1)$.

\begin{theorem}\label{thm:edge_separation_hard}
Suppose $\|\bmu\|= \kappa\sigma$ for some $\kappa>0$ and let $\Psi$ be any attention mechanism. Then, 
\begin{enumerate}
    \item With probability at least $1-o(1)$, $\Psi$ fails to correctly classify at least  $2\cdot\Phi_{\mathrm{c}}(\kappa)^2$ fraction of inter-class edges;
    \item For any $K>0$ if $q>\frac{K\log^2n}{n\Phi_{\rm c}(\kappa)^2}$, then with probability at least $1-O(n^{-8K \Phi_{\rm c}(\kappa)^2\log n})$,  $\Psi$ misclassify at least one inter-class edge.
\end{enumerate}
\end{theorem}

Part 1 of Theorem~\ref{thm:edge_separation_hard} implies that if $\|\bmu\|$ is \emph{linear} in the standard deviation $\sigma$, that is if $\kappa = O(1)$, then with overwhelming probability the attention mechanism fails to distinguish a constant fraction of inter-class edges from intra-class edges. Furthermore, part 2 of Theorem~\ref{thm:edge_separation_hard} characterizes a regime for the inter-class edge probability $q$ where the attention mechanism fails to distinguish at least one inter-class edge. It provides a lower bound on $q$ in terms of the scale at which the distance between the means grows compared to the standard deviation $\sigma$. This aligns with the intuition that as we increase the distance between the means, it gets easier for the attention mechanism to correctly distinguish inter-class and intra-class edges. However, if $q$ is also increased with the right proportion, in other words, if the noise in the graph is increased, then the attention mechanism would still fail to correctly distinguish at least one inter-class edge. For instance, for $\kappa=\sqrt{2\log \log n}$ and $K=\log^2 n$, we get that if $q>\Omega(\frac{\log^{6+o(1)}n}{n})$, then with probability at least $1-o(1)$, $\Psi$ misclassifies at least an inter-class edge.

The proof of Theorem~\ref{thm:edge_separation_hard} relies on analyzing the behavior of the Bayes optimal classifier in \eqref{eq:bayes-pairs}. We compute an upper bound on the probability with which the optimal classifier correctly classifies an arbitrary inter-class edge. Then the proof of part 1 of Theorem~\ref{thm:edge_separation_hard} follows from a concentration argument for the fraction of inter-class edges that are misclassified by the optimal classifier. For part 2, we use a similar concentration argument to choose a suitable threshold for $q$ that forces the optimal classifier to fail on at least one inter-class edge. We provide formal arguments in the appendix.

As a motivating example of how the attention mechanism would fail and what exactly the attention coefficients would behave in this regime, we focus on one of the most popular attention architecture~\cite{Velickovic2018GraphAN}, where $\alpha$ is a single-layer neural network parameterized by $(\bw,\ba,b)\in \R^d\times \R^{2}\times \R$ with $\LeakyRelu$ activation function. Namely, the attention coefficients are defined by
\begin{equation}
\label{eq:softmaxattention}
\gamma_{ij}\eqdef\frac{\exp\left(\LeakyRelu\left({\ba}^T \begin{bmatrix}
\bw^T\bX_i\\
\bw^T\bX_j
\end{bmatrix}+b\right)\right)}{\sum_{\ell\in N_i}\exp\left(\LeakyRelu\left({\ba}^T \begin{bmatrix}
\bw^T\bX_i\\
\bw^T\bX_\ell
\end{bmatrix}+b\right)\right)}.
\end{equation}
We show that, as a consequence of the inability of the attention mechanism to distinguish intra-class and inter-class edges, with overwhelming probability most of the attention coefficients $\gamma_{ij}$ given by \eqref{eq:softmaxattention} are going to be $\Theta(1/|N_i|)$. In particular, Theorem~\ref{thm:gamma_hard} says that for the vast majority of nodes in the graph, the attention coefficients on most edges are uniform irrespective of whether the edge is inter-class or intra-class. As a result, this means that the attention mechanism is unable to assign higher weights to important edges and lower weights to unimportant edges.

\begin{theorem}\label{thm:gamma_hard}
Assume that $\|\bmu\| \le K\sigma$ and $\sigma \le K'$ for some absolute constants $K$ and $K'$. Moreover, assume that the parameters $(\bw, \ba, b) \in \R^d \times \R^2 \times \R$ are bounded. Then, with probability at least $1-o(1)$ over the data $(\bX,\bA) \sim \CSBM(n,p,q,\bmu,\sigma^2)$, there exists a subset $\calA \subseteq [n]$ with cardinality at least $n(1-o(1))$ such that for all $i \in \calA$ the following hold:
\begin{enumerate}
\item There is a subset $J_{i,0} \subseteq N_i \cap C_0$ with cardinality at least $\frac{9}{10}|N_i \cap C_0|$, such that $\gamma_{ij} = \Theta(1/|N_i|)$ for all $j \in J_{i,0}$.
\item There is a subset $J_{i,1} \subseteq N_i \cap C_1$ with cardinality at least $\frac{9}{10}|N_i \cap C_1|$, such that $\gamma_{ij} = \Theta(1/|N_i|)$ for all $j \in J_{i,1}$.
\end{enumerate}
\end{theorem}

Theorem~\ref{thm:gamma_hard} is proved by carefully computing the numerator and the denominator in \eqref{eq:softmaxattention}. In this regime, $\|\bmu\|$ is not much larger than $\sigma$, that is, the signal does not dominate noise, so the numerator in \eqref{eq:softmaxattention} is not indicative of the class memberships of nodes $i,j$ but rather acts like Gaussian noise. On the other hand, denote the denominator in \eqref{eq:softmaxattention} by $\delta_i$ and observe that it is the same for all $\gamma_{i\ell}$ where $\ell \in N_i$. Using concentration arguments about $\{\bw^T\bX_\ell\}_{\ell\in[n]}$ yields $\gamma_{ij} = \Theta(1/\delta_i)$ and $\delta_i = \Theta(|N_i|)$ finishes up the proof. We provide details in the appendix.

Compared to the easy regime, it is difficult to obtain a separation result for the nodes without additional assumptions. In the easy regime, the distance between the means was much larger than the standard deviation, which made the ``signal" (the expectation of the convolved data) dominate the ``noise" (i.e., the variance of the convolved data). In the hard regime the ``noise" dominates the ``signal". Thus, we conjecture the following.

\begin{conjecture}\label{conj:node_c_hard} 
There is an absolute constant $M > 0$ such that, whenever $\|\bmu\| \le M \cdot \sigma \sqrt{\frac{\log n}{n(p+q)}(1-\max(p,q))} \cdot \frac{p+q}{|p-q|}$, every graph attention model fails to perfectly classify the nodes with high probability.
\end{conjecture}

The above conjecture means that in the hard regime, the performance of the graph attention model depends on $q$ as opposed to the easy regime, where in Corollary~\ref{cor:node_separation_easy} we show that it doesn't. This property is verified by our synthetic experiments in Section~\ref{sec:experimets}. The quantity $\sigma\sqrt{\frac{\log n}{n(p+q)}(1-\max(p,q))}$ in the threshold comes from our conjecture that the expected maximum ``noise'' of the graph attention convolved data over the nodes is at least $c \sigma \sqrt{\frac{\log n}{n(p+q)}(1-\max(p,q))}$ for some constant $c>0$. The quantity $\frac{p+q}{|p-q|}$ in the threshold comes from our conjecture that the distance between the means (i.e. ``signal'') of the graph attention convolved data is reduced to at most $|p-q|/(p+q)$ of the original distance. Proving Conjecture~\ref{conj:node_c_hard} would require delicate treatment of the correlations between the attention coefficients $\gamma_{ij}$ and the node features $\bX_i$ for $i \in [n]$.

\subsubsection{Are good attention coefficients helpful in the ``hard regime''?}\label{subsec:indep_edge}

In this subsection we are interested in understanding the implications of edge separability on node separability in the hard regime and when $\Psi$ is restricted to a specific class of functions. In particular, we show that Conjecture~\ref{conj:node_c_hard} is true under an additional assumption that $\Psi$ does not depend on the node features. In addition, we show that even if we were allowed to use an ``extremely good'' attention function $\tilde{\Psi}$ which separates the edges with an arbitrarily large margin, with high probability the graph attention convolution \eqref{eq:gat_output} will still misclassify at least one node as long as $\|\bmu\|/\sigma$ is sufficiently small.

We consider the class of functions $\tilde{\Psi}$ which can be expressed in the following form:
\begin{equation}\label{eq:good_psi}
	\tilde{\Psi}(i, j) = \left\{\begin{array}{ll} \hspace{3.5mm}\sign(p-q) t , & \mbox{if $(i,j)$ is an intra-class edge}, \\ -\sign(p-q)t, & \mbox{if $(i,j)$ is an inter-class edge}, \end{array}\right.
\end{equation}
for some $t \ge 0$. The particular class of functions in \eqref{eq:good_psi} is motivated by the property of the ideal edge classifier in \eqref{eq:ideal_edge_classifier} and the behavior of $\Psi$ in \eqref{eq:psi_expect} when it is applied to the means of the Gaussians. There are a few possible ways to obtain a function $\tilde{\Psi}$ which satisfies \eqref{eq:good_psi}. For example, in the presence of good edge features which reflect the class memberships of the edges, we can make $\tilde{\Psi}$ take as input the edge features. Moreover, if $|\sqrt{p}-\sqrt{q}| > \sqrt{2\log n/n}$, one such $\tilde{\Psi}$ may be easily realized from the eigenvectors of the graph adjacency matrix. By the exact spectral recovery result in Lemma~\ref{lem:spectral_recovery}, we know that there exists a classifier $\hat\tau$ which separates the nodes. Therefore, we can set $\tilde{\Psi}(i,j) = \sign(p-q)t$ if $\hat\tau(i) = \hat\tau(j)$ and $\tilde{\Psi}(i,j) = -\sign(p-q)t$ otherwise. 

\begin{lemma}[Exact recovery in~\cite{Abbe2018}]\label{lem:spectral_recovery} 
Suppose that $p,q=\Omega(\log^2n/n)$ and $|\sqrt{p}-\sqrt{q}|>\sqrt{2\log n/n}$. Then there exists a classifier $\hat \tau$ taking as input the graph $\bA$ and perfectly classifies the nodes with probability at least $1-o(1)$.
\end{lemma}

\begin{theorem}\label{thm:good_psi_negative}
Suppose that $p,q$ satisfy Assumption~\ref{ass:p_q} and that $p,q$ are bounded away from 1. For every $\epsilon > 0$, there are absolute constants $M, M' = O(\sqrt{\epsilon})$ such that, with probability at least $1-o(1)$ over the data $(\bX,\bA) \sim \CSBM(n,p,q,\bmu,\sigma^2)$, using the graph attention convolution in \eqref{eq:gat_output} and the attention architecture $\tilde{\Psi}$ in \eqref{eq:good_psi}, the model misclassifies at least $\Omega(n^{1-\epsilon})$ nodes for any $\bw$ such that $\|\bw\| = 1$, if
\begin{enumerate}
    \item $t = O(1)$ and $\|\bmu\| \le M \sigma \sqrt{\frac{\log n}{n(p+q)}(1-\max(p,q))}  \frac{p+q}{|p-q|}$;
    \item $t = \omega(1)$ and $\|\bmu\| \le M' \sigma \sqrt{\frac{\log n}{n(p+q)}(1-\max(p,q))}$.
\end{enumerate}
\end{theorem}

Theorem~\ref{thm:good_psi_negative} warrants some discussions. We start with the role of $t$ in the attention function \eqref{eq:good_psi}. One may think of $t$ as the multiplicative margin of separation for intra-class and intra-class edges. When $t = O(1)$, the margin of separation is at most a constant. This includes the special case when $\tilde{\Psi}(i,j) = 0$ for all $(i,j) \in E$, i.e, the margin of separation is 0. In this case, the graph attention convolution in \eqref{eq:gat_output} reduces to the standard graph convolution with uniform averaging among the neighbors. Therefore, part 1 of Theorem~\ref{thm:good_psi_negative} also applies to the standard graph convolution. On the other hand, when $t = \omega(1)$, the margin of separation is not only bounded away from $0$ but also it grows with $n$.

Next, we discuss the additional assumption that $p,q$ are bounded away from 1. This assumption is used to obtain a concentration result required for the proof of Theorem~\ref{thm:good_psi_negative}. It is also intuitive in the following sense. If both $p$ and $q$ are arbitrarily close to 1, then after the convolution the convolved node feature vectors collapse to approximately a single point, and thus this becomes a trivial case where no classifier is able to separate the nodes; on the other hand, if $p$ is arbitrarily close to 1 and $q$ is very small, then after the convolution the convolved node feature vectors collapse to approximately one of two points according to which class the node comes from, and in this case the nodes can be easily separated by a linear classifier. 

We now focus on the threshold for $\|\bmu\|$ under which the model is going to misclassify at least one node with high probability. In part 1 of Theorem~\ref{thm:good_psi_negative}, $t=O(1)$, i.e., the attention mechanism $\tilde{\Psi}$ is either unable to separate the edges or unable to separate the edges with a large enough margin. In this case, one can show that all attention coefficients are $\Theta(\frac{1}{n(p+q)})$. Consequently, the quantity $|p-q|$ appears in denominator of the threshold for $\|\bmu\|$ in part 1 of Theorem~\ref{thm:good_psi_negative}. Because of that, if $p$ and $q$ are arbitrarily close, then the model is not able to separate the nodes irrespective of how large $\|\bmu\|$ is. For example, treating $1-\max(p,q)$ as a constant since $p$ and $q$ are bounded away from 1 by assumption, we have that
\[
	|p-q| = o\left(\sqrt{\frac{p+q}{n}}\right) \ \mbox{implies} \ M \sigma \sqrt{\frac{\log n}{n(p+q)}(1-\max(p,q))} \frac{p+q}{|p-q|} = \omega(\sigma \sqrt{\log n}).
\]
This means that if $p$ and $q$ are close enough, every attention function $\tilde{\Psi}$ in the form of \eqref{eq:good_psi} and $t=O(1)$ cannot help classify all nodes correctly even if $\|\bmu\| = \omega(\sigma \sqrt{\log n})$. On the contrary, recall that in the easy regime where $\|\bmu\| = \omega(\sigma \sqrt{\log n})$, the attention mechanism given in \eqref{eq:psi_ansatz} and \eqref{eq:psi_ansatz_parameters} helps separate the nodes with high probability. This illustrates the limitation of every attention mechanism in the form of \eqref{eq:good_psi} which have an insignificant margin of separation. According to Theorem~\ref{thm:gamma_hard}, the vast majority of attention coefficients are uniform, and thus in Conjecture~\ref{conj:node_c_hard} we expect that graph attention in general shares similar limitations in the hard regime.

In part 2 of Proposition~\ref{thm:good_psi_negative}, $t=\omega(1)$, i.e., the attention mechanism $\tilde{\Psi}$ separates the edges with a large margin. In this case, one can show that the attention coefficients on important edges (e.g. intra-class edges) are exponentially larger than those on unimportant edges (e.g. inter-class edges). Consequently, the factor $(p+q)/|p-q|$ no longer appears in the threshold for $\|\bmu\|$ in part 2 of Theorem~\ref{thm:good_psi_negative}. However, at the same time, the threshold also implies that, even when we have a perfect attention mechanism that is able to separate the edges with a large margin, as long as $\|\bmu\|/\sigma$ is small enough, then the model is going to misclassify at least one node with high probability.

Finally, notice that in Theorem~\ref{thm:good_psi_negative} the parameter $\epsilon$ captures a natural tradeoff between the threshold for $\|\bmu\|$ and the lower bound on the number of misclassified nodes. Namely, the smaller the $\epsilon$ is, the smaller the threshold for $\|\bmu\|$ becomes, and hence the less signal there is in the node features, the more mistakes the model is going to make. This is precisely demonstrated by the scaling of $M,M' = O(\sqrt{\epsilon})$ and misclassification bound $\Omega(n^{1-\epsilon})$ with respect to $\epsilon$. We leave the proof of Theorem~\ref{thm:good_psi_negative} to the appendix.

\section{Robustness to structural noise and its implications beyond perfect node classification}\label{sec:beyond}

In this section, we provide a positive result on the capacity of graph attention convolution for node classification beyond the perfect classification regime. In particular, we show that independent of the parameters of CSBM, i.e., independent of $p$, $q$, $\bmu$ and $\sigma$, the two-layer MLP attention architecture $\Psi$ from Section~\ref{subsec:easy} is able to achieve the classification performance obtainable by the Bayes optimal classifier for node features. This is proved by showing that there is a parameter setting for $\Psi$ where the attention coefficient on self-loops can be made exponentially large. Consequently, the corresponding graph attention convolution behaves like a linear function of node features. We provide two corollaries of this result. The first corollary provides a ranking of graph attention convolution, simple graph convolution, and linear function in terms of their ability to classify nodes in CSBM. More specifically, by noticing that the simple graph convolution is also realized by a specific parameter setting of the attention architecture $\Psi$, our result implies that the performance of graph attention convolution for node classification in CSBM is lower bounded by the best possible performance between a linear classifier and simple graph convolution. In particular, graph attention is strictly more powerful than simple graph convolution when the graph is noisy (e.g. when $p \approx q$), and it is strictly more powerful than a linear classifier when the graph is less noisy (e.g. when $p$ and $q$ are significantly different). The second corollary provides perfect classification, almost perfect classification, and partial classification results for graph attention convolution. It follows immediately from the reduction of graph attention convolution to a linear function under the specification of $\Psi$ that we will discuss. In what follows we start with high-level ideas, then we present formal statements of the results, and we discuss the implications in detail.

Recall the two-layer MLP attention architecture $\Psi$ from \eqref{eq:psi_ansatz} and \eqref{eq:psi_ansatz_parameters} is equivalently written in \eqref{eq:psi_simplified} as
\[
    \Psi(\bX_i,\bX_j) =
    \left\{
    \begin{array}{ll}
    -2R(1-\beta)\tilde\bw^T\bX_i, & \mbox{if}~ \tilde\bw^T\bX_j \le -\left|\tilde\bw^T\bX_i\right|,\\
    \hspace{3mm}2R(1-\beta)\sign(\tilde\bw^T\bX_i)\tilde\bw^T\bX_j, & \mbox{if}~ -\left|\tilde\bw^T\bX_i\right| < \tilde\bw^T\bX_j < \left|\tilde\bw^T\bX_i\right|,\\
    \hspace{3mm}2R(1-\beta)\tilde\bw^T\bX_i, & \mbox{if}~ \tilde\bw^T\bX_j \ge \left|\tilde\bw^T\bX_i\right|.
    \end{array}
    \right.
\]
We make the following observations. Assuming that the scaling parameter $R > 0$,
\begin{itemize}
  \item If $\tilde\bw^T\bX_i > 0$, then the function $\Psi$ assigns more weight to an edge $(i,j)$ such that $\tilde\bw^T\bX_j > 0$ than an edge $(i,j')$ such that $\tilde\bw^T\bX_{j'} < 0$; 
  \item If $\tilde\bw^T\bX_i < 0$, then the function $\Psi$ assigns more weight to an edge $(i,j)$ such that $\tilde\bw^T\bX_j < 0$ than an edge $(i,j')$ such that $\tilde\bw^T\bX_{j'} > 0$;
  \item If $\tilde\bw^T\bX_i = 0$, then the function $\Psi$ assigns uniform weight to every edge $(i,j)$.
\end{itemize}
This means that the behavior of $\Psi$ depends on which side of the hyperplane $\{\bx : \tilde\bw^T\bx = 0\}$ that $\bX_i$ comes from. In other words, for fixed $\bX_j$, whether the attention function $\Psi$ will up-weight or down-weight an edge $(i,j)$ depends entirely on the classification of $\bX_i$ based on the linear classifier $\tilde\bw$. Moreover, note that the attention function value satisfies 
\[
    2R(1-\beta) \cdot \max\{-|\tilde\bw^T\bX_i|, -|\tilde\bw^T\bX_j|\} \le \Psi(\bX_i,\bX_j) \le 2R(1-\beta) \cdot \min\{|\tilde\bw^T\bX_i|,|\tilde\bw^T\bX_j|\}.
\]
Therefore, out of all unit norm vectors $\bw$, our choice $\tilde\bw = \bmu/\|\bmu\|$ maximizes the range of values that $\Psi$ can output. Recall from Lemma~\ref{lem:bayes} that $\tilde\bw$ also happens to characterize the Bayes optimal classifier for the node features. Finally, the attention function $\Psi$ achieves minimum/maximum at self-attention, i.e. 
\begin{align*}
    \Psi(\bX_i,\bX_i) &= 2R(1-\beta)|\tilde\bw^T\bX_i| = \max_{j \in [n]} \Psi(\bX_i,\bX_j),\\
    \Psi(\bX_i,-\bX_i) &= -2R(1-\beta)|\tilde\bw^T\bX_i| = \min_{j \in [n]} \Psi(\bX_i,\bX_j).
\end{align*}
A consequence of these observations is that, by setting the scaling parameter $R$ sufficiently large, one can make $\exp(\Psi(\bX_i,\bX_j))$ exponentially larger than $\exp(\Psi(\bX_i,\bX_{j'}))$ for any $j,j'$ such that $\sign(\tilde\bw^T\bX_j) = \sign(\tilde\bw^T\bX_i)$ and $\sign(\tilde\bw^T\bX_{j'}) \neq \sign(\tilde\bw^T\bX_i)$. According to the definition of attention coefficients in \eqref{def:attention_coeff}, this means that the attention coefficients $\gamma_{ij}$ where $\sign(\tilde\bw^T\bX_j) = \sign(\tilde\bw^T\bX_i)$ are going to be exponentially larger than the attention coefficients $\gamma_{ij'}$ where $\sign(\tilde\bw^T\bX_{j'}) \neq \sign(\tilde\bw^T\bX_i)$. Therefore, one could expect that in this case, if the linear classifier $\tilde\bw$ correctly classifies $\bX_i$ for some $i \in [n]$, e.g., for $i \in C_1$ this means that $\tilde\bw^T\bX_i > 0$, then graph attention convolution output $h_i' = \sum_{j \in N_i}\gamma_{ij}\tilde\bw^T\bX_j$ also satisfies $h_i' > 0$, due to sufficiently larger attention coefficients $\gamma_{ij}$ for which $\tilde\bw^T\bX_j > 0$. We state the result below in \cref{thm:gat_linear} and leave detailed arguments in the appendix.

\begin{theorem}\label{thm:gat_linear}
With probability at least $1-o(1)$ over the data $(\bX,\bA) \sim \CSBM(n,p,q,\bmu,\sigma^2)$, using the two-layer MLP attention architecture $\Psi$ given in \eqref{eq:psi_ansatz} and \eqref{eq:psi_ansatz_parameters} with $R = \Omega(n \log^2n/\sigma)$, the graph attention convolution output satisfies
\begin{align*}
&h_i' = \sum_{j\in N_i}\gamma_{ij}\tilde\bw^T\bX_j > 0 \; \mbox{if and only if} \; \tilde\bw^T\bX_i > 0, \; \forall i \in[n],\\
&h_i' = \sum_{j\in N_i}\gamma_{ij}\tilde\bw^T\bX_j < 0 \; \mbox{if and only if} \; \tilde\bw^T\bX_i < 0, \; \forall i \in[n].
\end{align*}
\end{theorem}

\cref{thm:gat_linear} means that there is a parameter setting for the attention architecture $\Psi$ such that the performance of graph attention convolution matches with the performance of the Bayes optimal classifier for node features. This shows the ability of graph attention to ``ignore'' the graph structure, which can be beneficial when the graph is noisy. For example, if $p=q$, then it is easy to see that simple graph convolution completely mixes up the node features, making it not possible to achieve any meaningful node classification result. On the other hand, as long as there is some signal from the original node features, i.e. $\|\bmu\| > 0$, then graph attention will be able to pick that up and classify the nodes at least as good as the best classifier for the node features alone. It is also worth noting that by setting $R = 0$, the attention function $\Psi$ has a constant value, and hence graph attention convolution reduces to the standard graph convolution, which has been shown to be useful in the regime where the original node features are not very strong but the graph has a nice structure~\cite{BFJ2021}. For example, when there is a significant gap between $p$ and $q$, $|p-q|/(p+q) = \Omega(1)$, then setting $R=0$ could significantly improve the node separability threshold over the best linear classifier~\cite{BFJ2021}. This shows the robustness of graph attention against noise in one of the two sources of information, namely node features and edge connectivities. Unlike the Bayes optimal classifier for node features which is sensitive to feature noise or the simple graph convolution which is sensitive to structural noise, one can expect graph attention to work as long as at least one of the two sources of information has a good signal. Therefore, we obtain a ranking of node classification models among graph attention convolution, simple graph convolution, and a linear classifier. We state this below in Corollary~\ref{cor:model_rank}. 

\begin{corollary}\label{cor:model_rank}
The node classification performance obtainable by graph attention convolution is lower bounded by the best possible node classification performance between a simple graph convolution and a linear classifier.
\end{corollary}

By a straightforward characterization of the performance of the linear classifier $\tilde\bw = \bmu/\|\bmu\|$, we immediately obtain the following classification results in Corollary~\ref{cor:gat_linear_recovery}. Recall that we denoted $\Phi$ as the cumulative distribution function of the standard normal distribution. Write 
\[
    \|\bmu\| = \kappa \sigma  \;\ \mbox{for some} \;\ \kappa>0.
\]

\begin{corollary}\label{cor:gat_linear_recovery}
With probability at least $1-o(1)$ over the data $(\bX,\bA) \sim \CSBM(n,p,q,\bmu,\sigma^2)$, using the two-layer MLP attention architecture $\Psi$ given in \eqref{eq:psi_ansatz} and \eqref{eq:psi_ansatz_parameters} with $R = \Omega(n \log^2n/\sigma)$, one has that
\begin{itemize}
    \item (Perfect classification) If $\kappa \ge \sqrt{2\log n}$ then all nodes are correctly classified;
    \item (Almost perfect classification) If $\kappa = \omega(1)$ then at least $1-o(1)$ fraction of all nodes are correctly classified;
    \item (Partial classification) If $\kappa = O(1)$ then at least $\Phi(\kappa)-o(1)$ fraction of all nodes are correctly classified.
\end{itemize}
\end{corollary}

Interestingly, we note that the perfect classification result in Corollary~\ref{cor:gat_linear_recovery} is nearly equivalent (up to a small order in the threshold $\kappa$) to the perfect classification result in Corollary~\ref{cor:node_separation_easy} from Section~\ref{subsec:easy}. They are obtained from different behaviors of the attention coefficients. This shows that there could be more than one type of ``good'' attention coefficients that are able to deliver good node classification performance.

\section{Experiments}\label{sec:experimets}

In this section, we demonstrate empirically our results on synthetic and real data. The parameters of the models that we experiment with are set by using an ansatz based on our theorems. The particular details are given in Section~\ref{subsec:ansatz-in-experiments}. For real datasets, we use the default splits which come from PyTorch Geometric~\cite{FL2019} and OGB~\cite{HFZDRLCL20}. In all our experiments we use MLP-GAT, where the attention mechanism $\Psi$ is set to be the two-layer network in \eqref{eq:psi_ansatz} and \eqref{eq:psi_ansatz_parameters} with $R = 1$. For synthetic experiments using CSBM with known $p$ and $q$, we use the variant that takes $p,q$ into account, $\Psi' = (\mathbf{1}_{p\ge q} - \mathbf{1}_{p < q}) \Psi$. In Figure~\ref{fig:gammas_MLPGAT_var_q_hard} and Figure~\ref{fig:Gammas_GAT_var_dist} we additionally consider the original GAT architecture of \cite{Velickovic2018GraphAN} to demonstrate Theorem~\ref{thm:gamma_hard}.

\subsection{Ansatz for GAT, MLP-GAT and GCN}\label{subsec:ansatz-in-experiments}
For the original GAT architecture we fix $\bw=\bmu/\|\bmu\|$ and define the first head as $\ba_1=\frac{1}{\sqrt{2}}(1,1)$ and $b_1=-\frac{1}{\sqrt{2}}\bw^T\bmu$; The second head is defined as $\ba_2=-\ba_1$ and $b_2=-b_1$. We now discuss the choice of such an ansatz. The parameter $\bw$ is picked based on the optimal Bayes classifier without a graph, and the attention is set such that the first head maintains intra-class edges in $C_1$ and the second head maintains intra-class edges in $C_0$. Note that for the original GAT~\cite{Velickovic2018GraphAN}, due to the fact that the attention mechanism consists of just one layer (i.e. a nonlinear activation applied on a linear transformation, see \eqref{eq:softmaxattention}), it is not possible for the original GAT to keep only $\gamma_{ij}$ which correspond to intra-class edges. More specifically, one may use the same techniques in the proof of Theorem~\ref{thm:edge_separation_easy} and Corollaries~\ref{cor:gamma_easy} and~\ref{cor:node_separation_easy} to prove the node separability results for the original GAT. In this particular case, the result will \emph{depend on $q$} in contrast to the result we get for MLP-GAT, where no dependence of $q$ was needed. For MLP-GAT we use the ansatz $\Psi' = (\mathbf{1}_{p\ge q} - \mathbf{1}_{p < q}) \Psi$ where $\Psi$ is given in \eqref{eq:psi_ansatz} and \eqref{eq:psi_ansatz_parameters} with $R = 1$. This choice of two-layer network allows us to bypass the ``XOR problem"~\cite{minsky1969perceptron} and separate inter-class from intra-class edges as shown in Theorem~\ref{thm:edge_separation_easy}. Note that no single-layer architecture will be able to separate the edges due to the ``XOR problem''. For GCN we used the ansatz from~\cite{BFJ2021} which is also $\bw=\bmu/\|\bmu\|$.

\subsection{Synthetic data}\label{subsec:synthetic_data}
We use the CSBM to generate the data. In a recent work~\cite{graphworld2022}, a simple variant of the CSBM was also chosen as the default generative model for benchmarking GNN performance for node classification tasks. We present two sets of experiments. In the first set, we fix the distance between the means and vary $q$, and in the second set, we fix $q$ and vary the distance. We set $n=1000$, $d=n/\log^2(n)$, $p=0.5$ and $\sigma=0.1$. Results are averaged over $10$ trials.

\subsubsection{Fixing the distance between the means and varying \texorpdfstring{$q$}{q}}\label{sec:Exp_fix_dist_var_q}

We consider the two regimes separately, where for the ``easy regime" we fix the mean $\bmu$ to be a vector where each coordinate is equal to $10\sigma\sqrt{\log{n^2}}/2\sqrt{d}$. This guarantees that the distance between the means is $10\sigma\sqrt{\log{n^2}}$. In the ``hard regime" we fix the mean $\bmu$ to a vector where each coordinate is equal to $\sigma/\sqrt{d}$, and this guarantees that the distance is $\sigma$. We fix $p=0.5$ and vary $q$ from $\log^2(n) / n$ to $1-\log^2(n) / n$.

In Figure~\ref{fig:edge_c_var_q_part1} we illustrate Theorem~\ref{thm:edge_separation_easy} and Corollaries~\ref{cor:gamma_easy},~\ref{cor:node_separation_easy} for the easy regime, and in Figure~\ref{fig:edge_c_var_q_part2} we illustrate Theorem~\ref{thm:edge_separation_hard} and Theorem~\ref{thm:gamma_hard} for the hard regime. In particular, in Figure~\ref{fig:edge_c_var_q_easy} we show Theorem~\ref{thm:edge_separation_easy}, MLP-GAT is able to classify intra-class and inter-class edges perfectly. In Figure~\ref{fig:gammas_MLPGAT_var_q_easy} we show that in the easy regime, as claimed in Corollary~\ref{cor:gamma_easy} for MLP-GAT, when $p \ge q$, the $\gamma$ that correspond to intra-class edges concentrate around $2/np$, while the $\gamma$ for the inter-class edges concentrate to tiny values; when $p < q$, we see the opposite, that is the $\gamma$ that correspond to intra-class edges concentrate to tiny values, while the $\gamma$ for the inter-class edges concentrate around $2/nq$. In Figure~\ref{fig:node_c_MLPGAT_var_q_easy} we observe that the performance of MLP-GAT for node classification is independent of $q$ in the easy regime as claimed in Corollary~\ref{cor:node_separation_easy}. However, in this plot, we observe that not using the graph also achieves perfect node classification, a result which is proved in Proposition~\ref{prop:linear_easy}. In the same plot, we also show the performance of simple graph convolution, where its performance depends on $q$ (see~\cite{BFJ2021}). In Figure~\ref{fig:edge_c_var_q_hard} we show Theorem~\ref{thm:edge_separation_hard}. MLP-GAT misclassifies a constant fraction of the intra and inter edges as proved in Theorem~\ref{thm:edge_separation_hard}. In Figure~\ref{fig:gammas_MLPGAT_var_q_hard} we show Theorem~\ref{thm:gamma_hard}, where $\gamma$'s in the hard regime concentrate around uniform (GCN) coefficients for both MLP-GAT and GAT. In Figure~\ref{fig:node_c_MLPGAT_var_q_hard} we illustrate that node classification accuracy is a function of $q$ for MLP-GAT. This is conjectured in Conjecture~\ref{conj:node_c_hard}. On the other hand, note that the performance of MLP-GAT is lower bounded by the performance of not using the graph, as proved in Theorem~\ref{thm:gat_linear}. 

\begin{figure}[ht!]
     \centering
     \begin{subfigure}[b]{0.49\textwidth}
         \centering
         \includegraphics[width=.95\textwidth]{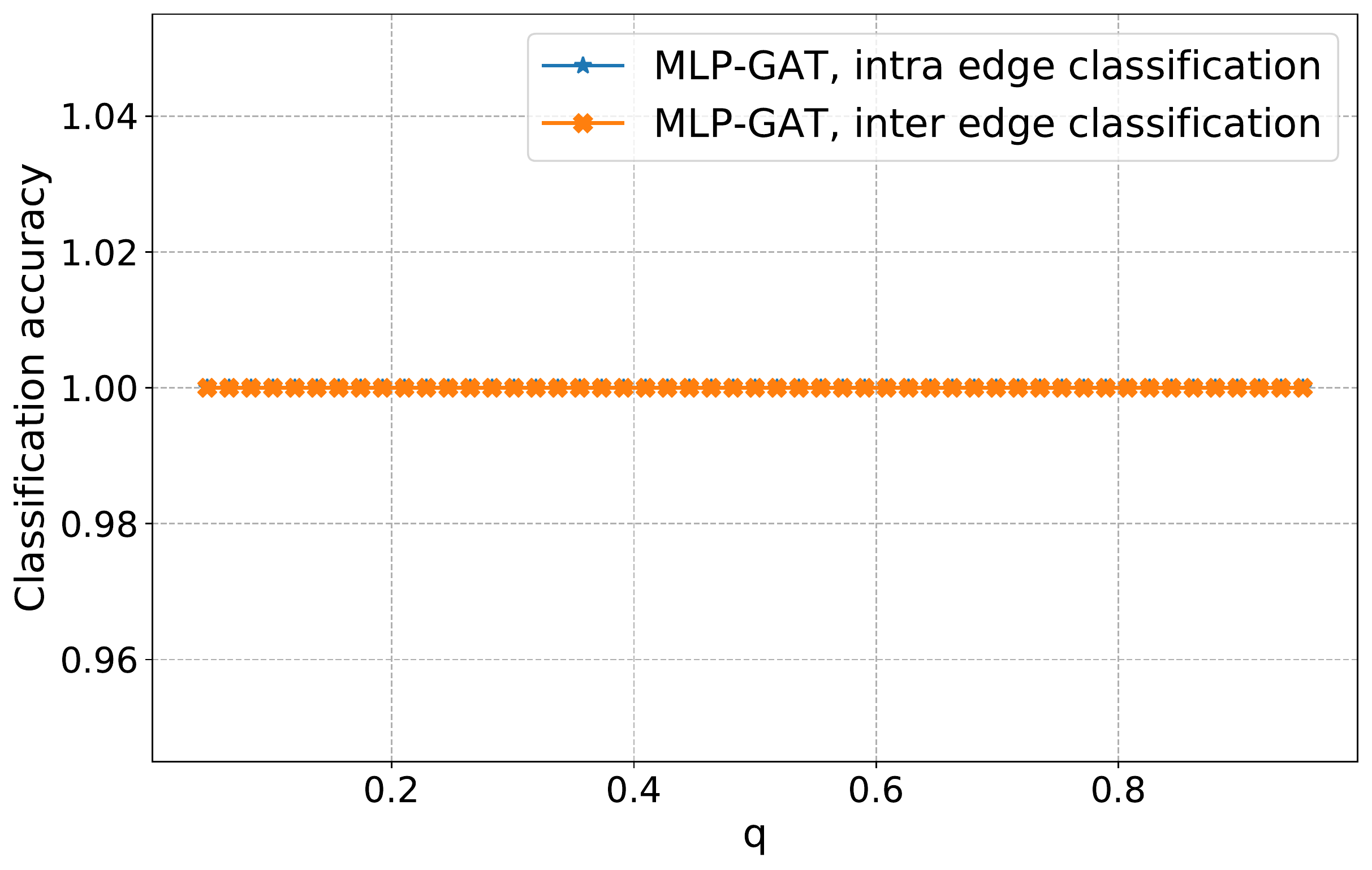}
         \caption{Edge classification}
         \label{fig:edge_c_var_q_easy}
     \end{subfigure}%
     \begin{subfigure}[b]{0.49\textwidth}
         \centering
         \includegraphics[width=.95\textwidth]{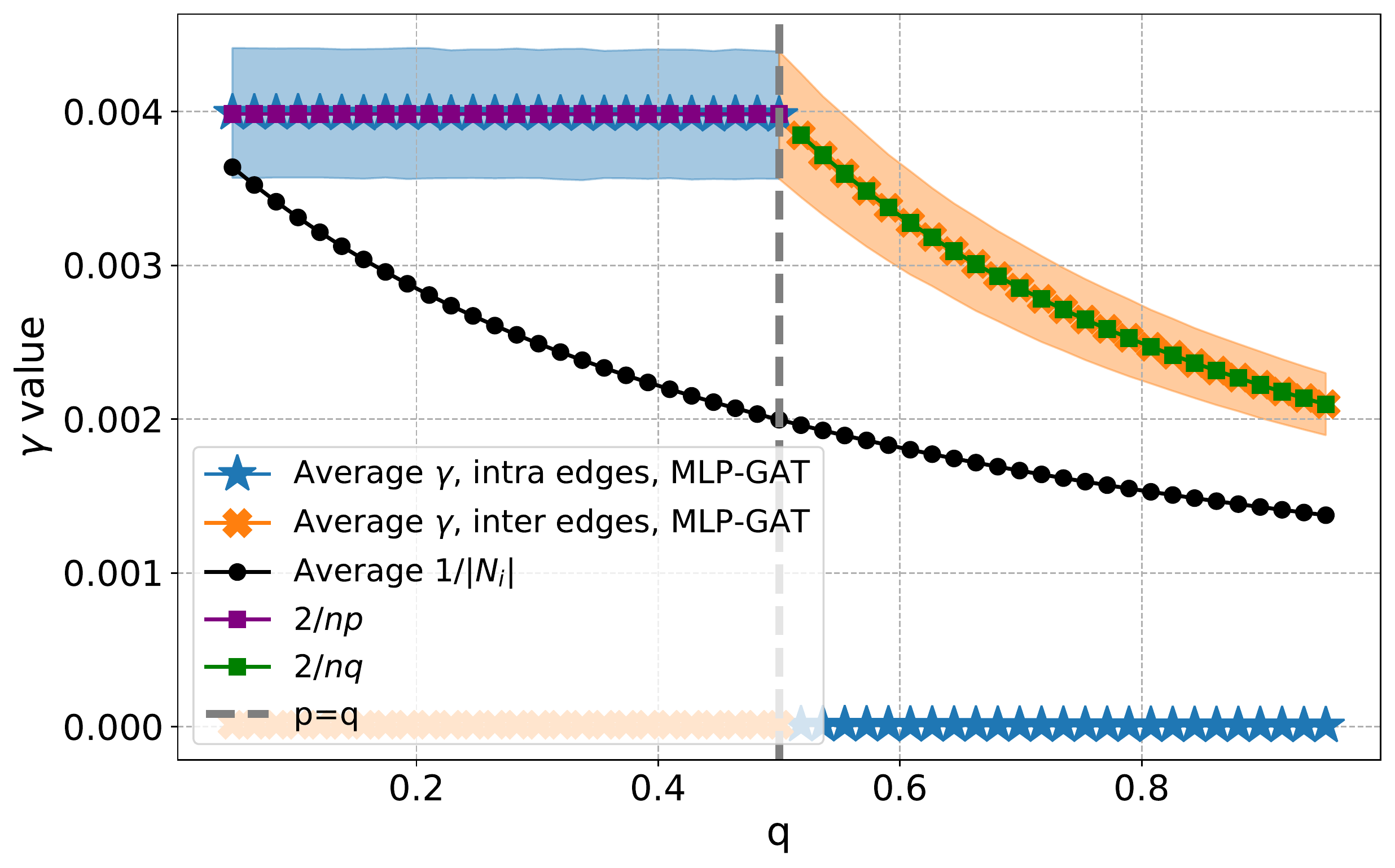}
         \caption{Attention coefficients}
         \label{fig:gammas_MLPGAT_var_q_easy}
     \end{subfigure}
     \begin{subfigure}[b]{0.49\textwidth}
      \vspace{2mm}
         \centering
         \includegraphics[width=.95\textwidth]{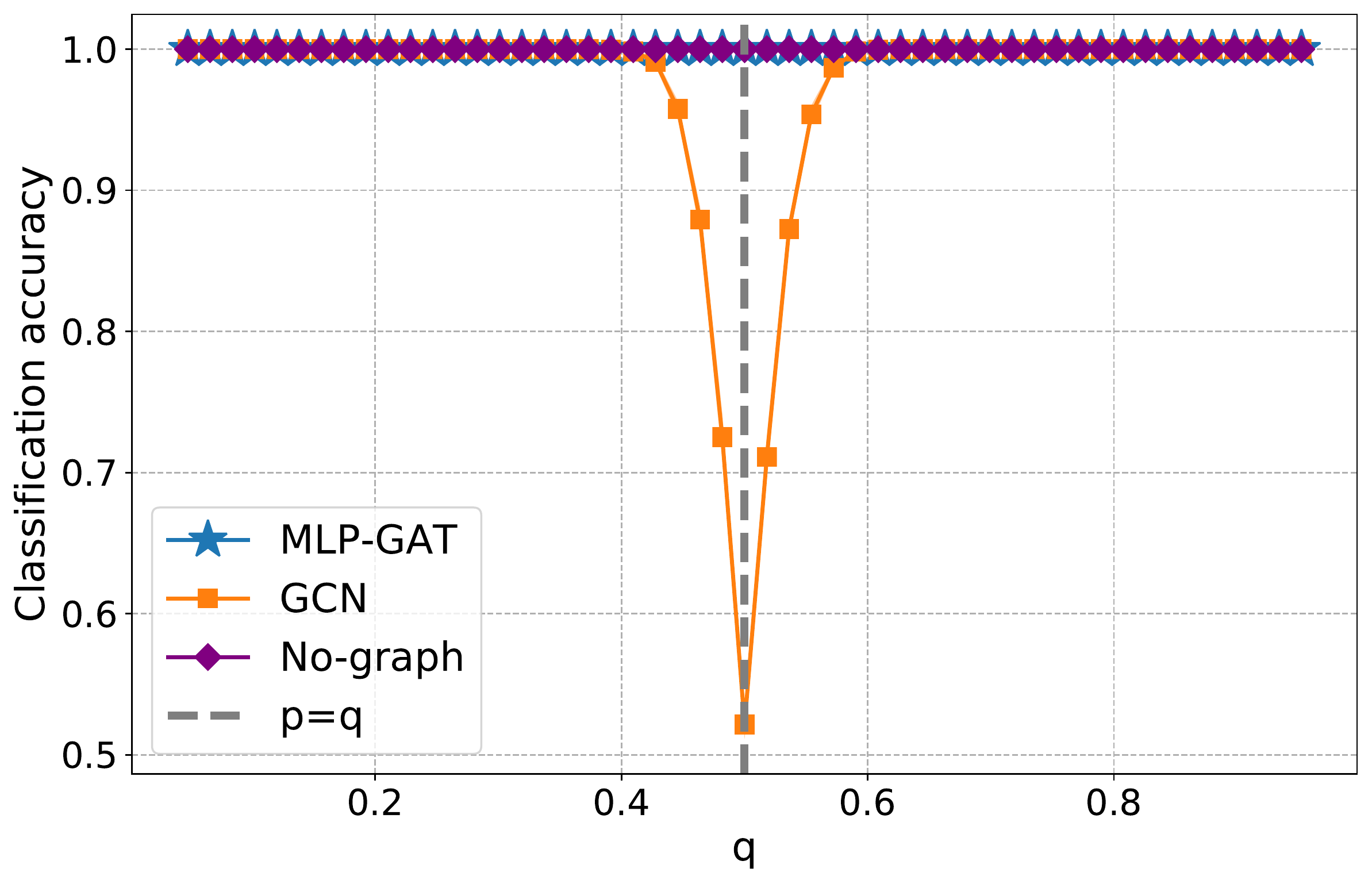}
         \caption{Node classification}
         \label{fig:node_c_MLPGAT_var_q_easy}
     \end{subfigure}
    \caption{Demonstration of Theorem~\ref{thm:edge_separation_easy} and Corollaries~\ref{cor:gamma_easy},~\ref{cor:node_separation_easy} for the easy regime. The shaded areas in the plots show standard deviation. Unlike GCN, the performance of MLP-GAT does not degrade as we vary $q$.}
    \label{fig:edge_c_var_q_part1}
\end{figure}

\begin{figure}[ht!]
     \centering
     \begin{subfigure}[b]{0.49\textwidth}
         \centering
         \includegraphics[width=.95\textwidth]{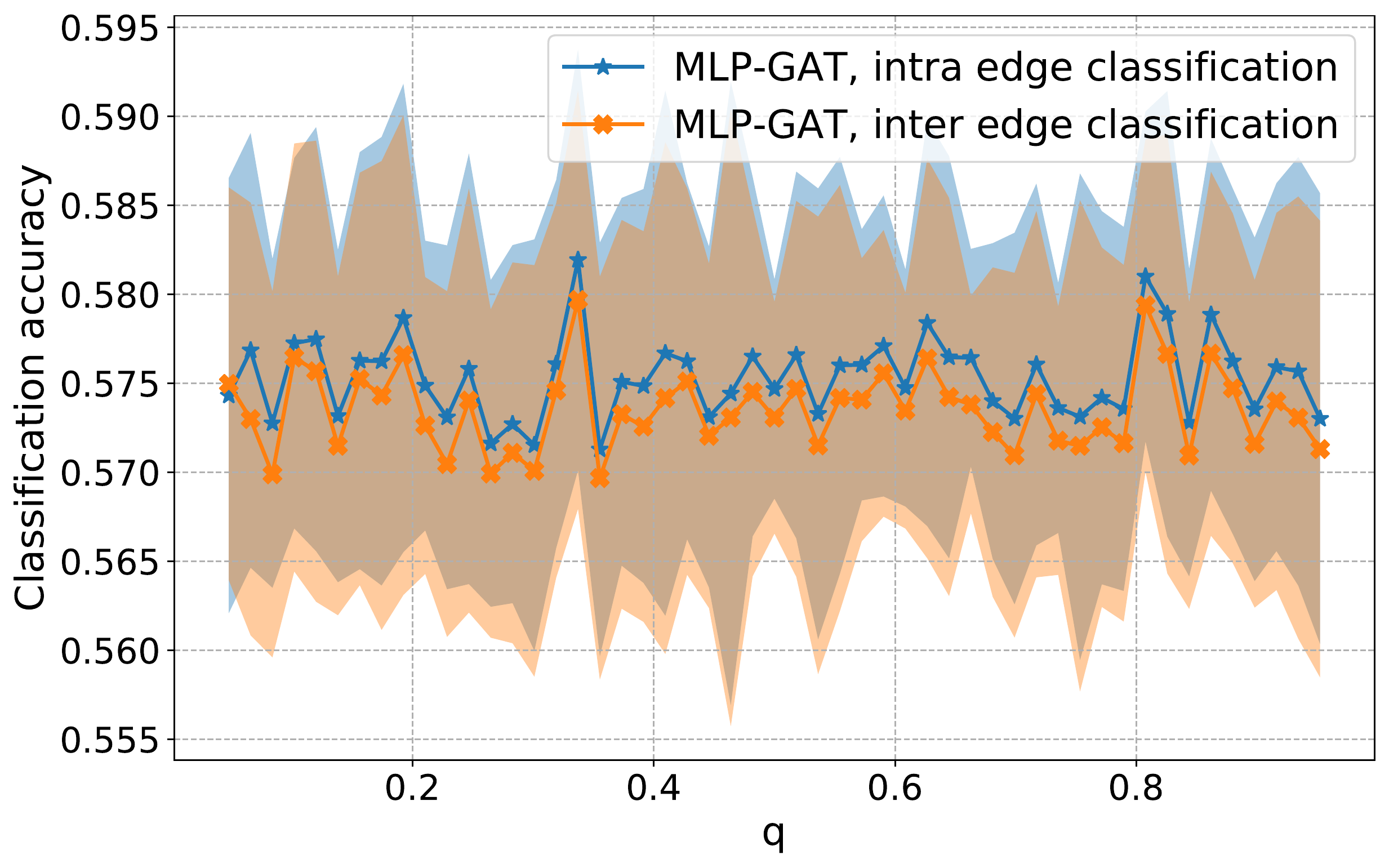}
         \caption{Edge classification}
         \label{fig:edge_c_var_q_hard}
     \end{subfigure}%
     \begin{subfigure}[b]{0.49\textwidth}
         \centering
         \includegraphics[width=.95\textwidth]{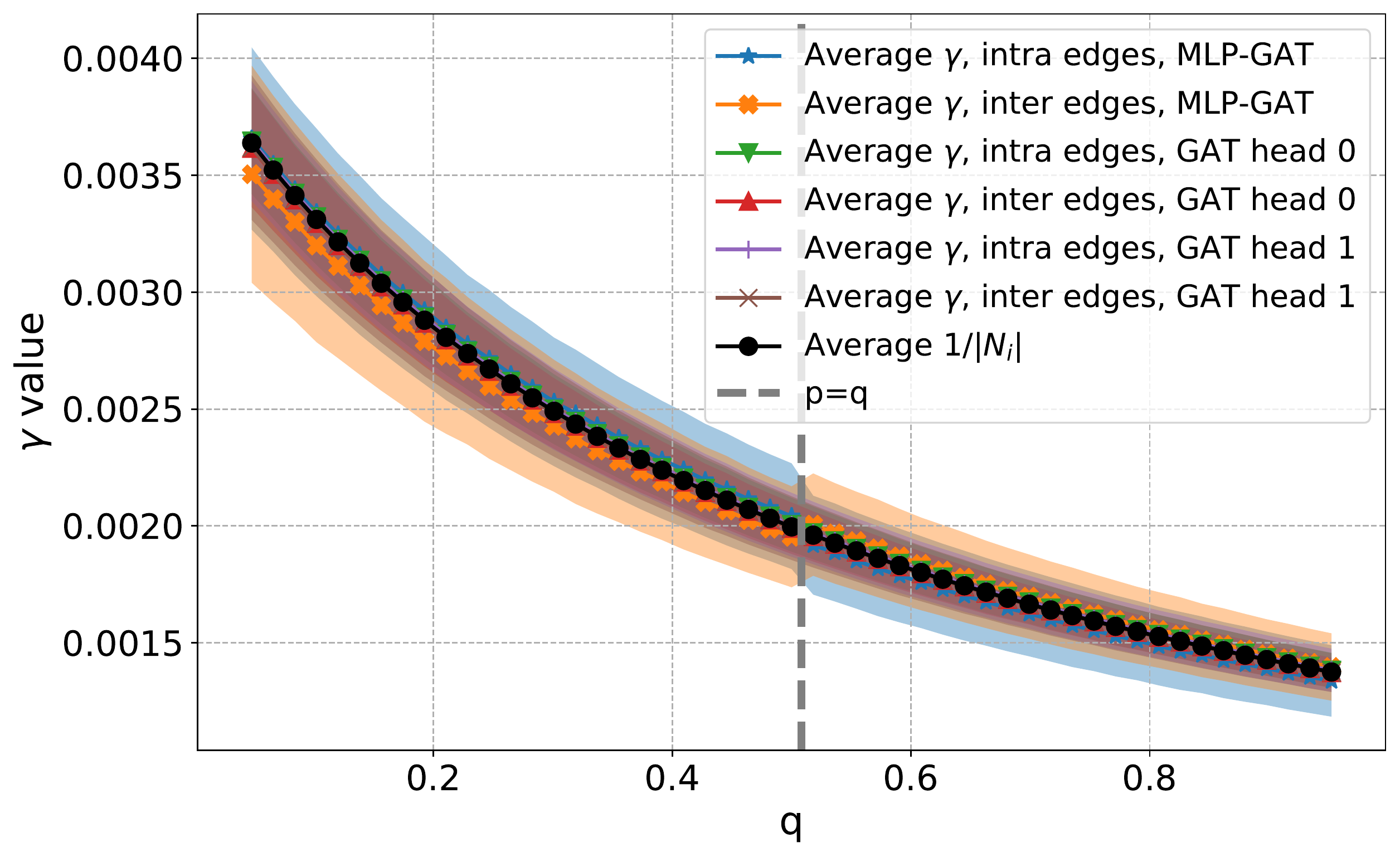}
         \caption{Attention coefficients}
         \label{fig:gammas_MLPGAT_var_q_hard}
     \end{subfigure}
     \begin{subfigure}[b]{0.49\textwidth}
      \vspace{2mm}
         \centering
         \includegraphics[width=.95\textwidth]{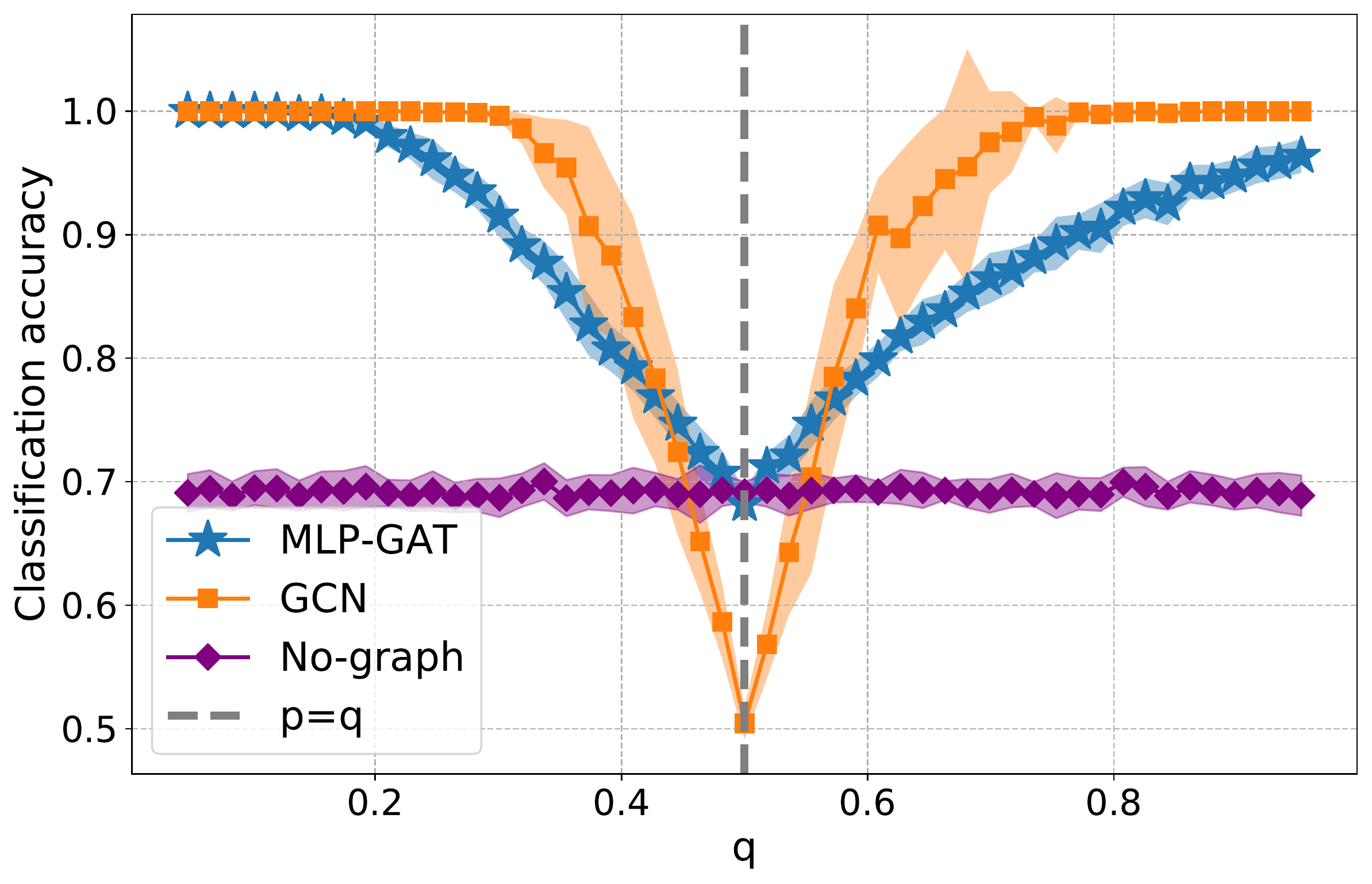}
         \caption{Node classification}
         \label{fig:node_c_MLPGAT_var_q_hard}
     \end{subfigure}
    \caption{Demonstration of Theorem~\ref{thm:edge_separation_hard} and Theorem~\ref{thm:gamma_hard} for the hard regime. The shaded areas in the plots show standard deviation. Unlike GCN, the performance of MLP-GAT is lower bounded by the performance of not using the graph.}
    \label{fig:edge_c_var_q_part2}
\end{figure}

\subsubsection{Fixing \texorpdfstring{$q$}{q} and varying the distance between the means}
We consider the case where $q=0.1$. In Figure~\ref{fig:dist_means} we show how the attention coefficients of MLP-GAT and GAT, the node and edge classification depend on the distance between the means. We also add a vertical line at $\sigma$ to approximately separate the easy (left of $\sigma$) and hard (right of $\sigma$) regimes. Figure~\ref{fig:Edge_c_MLPGAT_var_dist} illustrates Theorems~\ref{thm:edge_separation_easy} and~\ref{thm:edge_separation_hard} in the hard and easy regimes, respectively. In particular, we observe that in the hard regime, MLP-GAT fails to distinguish intra from inter edges, while in the easy regime, it is able to do that perfectly for a large enough distance between the means.

In Figure~\ref{fig:Gammas_MLPGAT_var_dist} we observe that in the hard regime, $\gamma$ concentrate around the uniform (GCN) coefficients, while in the easy regime, MLP-GAT is able to maintain the $\gamma$ for intra-class edges, while it sets the $\gamma$ to tiny values for inter-class edges. In Figure~\ref{fig:Gammas_GAT_var_dist}. we observe that in the hard regime, the $\gamma$ of GAT concentrate around the uniform coefficients (proved in Theorem~\ref{thm:gamma_hard}), while in the easy regime although the $\gamma$ concentrate, GAT is not able to distinguish intra from inter edges. This makes sense since the separation of edges can't be done by simple linear classifiers used by GAT, see the discussion below Theorem~\ref{thm:gamma_hard}. Finally, in Figure~\ref{fig:node_c_MLPGAT_var_dist} we show node classification results for MLP-GAT. In the easy regime, we observe perfect classification as proved in Corollary~\ref{cor:node_separation_easy}. However, as the distance between the means decreases, we observe that MLP-GAT starts to misclassify nodes.

\begin{figure}[ht!]
     \centering
      \begin{subfigure}[b]{0.49\textwidth}
         \centering
         \includegraphics[width=.95\textwidth]{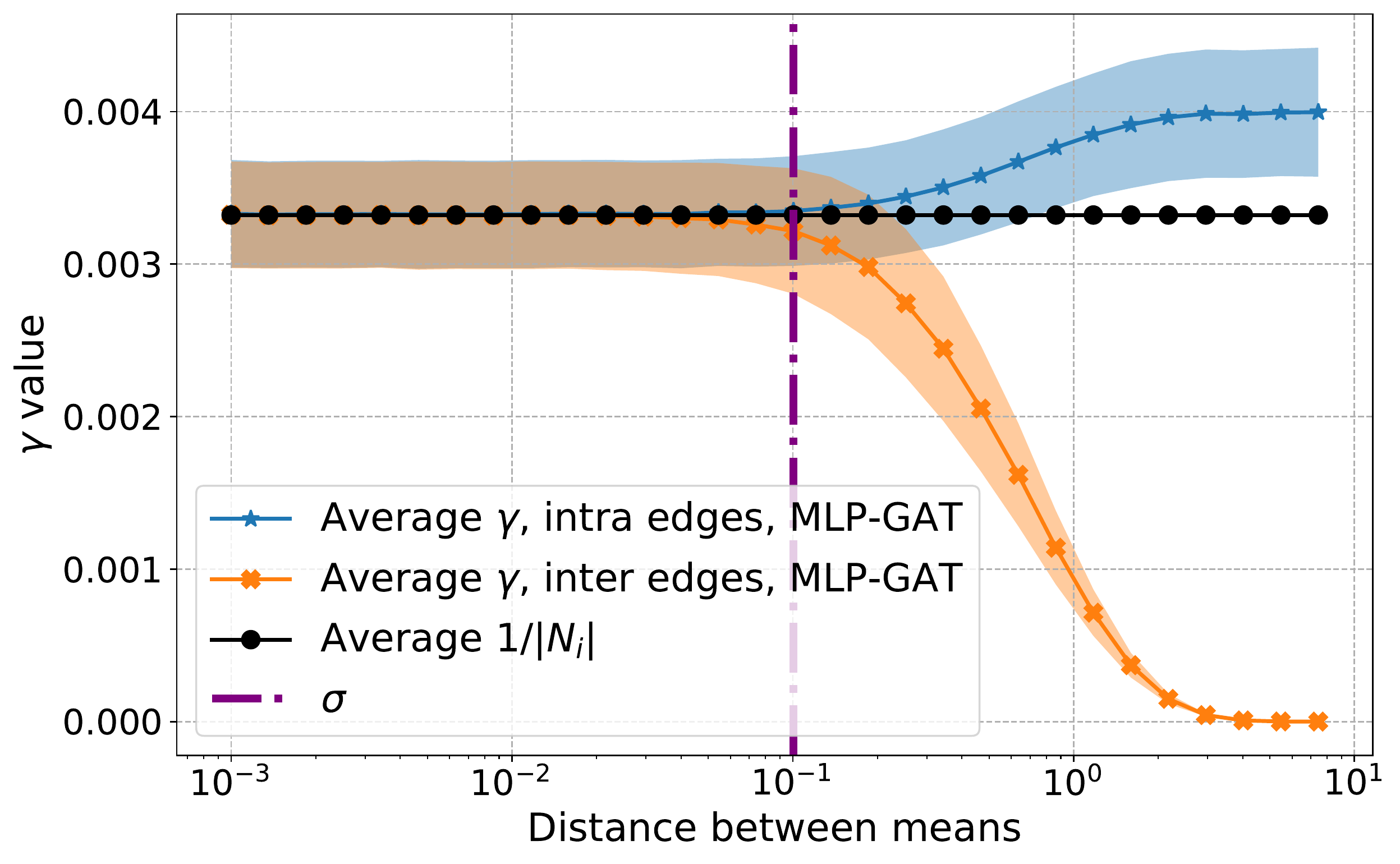}
         \caption{Attention coefficients of MLP-GAT}
         \label{fig:Gammas_MLPGAT_var_dist}
     \end{subfigure}%
     \begin{subfigure}[b]{0.49\textwidth}
         \centering
         \includegraphics[width=.95\textwidth]{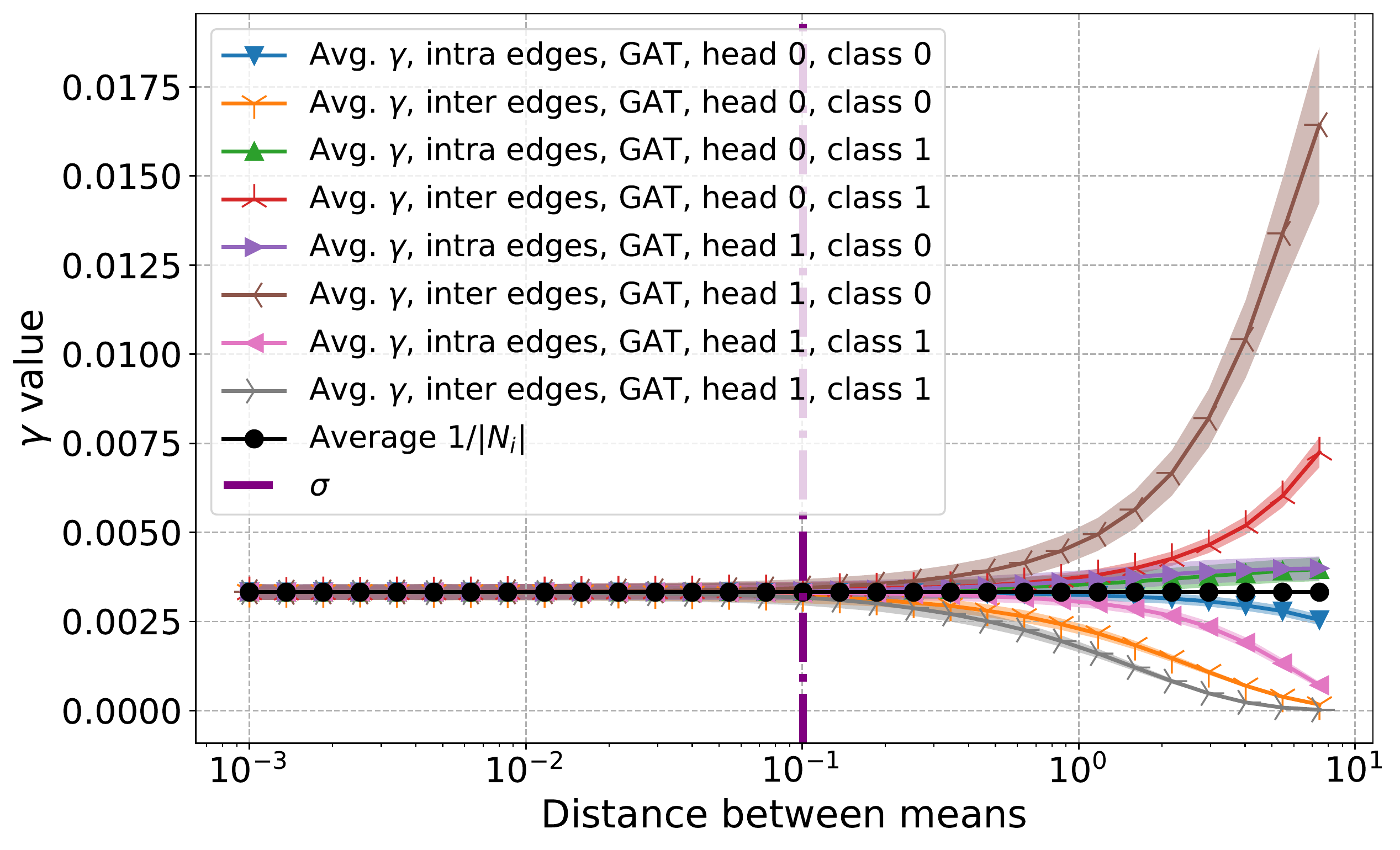}
         \caption{Attention coefficients of GAT}
         \label{fig:Gammas_GAT_var_dist}
     \end{subfigure}
     \begin{subfigure}[b]{0.49\textwidth}
         \vspace{2mm}
         \centering
         \includegraphics[width=.95\textwidth]{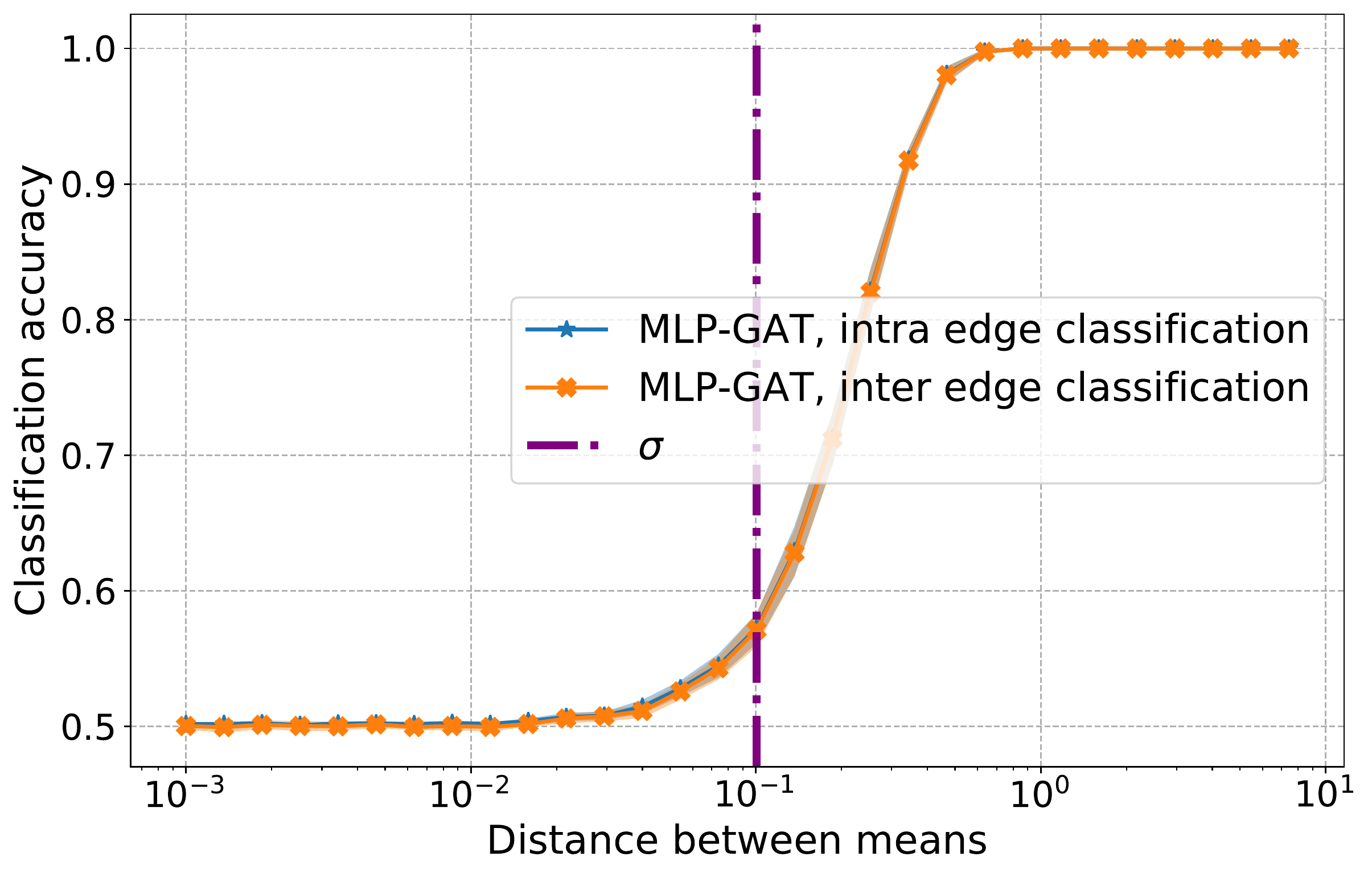}
         \caption{Edge classification accuracy}
         \label{fig:Edge_c_MLPGAT_var_dist}
     \end{subfigure}%
     \begin{subfigure}[b]{0.49\textwidth}
         \vspace{2mm}
         \centering
         \includegraphics[width=.95\textwidth]{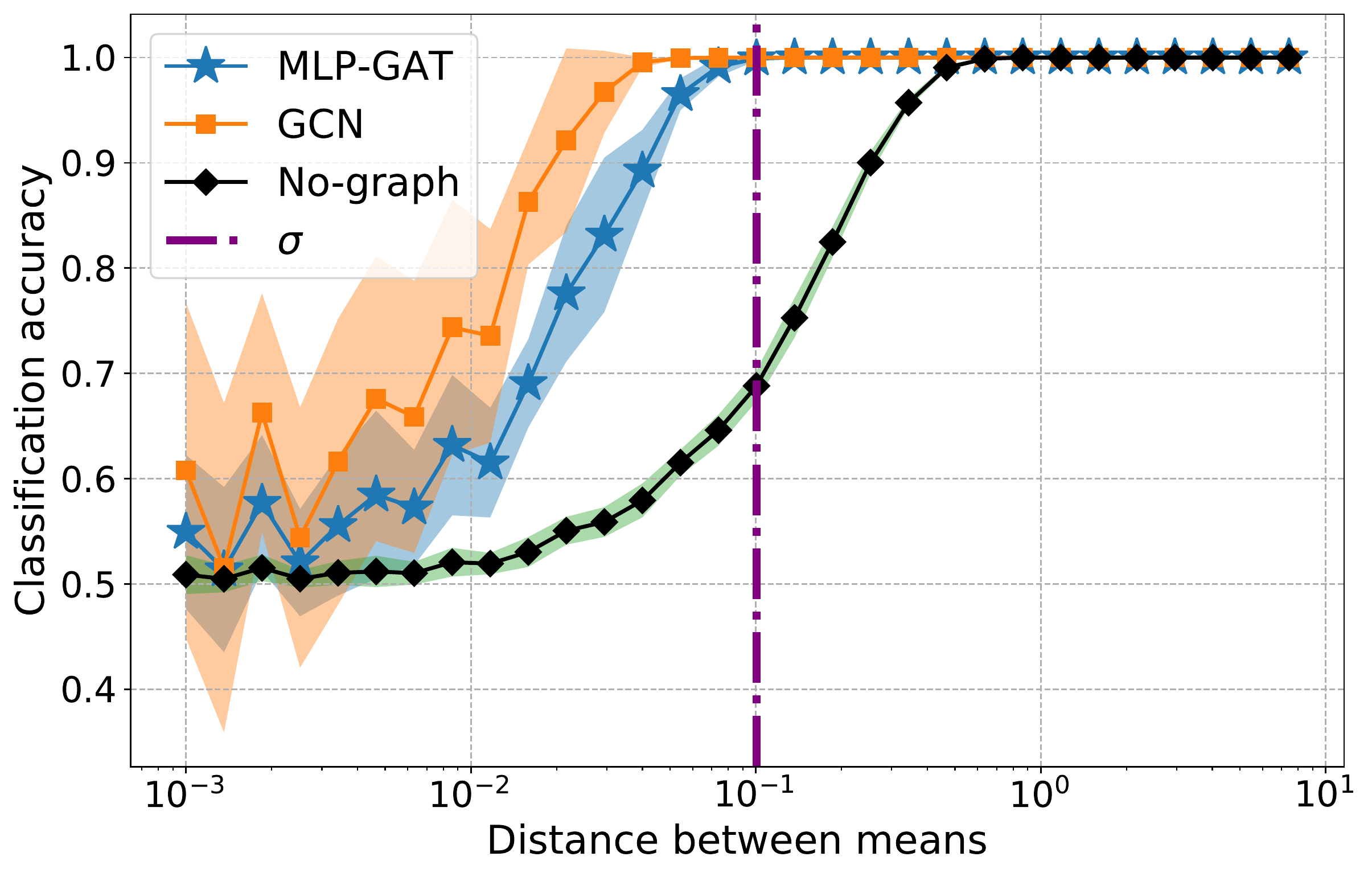}
         \caption{Node classification accuracy}
         \label{fig:node_c_MLPGAT_var_dist}
     \end{subfigure}
     \caption{Attention coefficients of MLP-GAT and GAT, node and edge classification as a function of the distance between the means. Shaded areas show standard deviation. When there is a sufficient distance between the means, attention coefficients of MLP-GAT demonstrate a nice separation while GAT does not.}\label{fig:dist_means}
\end{figure}

\subsection{Real data}\label{subsec:real_data}
For the experiment on real data, we illustrate the attention coefficients, node and edge classification for MLP-GAT as a function of the distance between the means. We use popular real-world graph datasets Cora, PubMed, and CiteSeer collected by PyTorch Geometric~\cite{FL2019} and ogbn-arxiv from Open Graph Benchmark~\cite{HFZDRLCL20}. The datasets come with multiple classes, however, for each of our experiments, we do a one-v.s.-all classification for a single class. This is a semi-supervised problem, only a fraction of the training nodes have labels. The rest of the nodes are used for measuring prediction accuracy. To control the distance between the means of the problem we use the true labels to determine the class of each node and then we compute the empirical mean for each class. We subtract the empirical means from their corresponding classes and we also add means $\bmu$ and $-\bmu$ to each class, respectively. This modification can be thought of as translating the mean of the distribution of the data for each class.

The results of this experiment are shown in Figure~\ref{fig:real_data}. For Cora, PubMed, and CiteSeer we show results for class 0 since these are small datasets, each dataset contains at most 7 classes, and the classes have similar sizes. In our experiments on other classes, we observed that the results are similar. For ogbn-arxiv we show results for the largest class (i.e. class 16) since this is a larger dataset which has 40 imbalanced classes. Picking a large class makes the one-v.s.-all classification task more balanced. In our experiments on other classes having similar sizes, we obtained similar results. We note that in the real data, we also observe similar behavior of MLP-GAT in the easy and hard regimes as for the synthetic data. In particular, for all datasets as the distance of means increases, MLP-GAT is able to accurately classify intra-class and inter-class edges, see Figures~\ref{fig:edge_c_cora},~\ref{fig:edge_c_pubmed} and~\ref{fig:edge_c_citeseer}. Moreover, as the distance between the means increases, the average intra-class $\gamma$ becomes much larger than the average inter-class $\gamma$, see Figures~\ref{fig:gammas_MLPGAT_cora},~\ref{fig:gammas_MLPGAT_pubmed},~\ref{fig:gammas_MLPGAT_citeseer}, and~\ref{fig:gammas_MLPGAT_ogbn_arxiv}, and the model is able to classify the nodes accurately, see Figures~\ref{fig:node_c_MLPGAT_cora},~\ref{fig:node_c_MLPGAT_pubmed},~\ref{fig:node_c_MLPGAT_citeseer}, and~\ref{fig:node_c_MLPGAT_ogbn_arxiv}. On the contrary, in the same figures, we observe that as the distance between the means decreases then MLP-GAT is not able to separate intra-class from inter-class edges, the averaged $\gamma$ are very close to uniform coefficients and the model can't classify the nodes accurately.

\begin{figure}[ht!]
     \centering
     \begin{subfigure}[b]{0.329\textwidth}
         \centering
         \includegraphics[width=\textwidth]{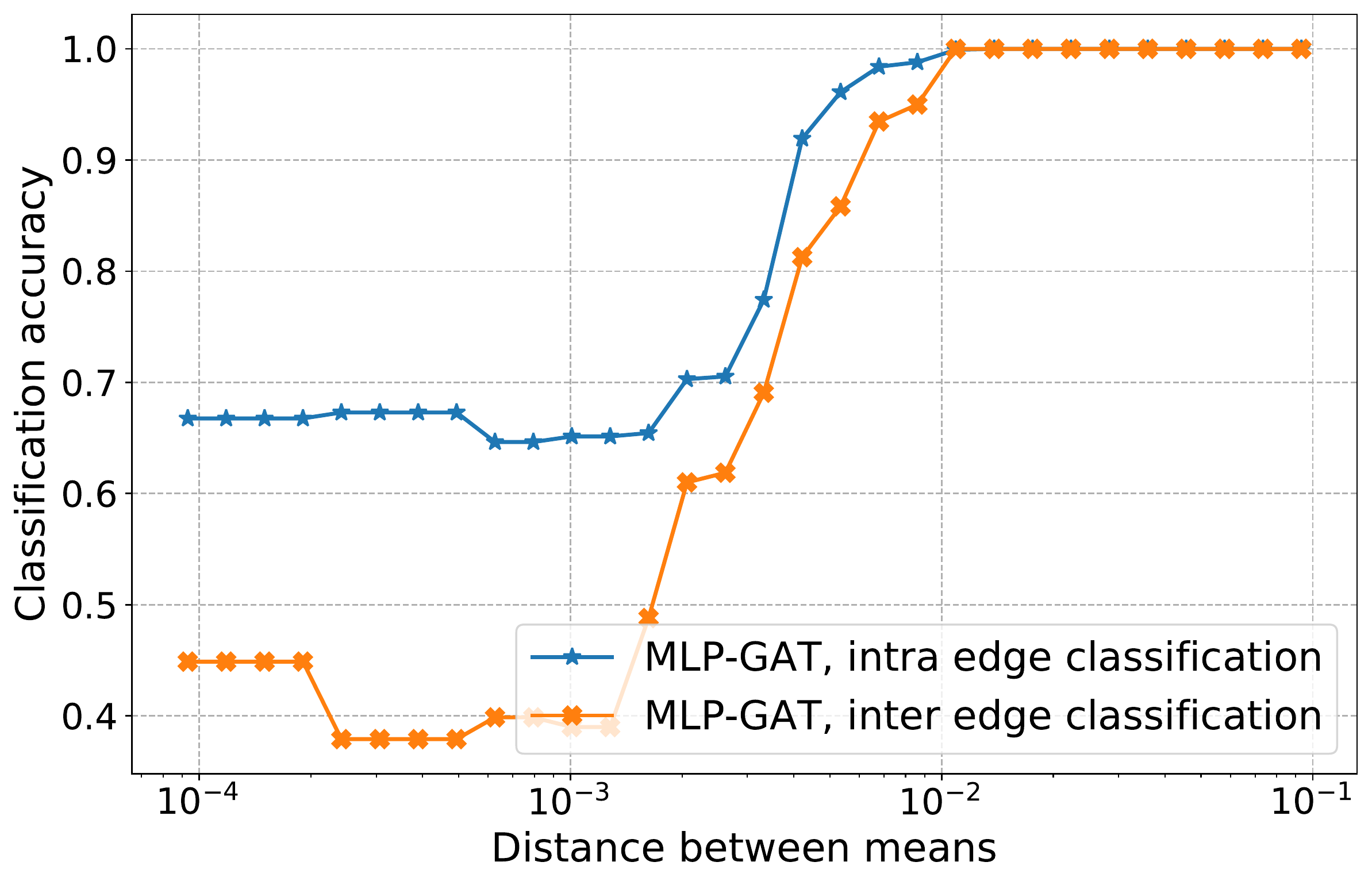}
         \caption{Edge class., Cora}
         \label{fig:edge_c_cora}
     \end{subfigure}%
     \begin{subfigure}[b]{0.329\textwidth}
         \centering
         \includegraphics[width=\textwidth]{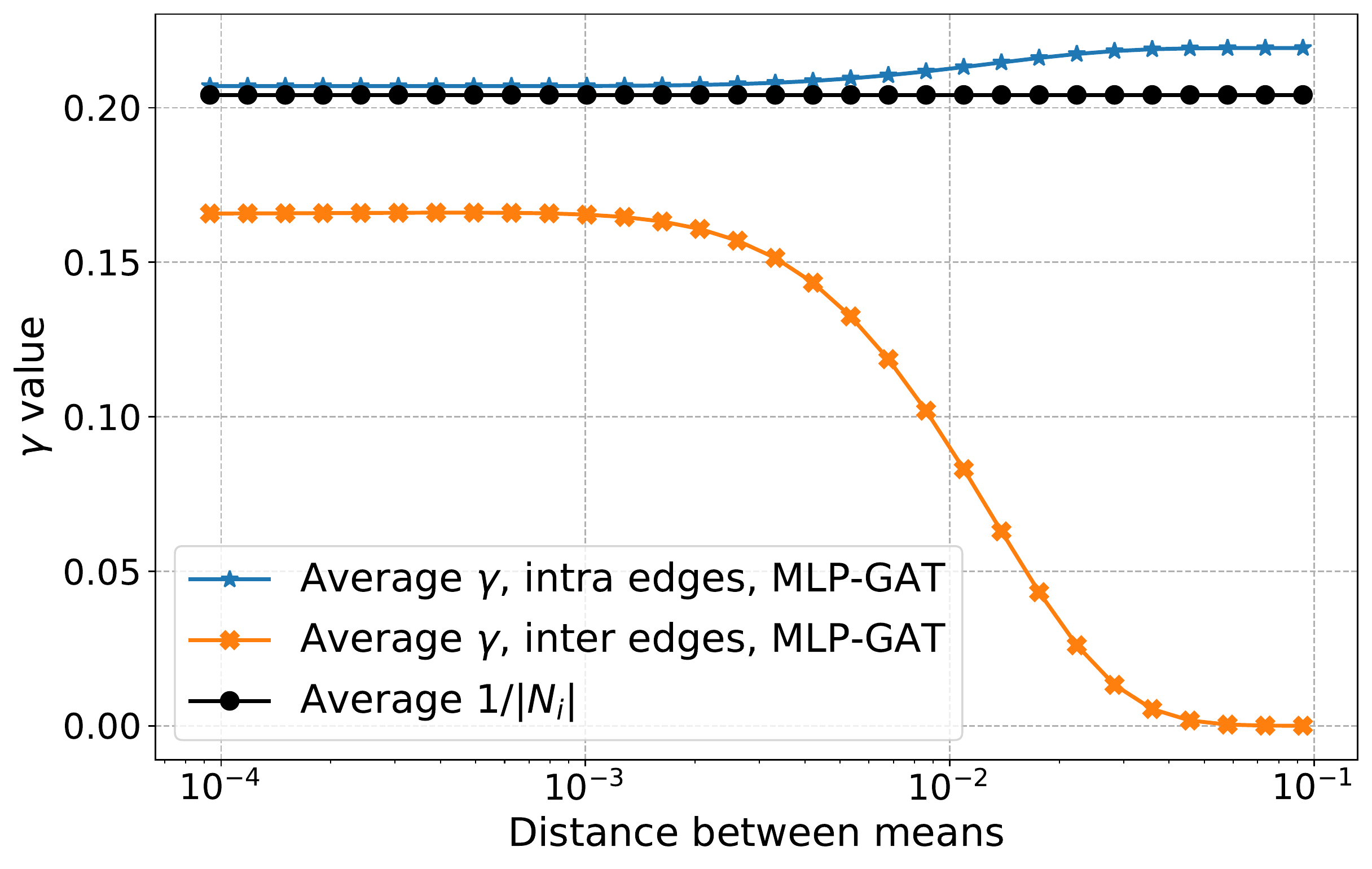}
         \caption{Attention coef., Cora}
         \label{fig:gammas_MLPGAT_cora}
     \end{subfigure}%
     \begin{subfigure}[b]{0.329\textwidth}
         \centering
         \includegraphics[width=\textwidth]{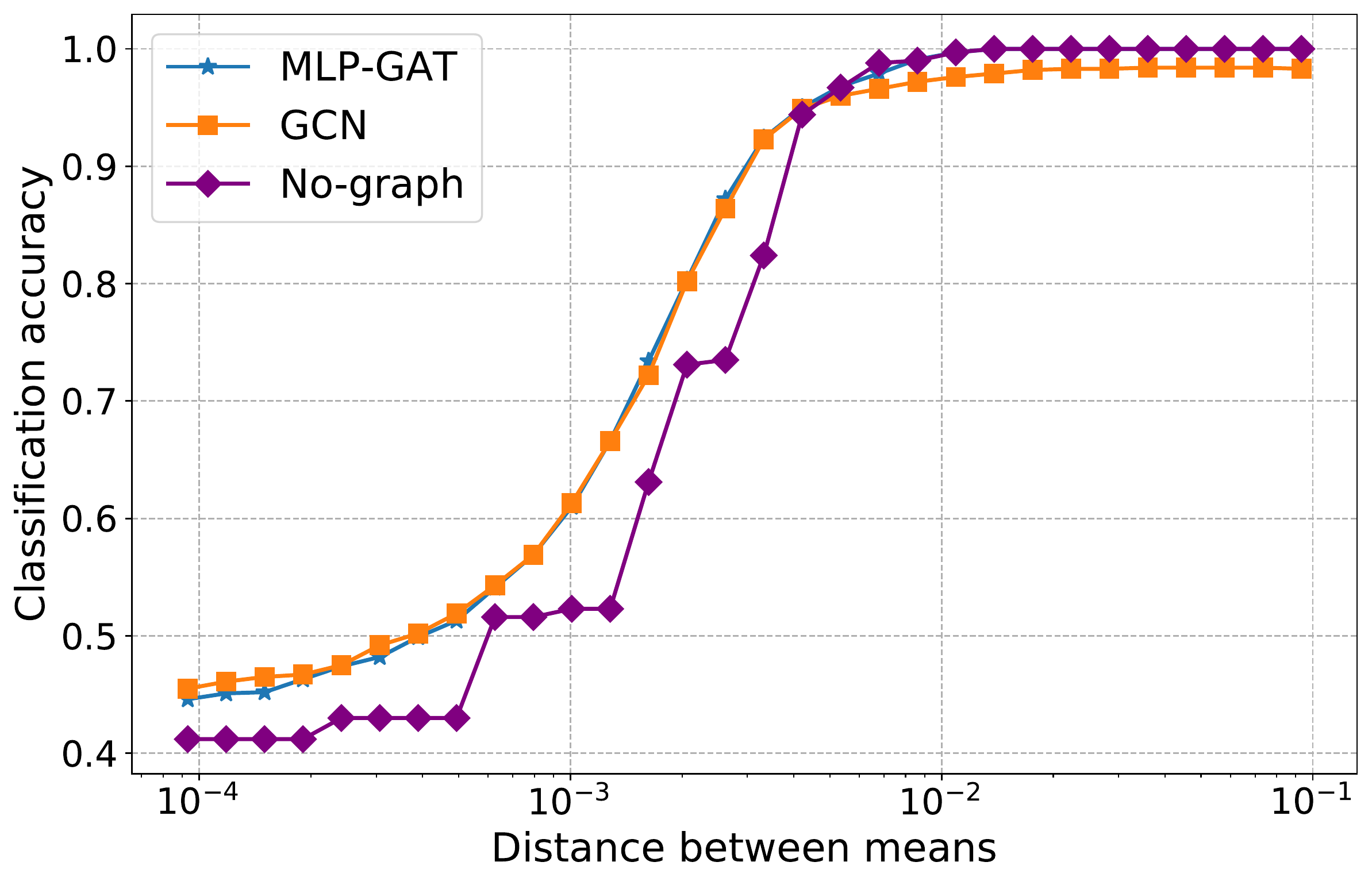}
         \caption{Node class., Cora}
         \label{fig:node_c_MLPGAT_cora}
     \end{subfigure}
     \begin{subfigure}[b]{0.329\textwidth}
         \centering
         \includegraphics[width=\textwidth]{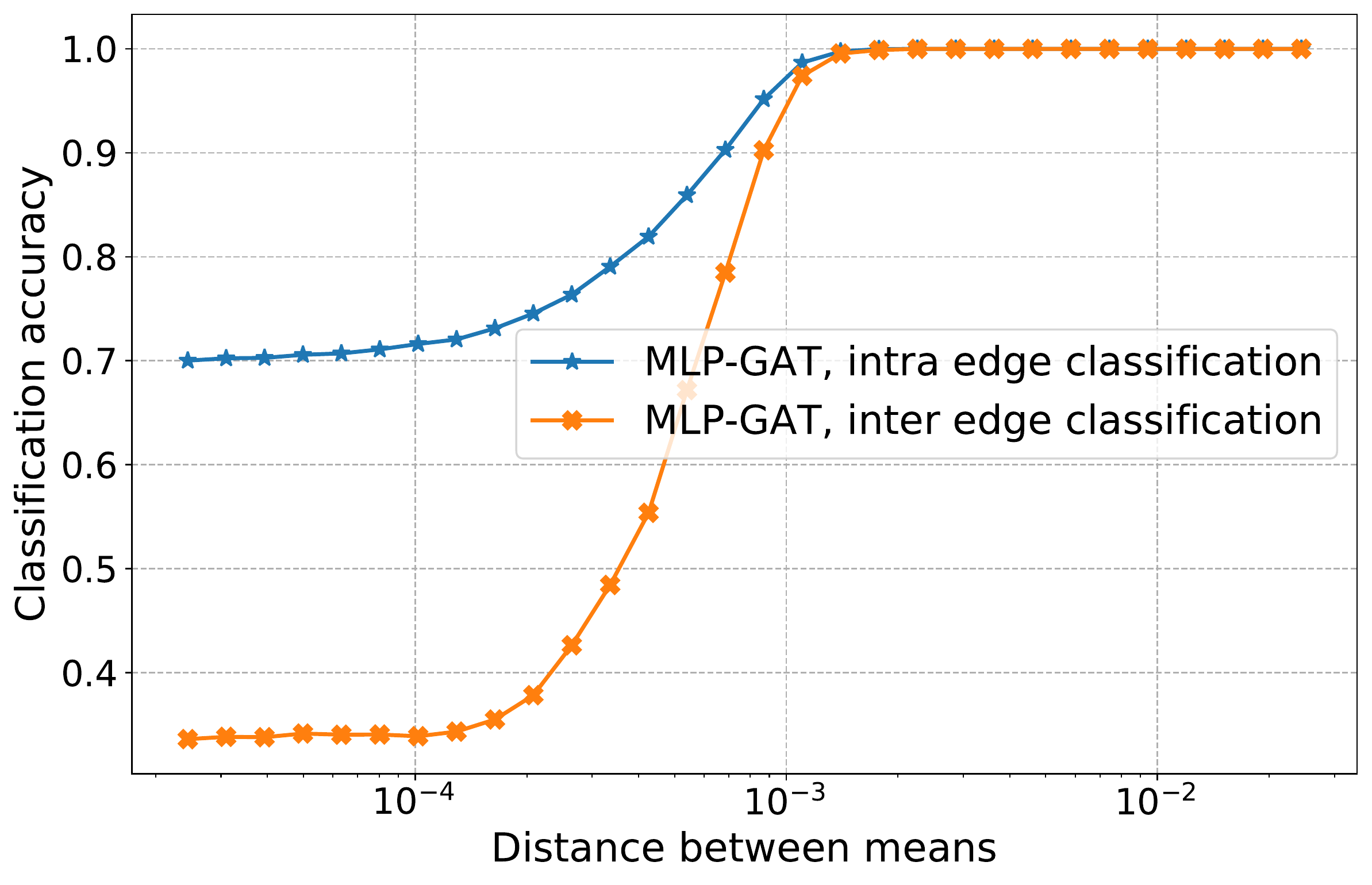}
         \caption{Edge class., PubMed}
         \label{fig:edge_c_pubmed}
     \end{subfigure}%
     \begin{subfigure}[b]{0.329\textwidth}
         \centering
         \includegraphics[width=\textwidth]{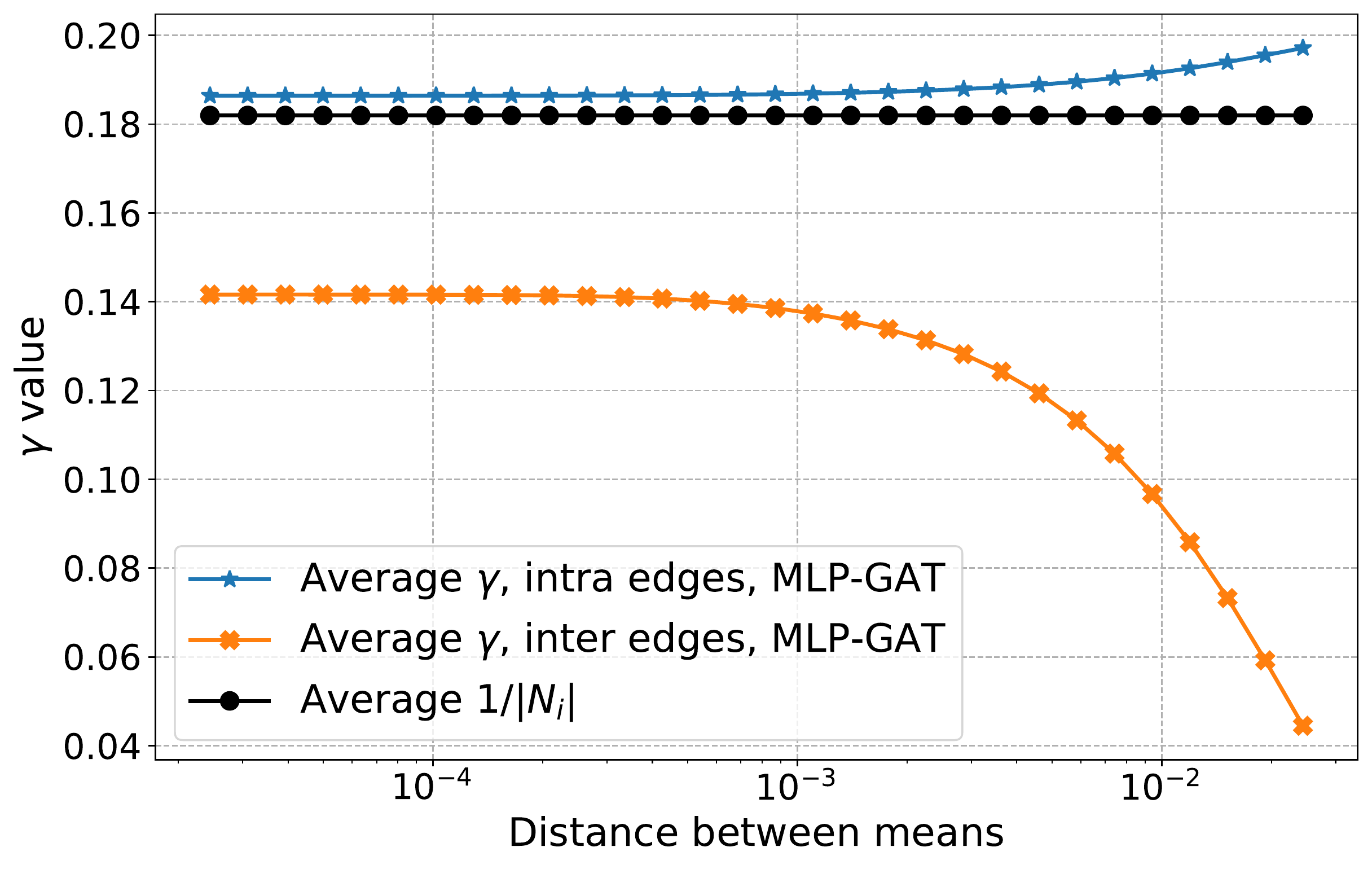}
         \caption{Attention coef., PubMed}
         \label{fig:gammas_MLPGAT_pubmed}
     \end{subfigure}%
     \begin{subfigure}[b]{0.329\textwidth}
         \centering
         \includegraphics[width=\textwidth]{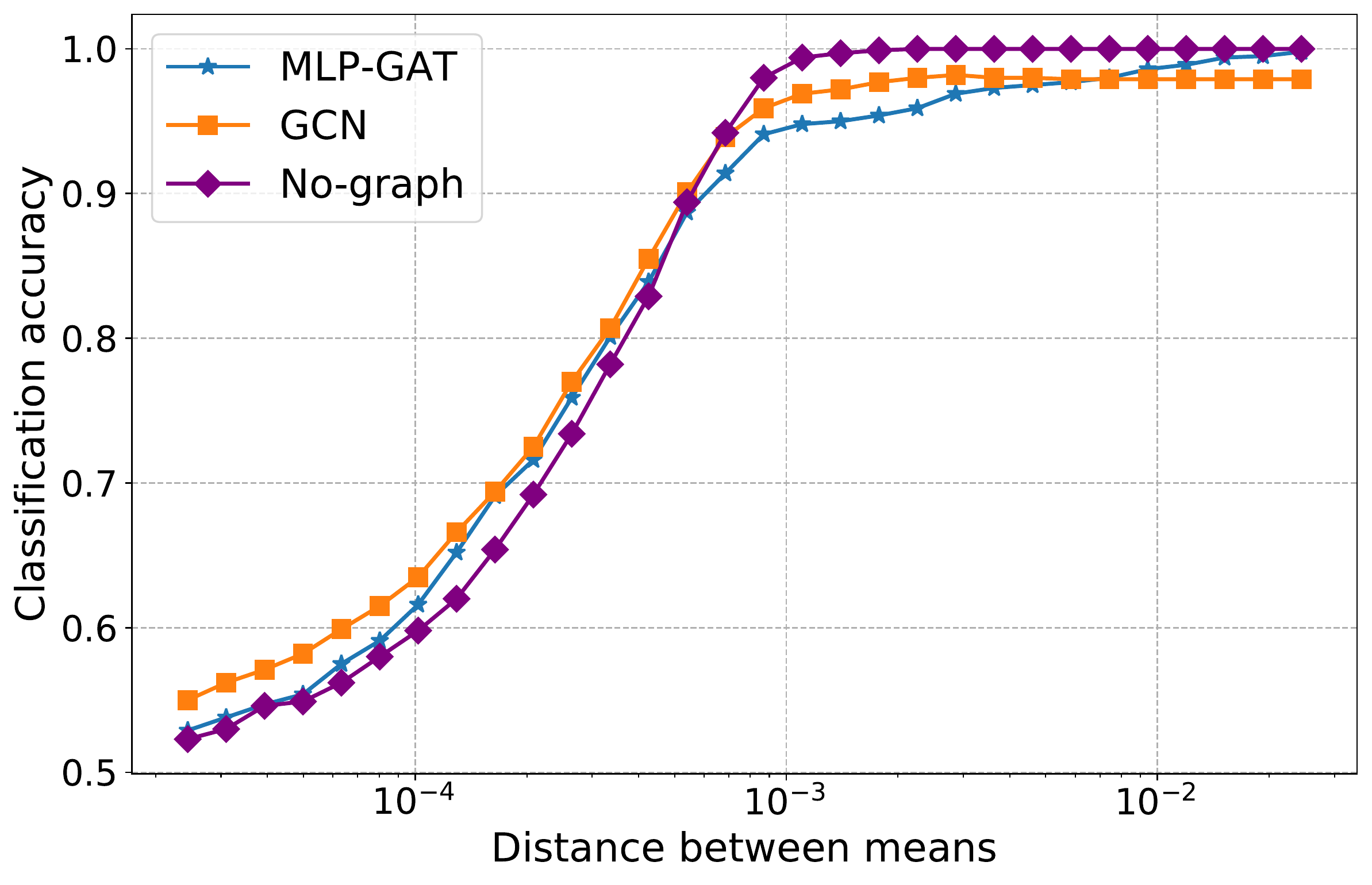}
         \caption{Node class., PubMed}
         \label{fig:node_c_MLPGAT_pubmed}
     \end{subfigure}
     \begin{subfigure}[b]{0.329\textwidth}
         \centering
         \includegraphics[width=\textwidth]{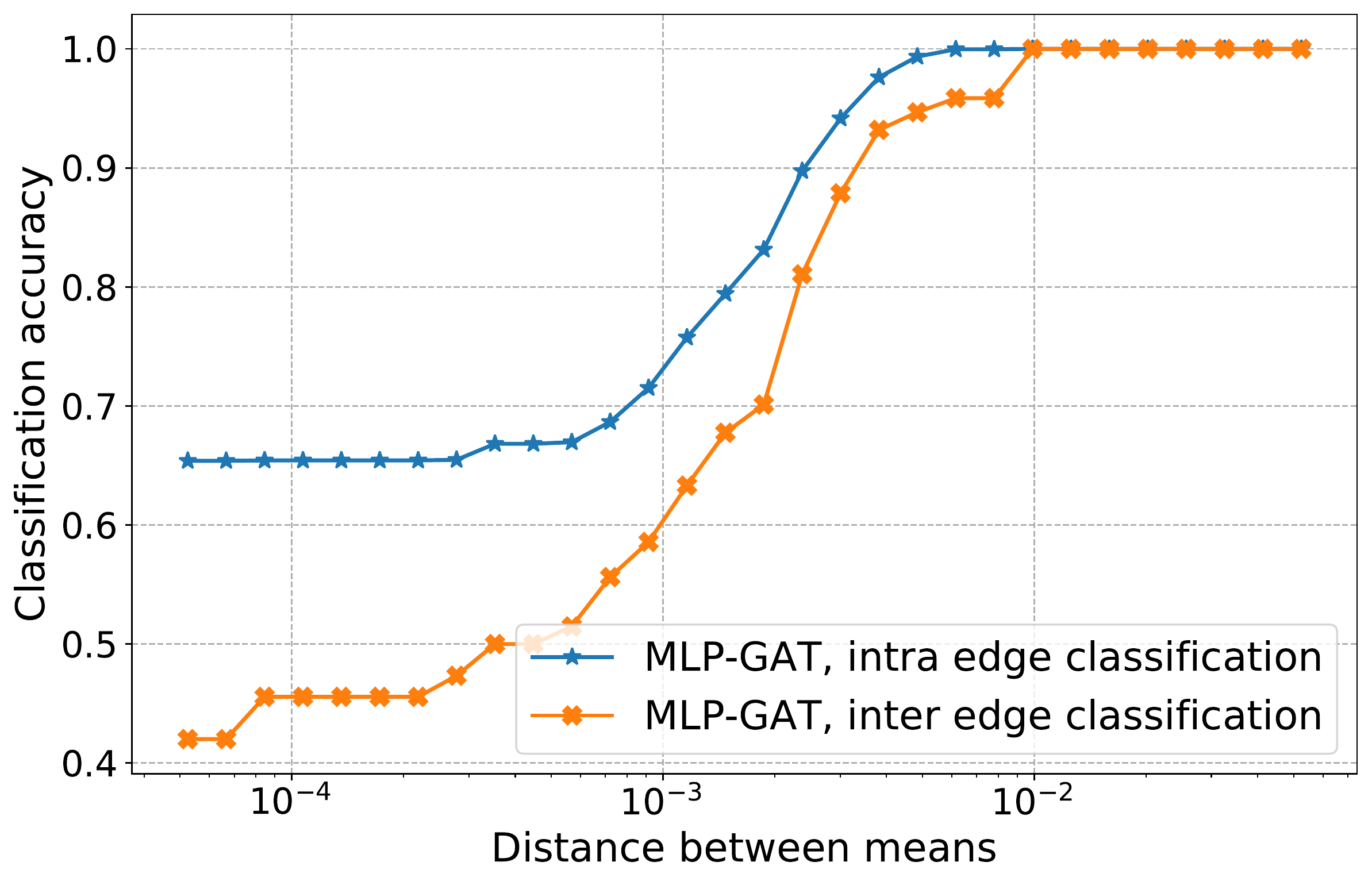}
         \caption{Edge class., CiteSeer}
         \label{fig:edge_c_citeseer}
     \end{subfigure}%
     \begin{subfigure}[b]{0.329\textwidth}
         \centering
         \includegraphics[width=\textwidth]{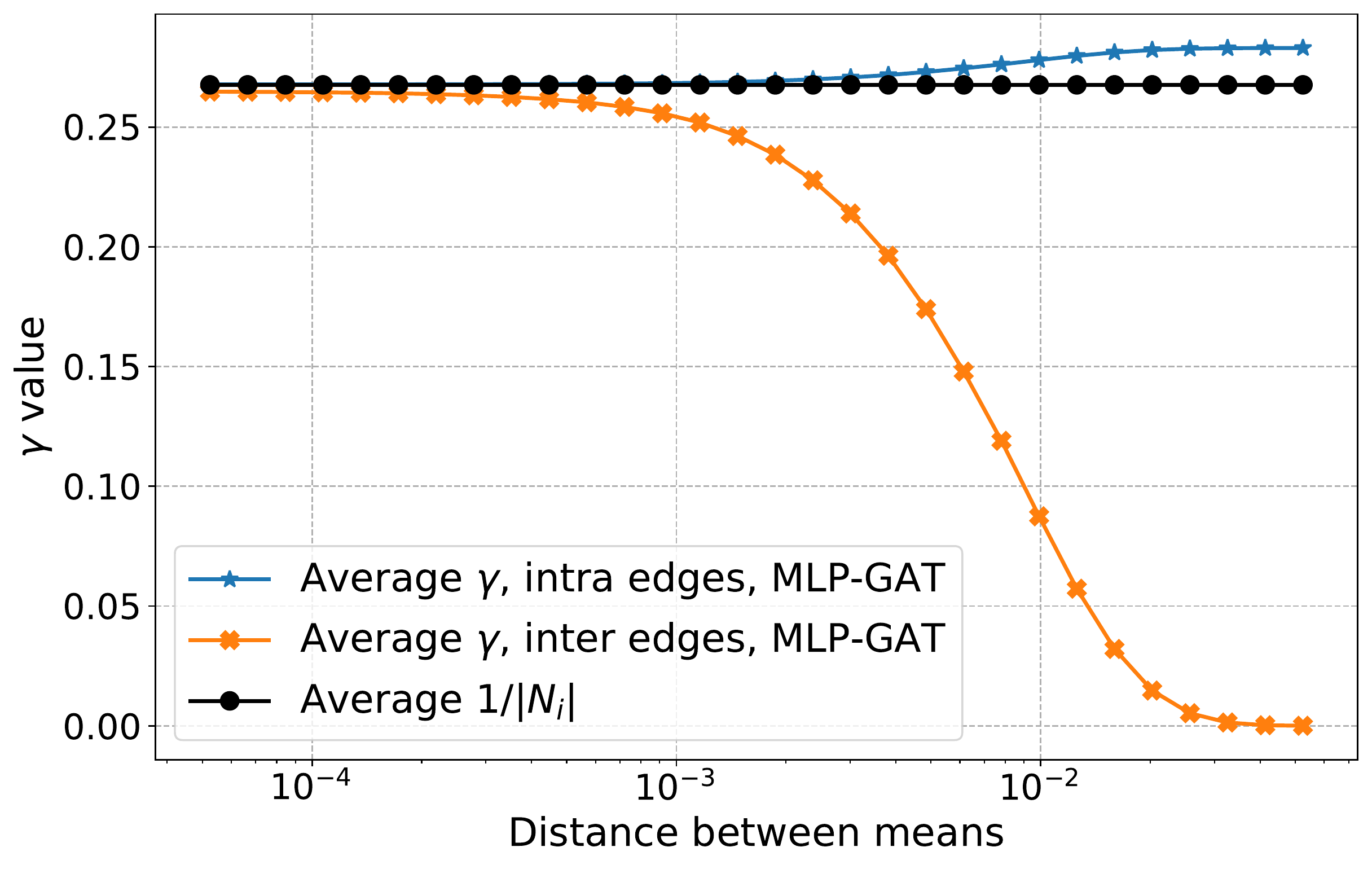}
         \caption{Attention coef., CiteSeer}
         \label{fig:gammas_MLPGAT_citeseer}
     \end{subfigure}%
     \begin{subfigure}[b]{0.329\textwidth}
         \centering
         \includegraphics[width=\textwidth]{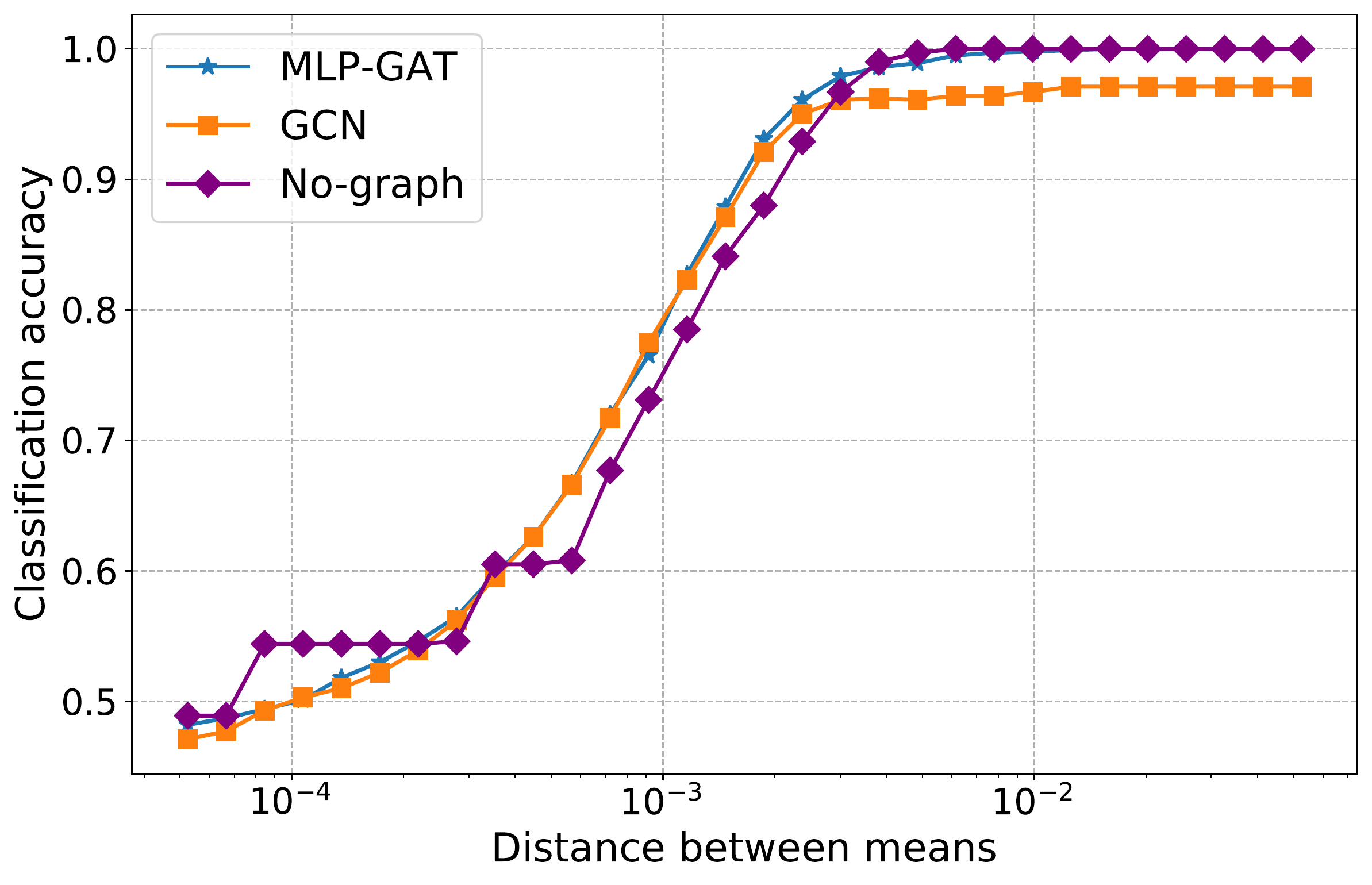}
         \caption{Node class., CiteSeer}
         \label{fig:node_c_MLPGAT_citeseer}
     \end{subfigure}
     \begin{subfigure}[b]{0.329\textwidth}
         \centering
         \includegraphics[width=\textwidth]{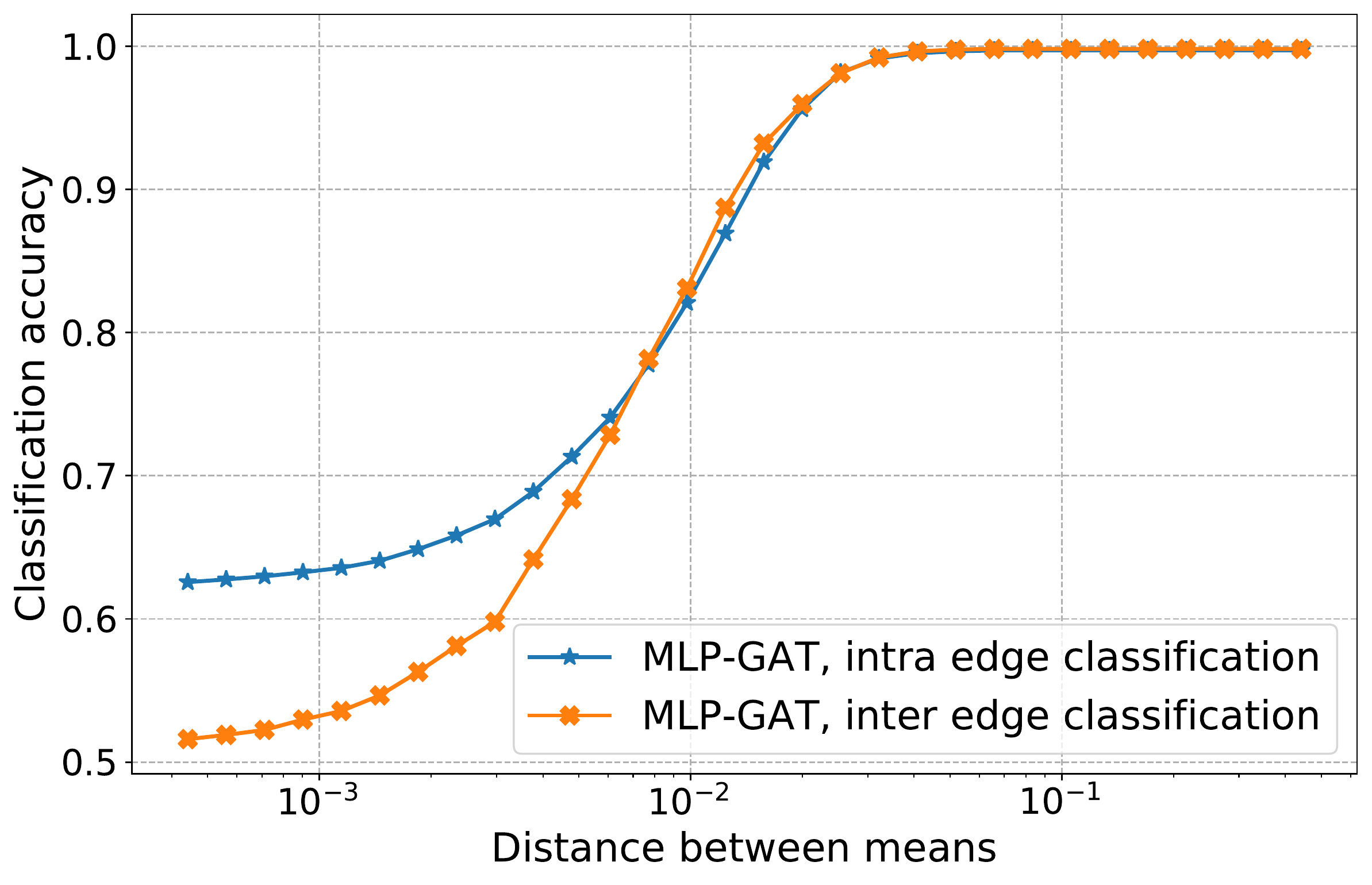}
         \caption{Edge class., ogbn-arxiv}
         \label{fig:edge_c_ogbn_arxiv}
     \end{subfigure}%
     \begin{subfigure}[b]{0.329\textwidth}
         \centering
         \includegraphics[width=\textwidth]{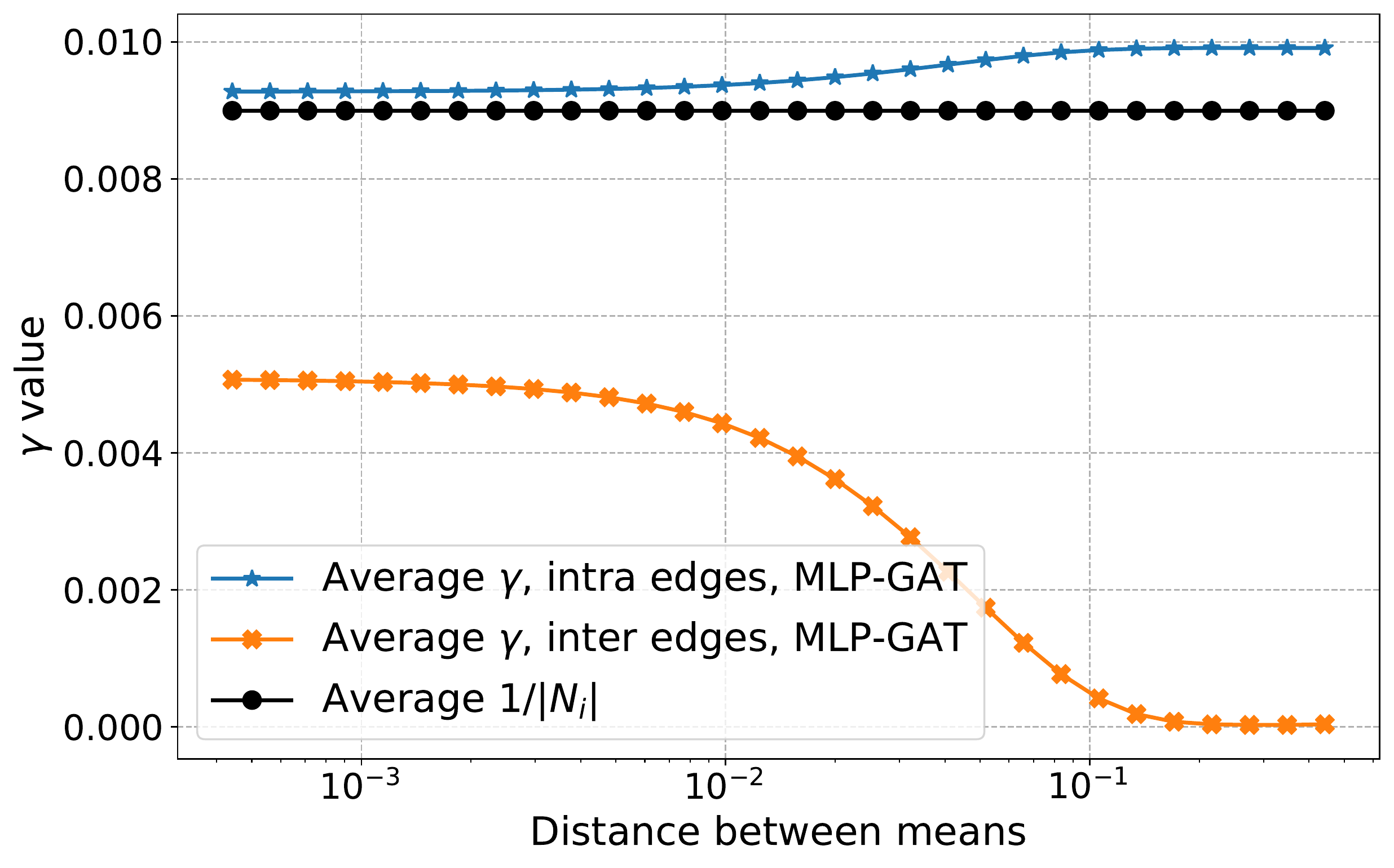}
         \caption{Attention coef., ogbn-arxiv}
         \label{fig:gammas_MLPGAT_ogbn_arxiv}
     \end{subfigure}%
     \begin{subfigure}[b]{0.329\textwidth}
         \centering
         \includegraphics[width=\textwidth]{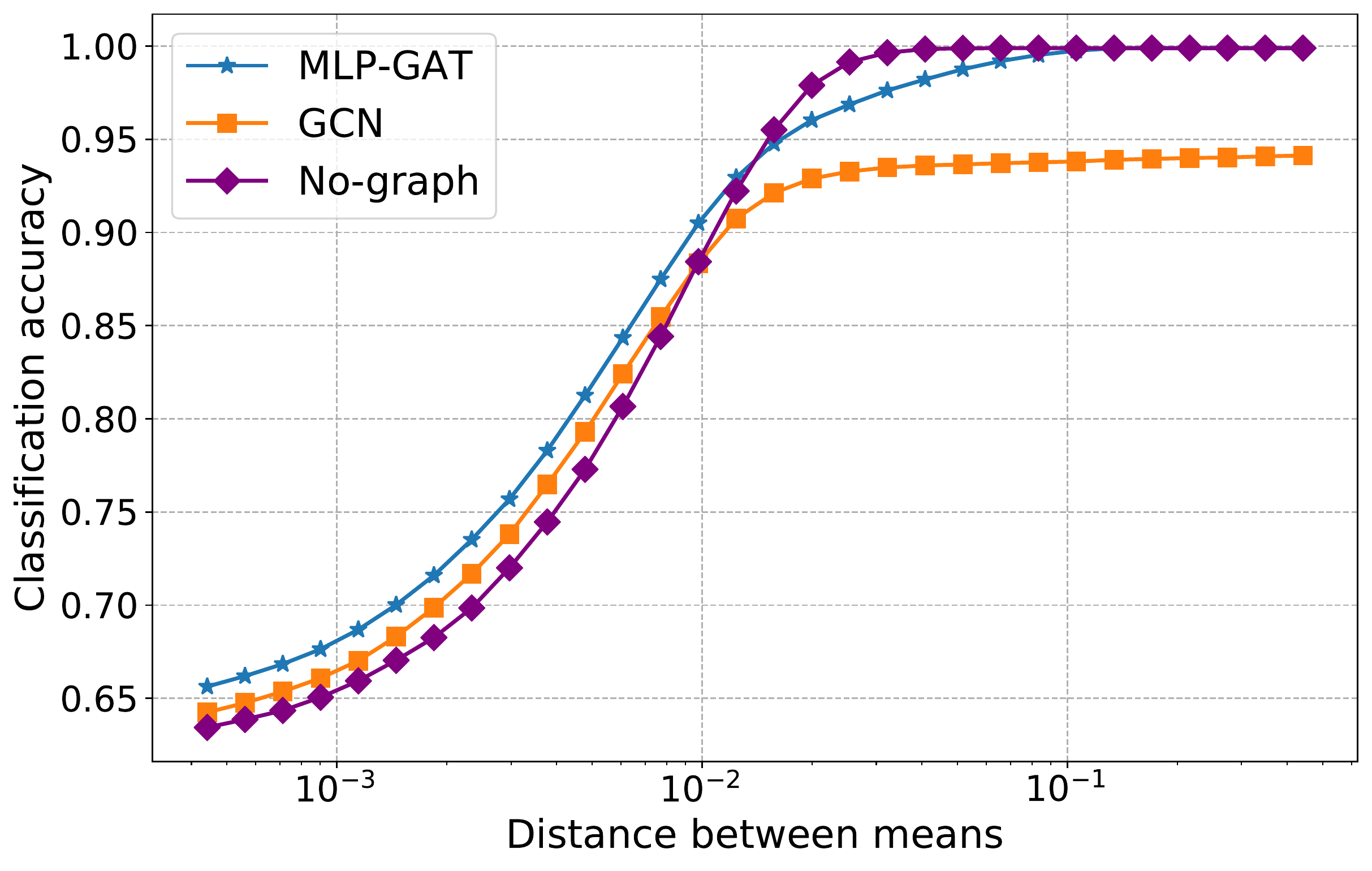}
         \caption{Node class., ogbn-arxiv}
         \label{fig:node_c_MLPGAT_ogbn_arxiv}
     \end{subfigure}
    \caption{Attention coefficients, node and edge classification for MLP-GAT as a function of the distance between the means for real data.}
    \label{fig:real_data}
\end{figure}
Note that Figure~\ref{fig:real_data} does not show the standard deviation for the attention coefficients $\gamma$. We show the standard deviation of $\gamma$ in Figure~\ref{fig:gamma_std}. We observe that the standard deviation is higher than what we observed in the synthetic data. In particular, it can be more than half of the averaged $\gamma$. This is to be expected since for the real data the degrees of the nodes do not concentrate as well. In Figure~\ref{fig:gamma_std} we show that the standard deviation of the uniform coefficients $1/|N_i|$ is also high. For Cora, PubMed, and CiteSeer, the standard deviation for intra-class $\gamma$ is similar to that of $1/|N_i|$, while the deviation for inter-class $\gamma$ is large for a small distance between the means, but it gets much smaller as the distance increases. For ognb-arxiv, the standard deviation of $1/|N_i|$ is particularly high. This implies that the degree distribution of nodes of ogbn-arxiv has a heavy tail, which could potentially result from the graph structure being noisier than other datasets and also explain GCN's relatively much worse performance when the distance between the means is large.
\begin{figure}[ht!]
     \centering
     \begin{subfigure}[b]{0.45\textwidth}
         \centering
         \includegraphics[width=.9\textwidth]{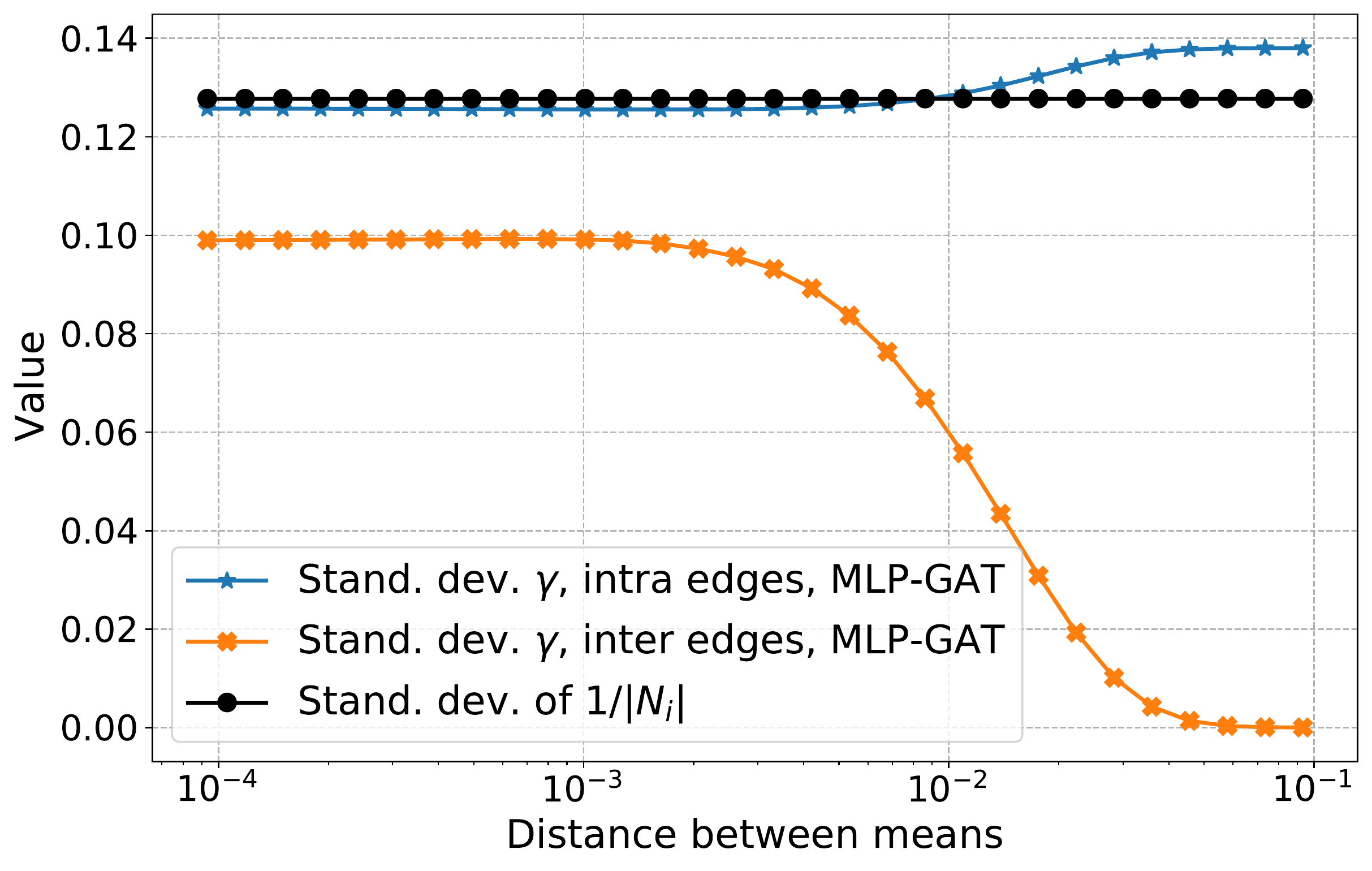}
         \caption{Cora}
         \label{fig:gamma_std_cora}
     \end{subfigure}%
     \begin{subfigure}[b]{0.45\textwidth}
         \centering
         \includegraphics[width=.9\textwidth]{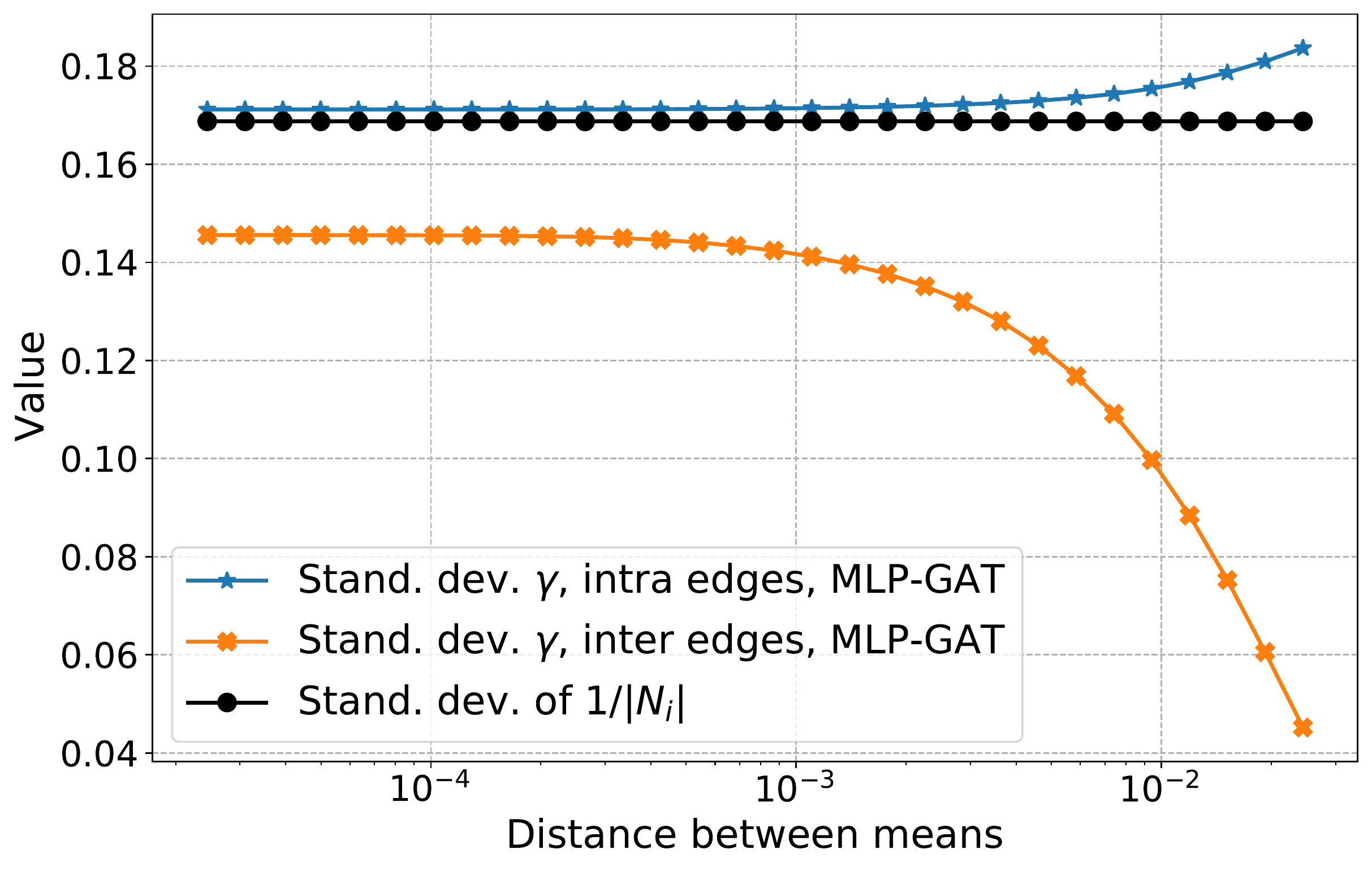}
         \caption{PubMed}
         \label{fig:gamma_std_pubmed}
     \end{subfigure}
     \begin{subfigure}[b]{0.45\textwidth}
         \centering
         \includegraphics[width=.9\textwidth]{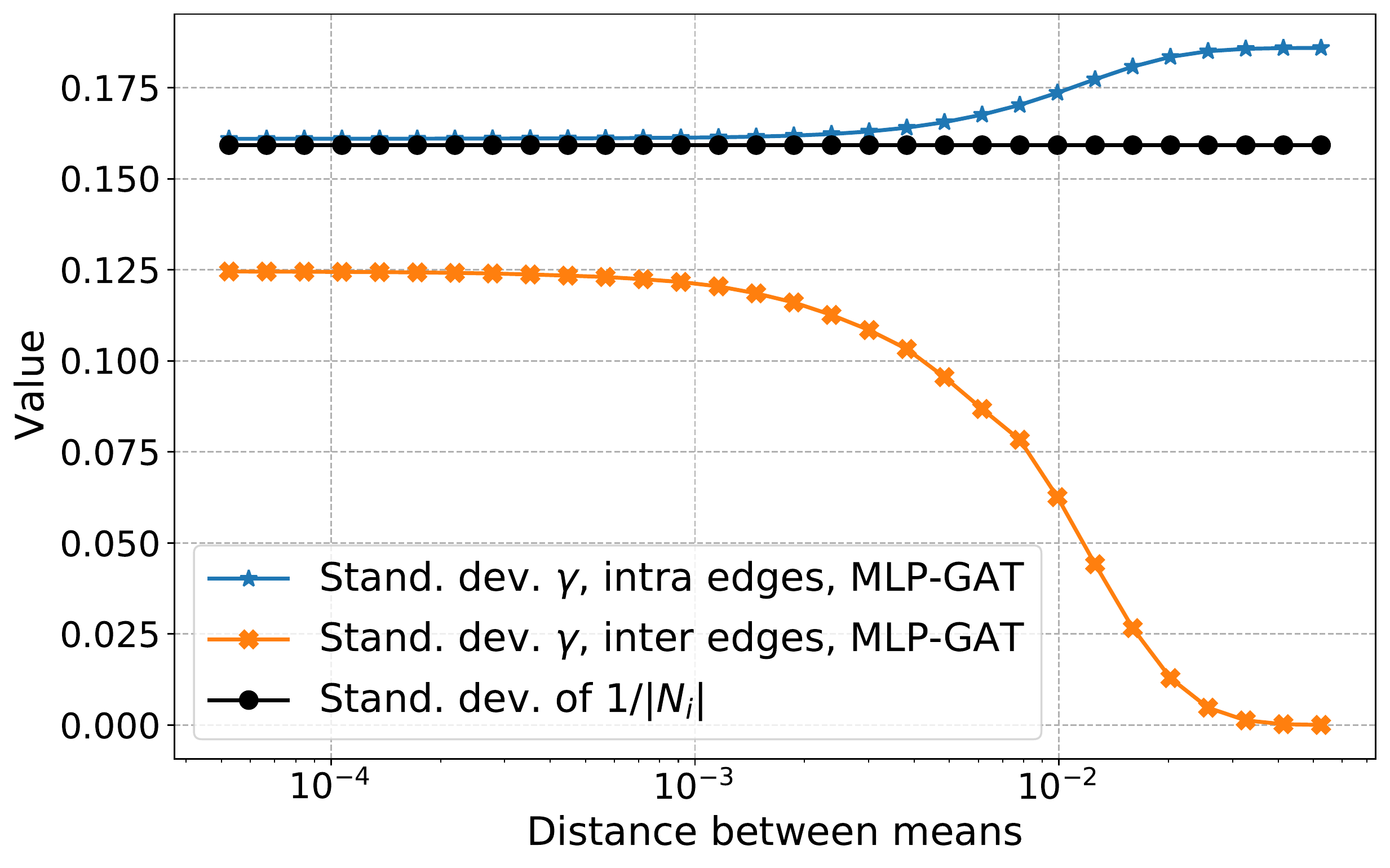}
         \caption{CiteSeer}
         \label{fig:gamma_std_citeseer}
     \end{subfigure}%
          \begin{subfigure}[b]{0.45\textwidth}
         \centering
         \includegraphics[width=.9\textwidth]{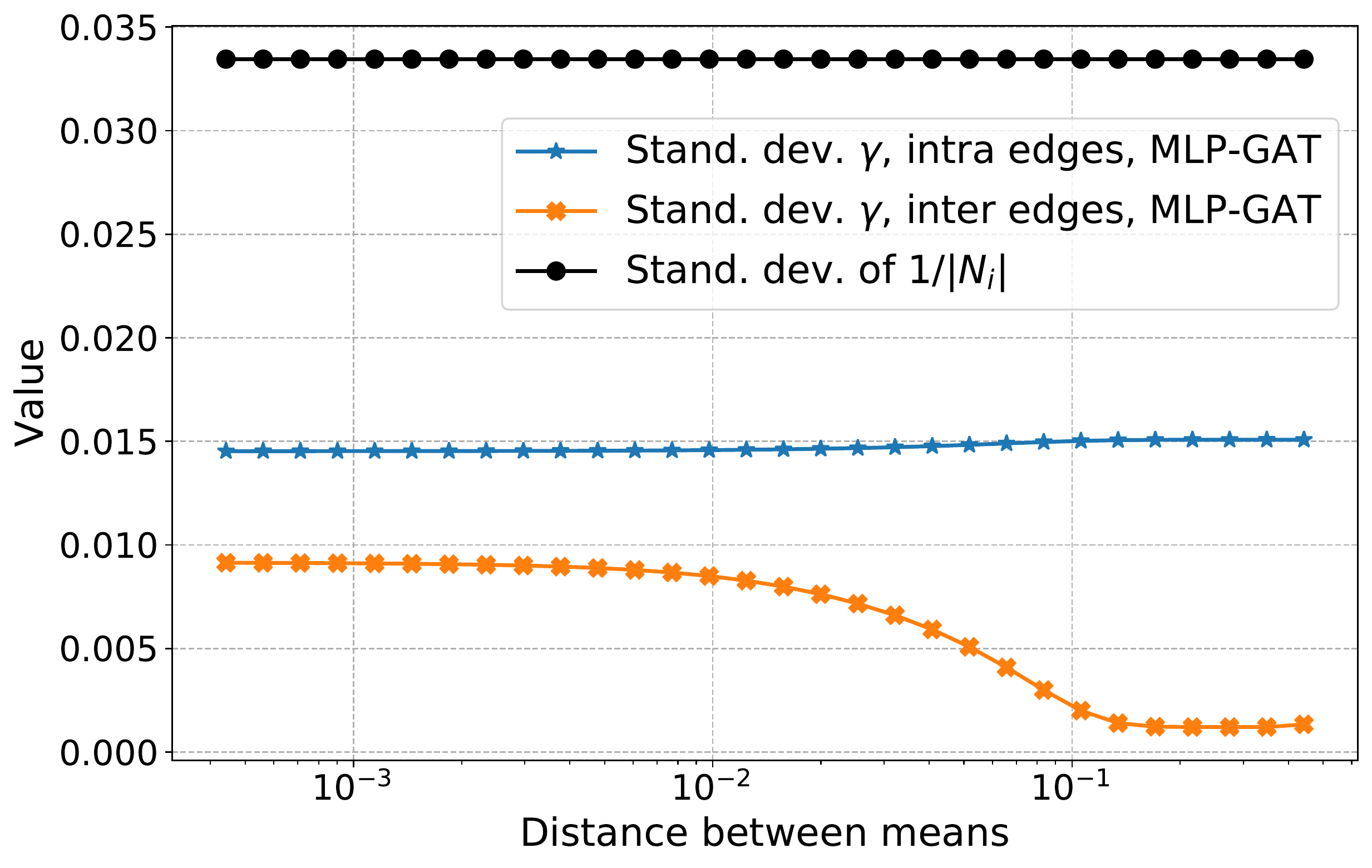}
         \caption{ogbn-arxiv}
         \label{fig:gamma_std_ogbn_arxiv}
     \end{subfigure}
     \caption{Standard deviation for attention coefficients of MLP-GAT.}\label{fig:gamma_std}
\end{figure}

\section{Summary of implications for practitioners}\label{sec:implications}

While this work focuses on theoretical understanding of graph attention's capability for edge and node classification using the CSBM generative model, our findings yield a series of potential suggestions for GNN practitioners. In this section we provide some interesting practical implications of our results.

\subsection{Why graph attention? Benefits of graph attention's robustness to structural noise} 

When the graph is very noisy, e.g. when a node has a similar number of neighbors from the same class and from different classes, simple graph convolution will mix up node features and thus make nodes from different classes indistinguishable. In this case, simple graph convolution can do more harm than good. However, in practice, it is difficult to determine how noisy the graph is, or if the graph is even useful at all. This could pose a challenge in choosing an architecture. Our results in Theorem~\ref{thm:gat_linear} and Corollary~\ref{cor:model_rank} imply that graph attention has the ability to dramatically reduce the impact of a noisy graph, in a way such that the output is at least as good as the output from the best linear classifier (on the input) which does not rely on the graph. This shows that graph attention convolution is more robust against structural noise in the graph, and hence on noisy graphs, it is strictly better than simple graph convolution.

\subsection{Which attention architecture? Benefits of multi-layer attention architecture} 

In this work, we are able to obtain positive results for graph attention by using the two-layer MLP attention architecture in \eqref{eq:psi_ansatz}. This is different from the original GAT which uses a single-layer attention architecture~\cite{Velickovic2018GraphAN}. In our analyses and empirical experiments, we found that the original single-layer attention does not have the important properties required for obtaining positive results (e.g. Theorem~\ref{thm:edge_separation_easy}, Corollary~\ref{cor:node_separation_easy}, Theorem~\ref{thm:gat_linear}, Corollary~\ref{cor:gat_linear_recovery}) for graph attention convolution. Coincidentally, this aligns with the findings of \cite{BAY21}, where the authors study limitations of the original GAT architecture from a different perspective than ours, and they propose a new architecture termed GATv2. The two-layer MLP attention architecture that we consider in \eqref{eq:psi_ansatz} encompasses GATv2 as a special case. Therefore, the two-layer MLP attention architecture can be a good candidate to consider when practitioners search for a suitable graph attention architecture for their specific downstream tasks. On the other hand, our results in Section~\ref{subsec:hard} imply that even the two-layer MLP attention architecture (and hence GATv2) has limitations when the node features are noisy. To fix that, a potential solution is to incorporate additional information such as edge features, which we discuss next.
    
\subsection{Will additional information help? Benefits of incorporating good edge features} 

Even though we do not consider edge features in our analyses, our results in Theorem~\ref{thm:good_psi_negative} imply that good attention functions that are able to classify the edges independently from the node features can be very helpful, as they help reduce the threshold under which graph attention convolution would fail to separate the nodes. One potential way to obtain good attention functions that behave like the one given in \eqref{eq:good_psi} is by incorporating good edge features. Furthermore, given our result in Theorem~\ref{thm:edge_separation_hard}, which says that graph attention based on noisy node features cannot perfectly classify the edges, the importance of incorporating informative edge features that are more indicative of edge class memberships (if they are accessible in practice) into the attention mechanism is more pronounced.

\section{Conclusion and future work} 

In this work, we study the impact of graph attention on edges and its implications for node classification. We show that graph attention improves robustness to noise in the graph structure. We also show that graph attention may not be very useful in a ``hard'' regime where the node features are noisy. Our work shows that single-layer graph attention convolution has limited power at distinguishing intra-class from inter-class edges. Given the empirical successes of graph attention and its many variants, a promising future work is to study the power of multi-layer graph attention convolutions for distinguishing intra-class and inter-class edges. Moreover, our negative results in Section~\ref{subsec:hard} for edge/node classification pertains to perfect classification and almost perfect classification. In practice, misclassification of a small constant fraction of nodes/edges is often inevitable, but nonetheless useful. Therefore, an interesting future line of work is to characterize the threshold under which graph attention is going to misclassify a certain proportion of nodes. Finally, variants of graph attention networks have been successfully used in tasks other than node classification, such as link prediction and graph classification. These tasks are typically solved by architectures that add a final aggregation layer which combines node representations generated from graph attention convolution. It is an interesting future direction to develop a good understanding of the benefits and limitations of the graph attention mechanism for those tasks.

\section*{Acknowledgement}

K. Fountoulakis would like to acknowledge the support of the Natural Sciences and Engineering Research Council of Canada (NSERC). Cette recherche a \'et\'e financ\'ee par le Conseil de recherches en sciences naturelles et en g\'enie du Canada (CRSNG), [RGPIN-2019-04067, DGECR-2019-00147]. 

A. Jagannath acknowledges the support of the Natural Sciences and Engineering Research Council of Canada (NSERC). Cette recherche a \'et\'e financ\'ee par le Conseil de recherches en sciences naturelles et en g\'enie du Canada (CRSNG),  [RGPIN-2020-04597, DGECR-2020-00199].

\newpage

\appendix

\section{Proofs}

We define the following high-probability events which will be used in some proofs. Each of these events holds with probability at least $1-o(1)$, which follows from straightforward applications of Chernoff bound and union bound, e.g., see \cite{BFJ2021}.

\begin{definition}\label{def:high_prob_events}
Define the following events over the randomness of $\bA$ and $\{\eps_i\}_{i\in[n]}$ and $\{\bX_i\}_{i\in[n]}$,
\begin{itemize}
    \item $\calbE_1$ is the event that $|C_0|=\frac{n}{2}\pm O(\sqrt{n\log n})$ and $|C_1|=\frac{n}{2}\pm O(\sqrt{n\log n})$.
    \item $\calbE_2$ is the event that for each $i\in [n]$, $\bD_{ii}=\frac{n(p+q)}{2}\left(1\pm \frac{10}{\sqrt{\log n}}\right)$.
    \item $\calbE_3$ is the event that for each $i \in [n]$, $|C_0\cap N_i|=\bD_{ii}\cdot \frac{(1-\eps_i)p+\eps_i q}{p+q}\left(1\pm  \frac{10}{\sqrt{\log n}}\right)$ and $|C_1\cap N_i|=\bD_{ii}\cdot \frac{(1-\eps_i)q+\eps_i p}{p+q}\left(1\pm  \frac{10}{\sqrt{\log n}}\right)$.
    \item $\calbE_4$ is the event that for each $i \in [n]$, $\left|\tilde{\bw}^T\bX_i - \Ex\left[ \tilde{\bw}^T\bX_i\right]\right| \le 10\sigma\sqrt{ \log n}$.
    \item $\calbE^*$ is the intersection of the above 4 events.
\end{itemize}
\end{definition}

\begin{lemma}[\cite{BFJ2021}]\label{lem:calE*} 
With probability at least $1-o(1)$ event $\calbE^*$ holds.
\end{lemma}

Some of our proofs also utilize the following simple observation on the mutual independence among $\{\bw^T\bg_i\}_{i\in[n]}$when $\{\bg_i\}_{i\in[n]}$ are i.i.d. Gaussian random vectors.

\begin{observation}\label{obs:w^Tg_j} 
Fix $\bw\neq 0$ in $\R^d$ and let $\bg_1,\ldots,\bg_n$ be i.i.d. drawn from $N(0,\bI)$. Then $\bw^T\bg_1, \bw^T\bg_2, \ldots, \bw^T\bg_n$ are independent.
\end{observation}
\begin{proof} 
Note that since $\bw^T\bg_i \sim N(0,\|\bw\|^2)$, it suffices to prove that the covariance $\Ex[\bw^T\bg_i\cdot \bw^T\bg_j]=0$ for all $i\neq j$. By definition, for $i \neq j$,
\begin{align*}
\Ex[\bw^T\bg_i\cdot \bw^T\bg_j]&=\Ex\left[\sum_{k\in[d]} \sum_{\ell\in [d]}\bw_k\bw_\ell\bg_{ik}\bg_{j\ell}\right] =\sum_{k\in[d]} \sum_{\ell\in [d]}\bw_k\bw_\ell\Ex[\bg_{ik}\bg_{j\ell}]=0,
\end{align*}
where the last equality follows from independence between $\bg_i$ and $\bg_j$.
\end{proof}

\subsection{Proof of Theorem~\ref{thm:edge_separation_easy}}\label{subsec:thm3proof}
We restate Theorem~\ref{thm:edge_separation_easy} for convenience.
\begin{theorem*}\label{thm:positive_pairs}
Suppose that $\|\bmu\|=\omega(\sigma \sqrt{\log n})$. Then with probability at least $1-o(1)$ over the data $(\bX,\bA) \sim \CSBM(n,p,q,\bmu,\sigma^2)$, the two-layer MLP attention architecture $\Psi$ given in \eqref{eq:psi_ansatz} and \eqref{eq:psi_ansatz_parameters} separates intra-class edges from inter-class edges.
\end{theorem*}

Recall from \eqref{eq:psi_simplified} that
\[
    \Psi(\bX_i,\bX_j) =
    \left\{
    \begin{array}{ll}
    -2R(1-\beta)\tilde\bw^T\bX_i, & \mbox{if}~ \tilde\bw^T\bX_j \le -\left|\tilde\bw^T\bX_i\right|,\\
    \hspace{3mm}2R(1-\beta)\sign(\tilde\bw^T\bX_i)\tilde\bw^T\bX_j, & \mbox{if}~ -\left|\tilde\bw^T\bX_i\right| < \tilde\bw^T\bX_j < \left|\tilde\bw^T\bX_i\right|,\\
    \hspace{3mm}2R(1-\beta)\tilde\bw^T\bX_i, & \mbox{if}~ \tilde\bw^T\bX_j \ge \left|\tilde\bw^T\bX_i\right|.
    \end{array}
    \right.
\]
We will condition on the event that $\left|\tilde{\bw}^T\bX_i - \Ex\left[ \tilde{\bw}^T\bX_i\right]\right| \le 10\sigma\sqrt{ \log n}$ for all $i \in [n]$, which holds with probability at least $1-O(1/n^{99})$ following a union bound and the Gaussian tail probability. Under this event, because $\|\bmu\| = \omega(\sigma\sqrt{\log n})$, for all $i,j  \in C_1$ we have
\begin{align*}
    \sign(\tilde\bw^T\bX_i)  = \sign(\tilde\bw^T\bX_j) = 1,\\
    \min\{\tilde\bw^T\bX_i,\tilde\bw^T\bX_j\} \ge \|\bmu\|-10\sigma\sqrt{\log n},\\
    \max\{\tilde\bw^T\bX_i,\tilde\bw^T\bX_j\} \le \|\bmu\|+10\sigma\sqrt{\log n},
\end{align*}
and hence
\begin{align*}
     \Psi(\bX_i,\bX_j) &\ge 2R(1-\beta) \cdot \min\{\tilde\bw^T\bX_i,\tilde\bw^T\bX_j\} \ge 2R(1-\beta) \cdot(\|\bmu\| - O(\sigma\sqrt{\log n})),\\
     \Psi(\bX_i,\bX_j) &\le 2R(1-\beta) \cdot \max\{\tilde\bw^T\bX_i,\tilde\bw^T\bX_j\} \le  2R(1-\beta) \cdot(\|\bmu\| + O(\sigma\sqrt{\log n})),
\end{align*}
which implies that $\Psi(\bX_i,\bX_j) = \|\bmu\| \pm O(\sigma\sqrt{\log n})$ for $i,j \in C_1$. Similarly, we get that
\begin{equation}\label{eq:psi_error_bound}
    \Psi(\bX_i,\bX_j) =
    \left\{
    \begin{array}{ll}
    \hspace{3mm}\|\bmu\| \pm O(\sigma\sqrt{\log n}) = \hspace{3mm}\|\bmu\|(1 \pm o(1)), & \mbox{if $(i,j)$ is an intra-class edge}, \\
    -\|\bmu\| \pm O(\sigma\sqrt{\log n}) = -\|\bmu\|(1 \pm o(1)),& \mbox{if $(i,j)$ is an inter-class edge}.
    \end{array}
    \right.
\end{equation}
Therefore, we probability at least $1-o(1)$, we have that
\begin{align*}
\sign(\Psi(\bX_i, \bX_j)) = \left\{\begin{array}{ll} \hspace{3mm}1, & \mbox{if $(i,j)$ is an intra-class edge}, \\ -1, & \mbox{if $(i,j)$ is an inter-class edge}, \end{array}\right.
\end{align*}
which means perfect separability of edges, and the proof is complete.

\subsection{Proof of Corollary~\ref{cor:gamma_easy}}\label{subsec:cor4proof}
We restate Corollary~\ref{cor:gamma_easy} for convenience.
\begin{corollary*}
Suppose that $\|\bmu\|=\omega(\sigma \sqrt{\log n})$ and that Assumption~\ref{ass:p_q} holds. Then with probability at least $1-o(1)$ over the data $(\bX,\bA) \sim \CSBM(n,p,q,\bmu,\sigma^2)$, the attention architecture $\Psi'$ yields attention coefficients $\gamma_{ij}$ such that
\begin{enumerate}
\item If $p \ge q$, then $\gamma_{ij}=\frac{2}{np}(1\pm o(1))$ if $(i,j)$ is an intra-class edge and $\gamma_{ij}=o(\frac{1}{n(p+q)})$ otherwise;
\item If $p < q$, then $\gamma_{ij}=\frac{2}{nq}(1\pm o(1))$ if $(i,j)$ is an inter-class edge and $\gamma_{ij}=o(\frac{1}{n(p+q)})$ otherwise.
\end{enumerate}
\end{corollary*}

The proof is straightforward by considering the cases $p \ge q$ and $p < q$ separately. When $p \ge q$, we have $\Psi' = \Psi$. Using the specification of $\Psi$ in \eqref{eq:psi_ansatz} and \eqref{eq:psi_ansatz_parameters}, the definition of attention coefficients in \eqref{def:attention_coeff}, the high probability event in Lemma~\ref{lem:calE*}, the expression of $\Psi(\bX_i,\bX_j)$ in \eqref{eq:psi_error_bound}, and picking $R$ such that $1/R = \omega(\sigma \sqrt{\log n})$ and $1/R = o(\|\bmu\|)$,  we obtain the claimed results. The result when $p < q$ is obtained in the same way.

\subsection{Proof of Corollary~\ref{cor:node_separation_easy}}\label{subsec:cor5proof}
We restate Corollary~\ref{cor:node_separation_easy} for convenience.
\begin{corollary*}
Suppose that $\|\bmu\|=\omega(\sigma \sqrt{\log n})$ and that Assumption~\ref{ass:p_q} holds. Then with probability at least $1-o(1)$ over the data $(\bX,\bA) \sim \CSBM(n,p,q,\bmu,\sigma^2)$, using the attention architecture $\Psi'$ with the graph attention convolution given in \eqref{eq:gat_output}, where $f$ is set to be the identify function, the model separates the nodes.
\end{corollary*}

We prove the case $p \ge q$ and the case $p < q$ follows analogously. Consider the attention architecture $\Psi' = (\mathbf{1}_{p \ge q} - \mathbf{1}_{p < q}) \cdot \Psi = \Psi$, where $\Psi$ is given in \eqref{eq:psi_ansatz} and \eqref{eq:psi_ansatz_parameters}. Pick $R$ such that $1/R = \omega(\sigma \sqrt{\log n})$ and $1/R = o(\|\bmu\|)$). Assume that $i\in C_1$, and denote the graph attention convolution output as
\[
	h_i' \eqdef \sum_{j \in N_i} \gamma_{ij} \tilde\bw^T\bX_j .
\]
We will condition on the event $\calbE^*$, which holds with probability at least $1-o(1)$. By using Corollary~\ref{cor:gamma_easy} we have
\begin{align*}
\sum_{j \in N_i} \gamma_{ij} \tilde{\bw}^T\bX_j
&= \sum_{j \in C_0 \cap N_i} \gamma_{ij} \tilde{\bw}^T\bX_j + \sum_{j \in C_1 \cap N_i} \gamma_{ij} \tilde{\bw}^T\bX_j \\
&\le |C_1 \cap N_i| \left(\frac{2}{np}(1\pm o(1)) \left(\|\bmu\| + 10\sigma\sqrt{ \log n}\right)\right) \\
&\qquad + |C_0 \cap N_i| \left(o\left(\frac{1}{n(p+q)}\right) \left(-\|\bmu\| + 10\sigma\sqrt{ \log n}\right)\right) \\
&= (1\pm o(1)) \cdot \left(\|\bmu\| + 10 \sigma\sqrt{ \log n} \right) -  \frac{nq(1\pm o(1))}{\omega(n(p+q))} \cdot \left(\|\bmu\| - 10\sigma\sqrt{ \log n}\right)\\
&= \|\bmu\| (1\pm o(1)).
\end{align*}
Similarly, we have that
\begin{align*}
\sum_{j \in N_i} \gamma_{ij} \tilde{\bw}^T\bX_j 
&\ge  (1\pm o(1)) \cdot \left(\|\bmu\| - 10 \sigma\sqrt{ \log n}\right) -  \frac{nq(1\pm o(1))}{\omega(n(p+q))} \cdot \left(\|\bmu\| + 10\sigma\sqrt{ \log n}\right) \\
&=  \|\bmu\| (1\pm o(1)).
\end{align*}
This means that $h_i' = \|\bmu\| (1\pm o(1))$ for $i \in C_1$. Applying the same reasoning we get that $h_i' = -\|\bmu\|(1\pm o(1))$ for $i\in C_0$. Therefore, with probability at least $1-o(1)$, the graph attention convolution separates the nodes.

\subsection{Proof of Lemma~\ref{lem:bayes-pairs}}
We restate Lemma~\ref{lem:bayes-pairs} for convenience.
\begin{lemma*}
Let $(\bX, \bA) \sim \CSBM(n, p, q, \bmu, \sigma^2)$ and let $\bX'_{ij}$ be defined as in \eqref{eq:edge_feats}. The Bayes optimal classifier for $\bX'_{ij}$ is realized by the following function,
\[
h^*(\bx)= \left\{
\begin{array}{ll}
0, & \text{if} \ p\cosh\left({\frac{\bx^T\bmu'}{\sigma^2}}\right) \le q\cosh\left({\frac{\bx^T\bnu'}{\sigma^2}}\right), \\
1, & \text{otherwise},
\end{array}
\right.
\]
where $\bmu' \eqdef \begin{pmatrix} \bmu \\ \bmu \end{pmatrix}$ and $\bnu' \eqdef \begin{pmatrix} \bmu \\ -\bmu \end{pmatrix}$.\end{lemma*}

\begin{proof}
Note that $\bX'_{ij}$ is a mixture of $2d$-dimensional Gaussian distributions,
\[
\bX'_{ij}\sim \begin{cases}
N(-\bmu', \sigma^2 \bI) & i\in C_0, j\in C_0\\
N(\bmu', \sigma^2 \bI) & i\in C_1, j\in C_1\\
N(-\bnu', \sigma^2 \bI) & i\in C_0, j\in C_1\\
N(\bnu', \sigma^2 \bI) & i\in C_1, j\in C_0
\end{cases}.
\]
The optimal classifier is then given by
\[
	h^*(\bx)=\argmax_{c\in\zo}\Prx[y=c \mid \bx].
\]
Note that $\Prx[y=0] = \frac{q}{p+q}$ and $\Pr[y=1] = \frac{p}{p+q}$. Thus, by Bayes rule, we obtain that
\begin{align*}
    \Prx[y=c\mid \bx]
    &=\frac{\Prx[y=c]\cdot f_{\bx|y}(\bx\mid y=c)}{\Prx[y=0]f_{\bx|y=0}(\bx\mid y=0)+\Prx[y=1]f_{\bx|y=1}(\bx\mid y=1)}\\
    &=\frac{1}{1+\frac{\Prx[y=1-c]\cdot f_{\bx|y}(\bx\mid y=1-c)}{\Prx[y=c]\cdot f_{\bx|y}(\bx\mid y=c)}}.
\end{align*}
Suppose that $\bx = \bX'_{ij}$ such that $i\nsim j$. Then $h^*(\bx)=0$ if and only if $\Prx[y=0\mid \bx]\ge \frac12$. Hence, for $c=0$ we require that
\[
	\frac{\Prx[y=1-c]\cdot f_{\bx|y}(\bx\mid y=1-c)}{\Prx[y=c]\cdot f_{\bx|y}(\bx\mid y=c)} = \frac{p}{q}\frac{f_{\bx|y}(\bx\mid y=1)}{f_{\bx|y}(\bx\mid y=0)} = \frac{p}{q}\frac{\cosh\left({\frac1{\sigma^2}\bx^T\bmu'}\right)}{\cosh\left({\frac1{\sigma^2}\bx^T\bnu'}\right)}\le 1,
\] 
Similarly, we obtain the reverse condition for $h^*(\bx)=1$.
\end{proof}

\subsection{Proof of Theorem~\ref{thm:edge_separation_hard}}\label{subsec:hard_pairs_proof}
We restate Theorem~\ref{thm:edge_separation_hard} for convenience.
\begin{theorem*}
Suppose $\|\bmu\|= \kappa\sigma$ for some $\kappa>0$ and let $\Psi$ be any attention mechanism. Then, 
\begin{enumerate}
    \item With probability at least $1-o(1)$, $\Psi$ fails to correctly classify at least $2\cdot\Phi_{\mathrm{c}}(\kappa)^2$ fraction of inter-class edges;
    \item For any $K>1$ if $q>\frac{K\log^2n}{n\Phi_{\rm c}(\kappa)^2}$, then with probability at least $1-O(n^{-\frac{K}{4} \Phi_{\rm c}(\kappa)^2\log n})$,  $\Psi$ misclassify at least one inter-class edge.
\end{enumerate}\end{theorem*}

We will write $i\sim j$ if node $i$ and node $j$ are in the same class and $i\nsim j$ otherwise. From Lemma~\ref{lem:bayes-pairs}, we observe that for successful classification by the optimal classifier, we need
\begin{align*}
    p\cosh\left({\tfrac{\bx^T\bmu'}{\sigma^2}}\right) \le q\cosh\left({\tfrac{\bx^T\bnu'}{\sigma^2}}\right) &\quad \text{for}\; i\nsim j,\\
    p\cosh\left({\tfrac{\bx^T\bmu'}{\sigma^2}}\right) > q\cosh\left({\tfrac{\bx^T\bnu'}{\sigma^2}}\right) &\quad \text{for}\; i\sim j.
\end{align*}
We will split the analysis into two cases. First, note that when $p\ge q$ we have for $i\nsim j$ that
\begin{align*}
    p\cosh\left({\tfrac{\bx^T\bmu'}{\sigma^2}}\right) \le q\cosh\left({\tfrac{\bx^T\bnu'}{\sigma^2}}\right)
    \implies \cosh\left({\tfrac{\bx^T\bmu'}{\sigma^2}}\right) \le \cosh\left({\tfrac{\bx^T\bnu'}{\sigma^2}}\right)
    \implies |\bx^T\bmu'| \le |\bx^T\bnu'|.
\end{align*}
In the first implication, we used that $p\ge q$, while the second implication follows from the fact that $\cosh(a)\le \cosh(b)\implies |a|\le |b|$ for all $a,b\in \R$. Similarly, for $p<q$ we have for $i\sim j$ that
\begin{align*}
    p\cosh\left({\tfrac{\bx^T\bmu'}{\sigma^2}}\right) > q\cosh\left({\tfrac{\bx^T\bnu'}{\sigma^2}}\right)
    \implies \cosh\left({\tfrac{\bx^T\bmu'}{\sigma^2}}\right) > \cosh\left({\tfrac{\bx^T\bnu'}{\sigma^2}}\right)
    \implies |\bx^T\bmu'| > |\bx^T\bnu'|.
\end{align*}
Therefore, for each of the above cases, we can upper bound the probability for either $i\sim j$ or $i\nsim j$ that $\bX'_{ij}$ is correctly classified, by the probability of the event $|\bX^{'T}_{ij}\bmu'| \le |\bX^{'T}_{ij}\bnu'|$ or equivalently $|\bX^{'T}_{ij}\bmu'| > |\bX^{'T}_{ij}\bnu'|$. We focus on the former as the latter is equivalent and symmetric. Writing $\bX_i = \bmu + \sigma\bg_i$ and $\bX_j = -\bmu + \sigma\bg_j$, we have that for $i\in C_1$ and $j\in C_0$,
\begin{align*}
    \Prx[h^*(\bX'_{ij}) = 0] &\le \Prx\left[|\bX^{'T}_{ij}\bmu'| \le |\bX^{'T}_{ij}\bnu'|\right]\\
    &= \Prx\left[|\bX_i^T\bmu + \bX_j^T\bmu| \le |\bX_i^T\bmu - \bX_j^T\bmu|\right]\\
    &= \Prx\left[\sigma|\bg_i^T\bmu + \bg_j^T\bmu| \le |\pm2\|\bmu\|^2 + \sigma\bg_i^T\bmu - \sigma\bg_j^T\bmu|\right]\\
    &\le \Prx\left[|\bg_i^T\hat{\bmu} + \bg_j^T\hat{\bmu}| - |\bg_i^T\hat{\bmu} - \bg_j^T\hat{\bmu}|\le \frac{2\|\bmu\|}{\sigma}\right]\\
    &= \Prx\left[|\bg_i^T\hat{\bmu} + \bg_j^T\hat{\bmu}| - |\bg_i^T\hat{\bmu} - \bg_j^T\hat{\bmu}|\le 2\kappa\right],
\end{align*}
where we denote $\hat{\bmu} = \bmu/\|\bmu\|$. In the second to last step above, we used triangle inequality to pull $2\|\bmu\|^2$ outside the absolute value, while in the last equation we use $\|\bmu\|=\kappa\sigma$.

We now denote $z_i = \bg_i^T\hat{\bmu}$ for all $i\in[n]$. Then the above probability is $\Prx[|z_i + z_j| - |z_i - z_j| \le 2\kappa]$, where $z_i, z_j\sim N(0, 1)$ are independent random variables. Note that we have
\begin{align}
\Prx[h^*(\bX'_{ij}) = 0]
&\le \Prx[|z_i + z_j| - |z_i - z_j| \le 2\kappa] \nonumber\\
&= \Prx[|z_i + z_j| - |z_i - z_j| \le 2\kappa, |z_i|\le \kappa] \nonumber\\ 
& \qquad + \Prx[|z_i + z_j| - |z_i - z_j| \le 2\kappa, |z_i|>\kappa] \nonumber\\
&= \Prx[|z_i|\le \kappa] + \Phi(\kappa)\Prx[|z_i| > \kappa].\label{eq:prob-distinguish-pair}
\end{align}
To see how we obtain the last equation, observe that if $|z_i|\le \kappa$ then we have
\begin{align*}
    |z_i + z_j| - |z_i - z_j| &= |z_i + z_j| - |z_j - z_i|\\
    &\le |z_i| + |z_j| - |z_j - z_i|& \text{by triangle inequality}\\
    &\le |z_i| + |z_j| - \big| |z_j| - |z_i| \big| & \text{by reverse triangle inequality}\\
    &\le |z_i| + |z_j| - (|z_j| - |z_i|) = 2|z_i|\\
    &\le 2\kappa,
\end{align*}
hence,  $\Prx[|z_i + z_j| - |z_i - z_j| \le 2\kappa, |z_i|\le \kappa] = \Prx[|z_i|\le \kappa]$. On the other hand, for $|z_i|>\kappa$, we look at each case, conditioned on the events $z_i>\kappa$ and $z_i<-\kappa$ for each of the four cases based on the signs of $z_i + z_j$ and $z_i - z_j$. We denote by $E$ the event that $|z_i + z_j| - |z_i - z_j| \le 2\kappa$, and analyze the cases in detail. First consider the case $z_i < -\kappa$:
\begin{align*}
\Prx[E, z_i + z_j\ge 0, z_i - z_j\ge 0\mid z_i<-\kappa] &= \Prx[z_j \le z_i, z_j\ge -z_i\mid z_i< - \kappa] = 0,\\
\Prx[E, z_i + z_j\ge 0, z_i - z_j< 0\mid z_i<-\kappa] &= \Prx[z_j > |z_i|, z_i\le \kappa \mid z_i< - \kappa] = \Phi(z_i),\\
\Prx[E, z_i + z_j< 0, z_i - z_j \ge 0\mid z_i<-\kappa] &= \Prx[z_j < -|z_i|, z_i\ge -\kappa \mid z_i< - \kappa] = 0,\\
\Prx[E, z_i + z_j< 0, z_i - z_j< 0\mid z_i<-\kappa] &= \Prx[z_i < z_j < -z_i, z_j > -\kappa \mid z_i< - \kappa] \\
&= \Phi(\kappa) - \Phi(z_i).
\end{align*}
The sum of the four probabilities in the above is $\Prx[E\mid z_i < -\kappa] = \Phi(\kappa)$. Similarly, we analyze the other case, $z_i > \kappa$:
\begin{align*}
\Prx[E, z_i + z_j\ge 0, z_i - z_j\ge 0\mid z_i >\kappa] &= \Prx[-z_i\le z_j \le z_i, z_j\le \kappa \mid z_i > \kappa] \\
&= \Phi(\kappa) - \Phi_{\rm c}(z_i),\\
\Prx[E, z_i + z_j\ge 0, z_i - z_j< 0\mid z_i >\kappa] &= \Prx[z_j > |z_i|, z_i\le \kappa \mid z_i > \kappa] = 0,\\
\Prx[E, z_i + z_j< 0, z_i - z_j \ge 0\mid z_i>\kappa] &= \Prx[z_j < -|z_i|, z_i\ge -\kappa \mid z_i > \kappa] = \Phi_{\rm c}(z_i),\\
\Prx[E, z_i + z_j< 0, z_i - z_j< 0\mid z_i>\kappa] &= \Prx[z_j < -z_i, z_j>z_i \mid z_i > \kappa] = 0.
\end{align*}
The sum of the four probabilities above is $\Prx[E\mid z_i > \kappa] = \Phi(\kappa)$. Therefore, we obtain that
\[\Prx[|z_i + z_j| - |z_i - z_j| \le 2\kappa \mid |z_i|>\kappa] = \Phi(\kappa),\]
which justifies \eqref{eq:prob-distinguish-pair}.

Next, note that $\Prx[|z_i|\le \kappa] = \Phi(\kappa) - \Phi_{\rm c}(\kappa)$ and $\Prx[|z_i| > \kappa] = 2\Phi_{\rm c}(\kappa)$, so we have from \eqref{eq:prob-distinguish-pair} that
\begin{align*}
\Prx[h^*(\bX'_{ij}) = 0] 
&\le \Phi(\kappa) - \Phi_{\rm c}(\kappa) + 2\Phi_{\rm c}(\kappa)\Phi(\kappa)\\
&= 1 - 2\Phi_{\rm c}(\kappa) + 2\Phi_{\rm c}(\kappa)\Phi(\kappa)
= 1 - 2\Phi_{\rm c}(\kappa)^2.
\end{align*}
Thus, $\bX'_{ij}$ is misclassified with probability at least $2\Phi_{\rm c}(\kappa)^2$.

We will now construct sets of pairs with mutually independent elements, such that the union of those sets covers all inter-class edges. This will enable us to use a concentration argument that computes the fraction of the inter-class edges that are misclassified. Since the graph operations are permutation invariant, let us assume for simplicity that $C_0 = \{1,\ldots,\frac{n}{2}\}$ and $C_1=\{\frac{n}{2}+1,\ldots,n\}$ for an even number of nodes $n$. Also, define the function
\[m(i,l) = \begin{cases}
i + l, & i+l\le \frac{n}{2},\\
i + l - \frac{n}{2}, & i+l > \frac{n}{2}.
\end{cases}.\]
We now construct the following sequence of sets for all $l\in \{0,\ldots,\frac{n}{2}-1\}$:
\[
	S_l = \{(X_{m(i,l)}, X_{i+\frac{n}{2}})\; \text{for all } i\in C_0 \text{ such that } (m(i,l),i+n/2)\in E\}.
\]
Fix $l\in \{0,\ldots,\frac{n}{2}-1\}$ and observe that the pairs in the set $S_l$ are mutually independent. Define a Bernoulli random variable, $\beta_i$, to be the indicator that $(X_{m(i,l)}, X_{i+\frac{n}{2}})$ is misclassified. We have that $\Ex[\beta_i] \ge 2\Phi_{\rm c}(\kappa)^2$. Note that the fraction of pairs in the set $S_l$ that are misclassified is $\frac{1}{|S_l|}\sum_{i:(X_{m(i,l)},X_{i+n/2})\in S_l}\beta_{i}$, which is a sum of independent Bernoulli random variables. Hence, by the additive Chernoff bound, we obtain
\[
\Prx\left[\sum_{i\in C_0\cap N_{m(i,l)}}\beta_{i} \ge 2|S_l|\Phi_{\rm c}(\kappa)^2 - |S_l|t\right] \ge 1 - \exp(-2|S_l|t^2).
\]
Since $p,q = \Omega(\frac{\log^2 n}{n})$, we have by the Chernoff bound and a union bound that with probability at least $1-1/{\rm poly}(n)$, $|S_l| = nq(1 \pm o(1))$ for all $l$. We now choose $t=\sqrt{\frac{C\log n}{|S_l|}} = o(1)$ to obtain that on the event where $|S_l| = nq(1 \pm o(1))$, we have the following for any large $C>1$:
\[
\Prx\left[\frac{1}{|S_l|}\sum_{i\in C_0\cap N_{m(i,l)}}\beta_{i} \ge 2\Phi_{\rm c}(\kappa)^2 - o(1) \right] \ge 1 - n^{-C}.
\]
Following a union bound over all $l\in \{0,\ldots,\frac{n}{2}-1\}$, we conclude that for any $c>0$,
\[
\Prx\left[\frac{1}{|S_l|}\sum_{i\in C_0\cap N_{m(i,l)}}\beta_{i} \ge 2\Phi_{\rm c}(\kappa)^2 - o(1),\;\; \forall l\in \left\{0,\ldots,\frac{n}{2}-1\right\} \right] \ge 1 - O(n^{-c}).
\]
Thus, out of all the pairs $\bX'_{ij}$ with $j\nsim i$, with probability at least $1 - o(1)$, we have that at least a fraction $2\Phi_{\rm c}(\kappa)^2$ of the pairs are misclassified by the attention mechanism. This concludes part 1 of the theorem.

For part 2, note that by the additive Chernoff bound we have for any $t\in (0, 1)$,
\[
\Prx\left[\sum_{i\in C_0\cap N_{m(i,l)}}\beta_{i} \ge 2|S_l|\Phi_{\rm c}(\kappa)^2 - |S_l|t \right] \ge 1 - \exp(-2|S_l|t^2).
\]
Since $|S_l| = \frac{nq}{2}(1\pm o(1))$ with probability at least $1/{\rm poly}(n)$, we choose $t=2\sqrt{\frac{K\Phi_{\rm c}(\kappa)^2\log^2 n}{nq}}$ to obtain
\[\Prx\left[\sum_{i\in C_0\cap N_{m(i,l)}}\beta_{i} \ge nq\Phi_{\rm c}(\kappa)^2(1\pm o(1)) - \sqrt{Knq\Phi_{\rm c}(\kappa)^2\log^2 n} \right] \ge 1 - O(n^{-8K \Phi_{\rm c}(\kappa)^2\log n}).\]
Now note that if $q>\frac{K \log^2 n}{n\Phi_{\rm c}(\kappa)^2}$ then we have $nq\Phi_{\rm c}(\kappa)^2 > K\log^2 n$, which implies that
\[nq\Phi_{\rm c}(\kappa)^2 - \sqrt{Knq\Phi_{\rm c}(\kappa)^2\log^2 n} > 0.\]
Hence, in this regime of $q$,
\[\Prx\left[\sum_{i\in C_0\cap N_{m(i,l)}}\beta_{i} > 0 \right] \ge 1 - O(n^{-8K \Phi_{\rm c}(\kappa)^2\log n}),\] and the proof is complete.

\subsection{Proof of Theorem~\ref{thm:gamma_hard}}\label{subsec:constgamma_proof}
We restate Theorem~\ref{thm:gamma_hard} for convenience
\begin{theorem*}
Assume that $\|\bmu\| \le K\sigma$ and $\sigma \le K'$ for some absolute constants $K$ and $K'$. Moreover, assume that the parameters $(\bw, \ba, b) \in \R^d \times \R^2 \times \R$ are bounded. Then, with probability at least $1-o(1)$ over the data $(\bX,\bA) \sim \CSBM(n,p,q,\bmu,\sigma^2)$, there exists a subset $\calA \subseteq [n]$ with cardinality at least $n(1-o(1))$ such that for all $i \in \calA$ the following hold:
\begin{enumerate}
\item There is a subset $J_{i,0} \subseteq N_i \cap C_0$ with cardinality at least $\frac{9}{10}|N_i \cap C_0|$, such that $\gamma_{ij} = \Theta(1/|N_i|)$ for all $j \in J_{i,0}$.
\item There is a subset $J_{i,1} \subseteq N_i \cap C_1$ with cardinality at least $\frac{9}{10}|N_i \cap C_1|$, such that $\gamma_{ij} = \Theta(1/|N_i|)$ for all $j \in J_{i,1}$.
\end{enumerate}
\end{theorem*}

For $i \in [n]$ let us write $\bX_i = (2\eps_i - 1)\bmu + \sigma \bg_i$ where $\bg_i \sim N(0, \bI)$, $\eps_i = 0$ if $i \in C_0$ and $\eps_i = 1$ if $i \in C_1$. Moreover, since the parameters $(\bw, \ba, b) \in \R^d \times \R^2 \times \R$ are bounded, we can write $\bw = R\hat\bw$ and $\ba = R'\hat\ba$ such that $\|\hat\bw\|=1$ and $\|\hat\ba\|=1$ and $R, R'$ are some constants. We define the following sets which will become useful later in our computation of $\gamma_{ij}$'s. Define
\[
\calA \eqdef \left\{i \in [n] ~\bigg|~
\begin{array}{l}
|\hat{\ba}_1\hat{\bw}^T\bg_i| \le 10\sqrt{\log (n(p+q))}, \ \mbox{and}\\ |\hat{\ba}_2\hat{\bw}^T\bg_j| \le 10\sqrt{\log (n(p+q))}, \ \forall j \in N_i
\end{array}
\right\}.
\]
For $i \in [n]$ define
\begin{align*}
	J_{i,0} &\eqdef \left\{j \in N_i \cap C_0 ~|~ |\hat{\ba}_2\hat{\bw}^T\bg_j| \le \sqrt{10} \right\},\\
	J_{i,1} &\eqdef \left\{j \in N_i \cap C_1 ~|~ |\hat{\ba}_2\hat{\bw}^T\bg_j| \le \sqrt{10} \right\},\\
	B_{i,0}^t &\eqdef \left\{j \in N_i \cap C_0 ~|~ 2^{t-1} \le \hat{\ba}_2\hat{\bw}^T\bg_j  \le 2^t\right\}, \ t = 1,2,\ldots,T, \\
	B_{i,1}^t &\eqdef \left\{j \in N_i \cap C_1 ~|~ 2^{t-1} \le \hat{\ba}_2\hat{\bw}^T\bg_j  \le 2^t\right\}, \ t = 1,2,\ldots,T,
\end{align*}
where $T \eqdef \left\lceil\log_2\left(10\sqrt{\log (n(p+q))}\right)\right\rceil$.
\medskip

\noindent We start with a few claims about the sizes of these sets.

\begin{claim}
\label{claim:setA}
With probability at least $1-o(1)$, we have that $|\calA| \ge n(1-o(1))$.
\end{claim}

\begin{proof}
Because $|\hat\ba_2| \le 1$ we know that $\calA$ is a superset of $\calA'$ where
\[
\calA' \eqdef \left\{i \in [n] ~\bigg|~ 
\begin{array}{l}
|\hat{\bw}^T\bg_i| \le 10\sqrt{\log (n(p+q))}, \ \mbox{and}\\ |\hat{\bw}^T\bg_j| \le 10\sqrt{\log (n(p+q))}, \ \forall j \in N_i
\end{array}
\right\}.
\]
We give a lower bound for $|\calA'|$ and hence prove the result. First of all, note that if $p+q \ge \Omega(1/\log^2 n)$, then $\log(n(p+q)) = \log n(1-o(1))$ and we easily get that with probability at least $1-o(1)$, $|\hat{\bw}^T\bg_i| \le 10\sqrt{\log (n(p+q))}$ for all $i \in [n]$, and thus $|\calA| = |\calA'| = n$. Therefore let us assume without loss of generality that $p+q \le O(1/\log^2 n)$. Consider the following sum of indicator random variables
\[
	S \eqdef  \sum_{i \in [n]} \ind_{\left\{|\hat\bw^T\bg_i| \ge 10 \sqrt{\log(n(p+q))}\right\}}.
\] 
By the multiplicative Chernoff bound, for any $\delta > 0$ we have
\[
	\Pr\left[S \ge nb(1+\delta)\right] \le \exp\left(-\frac{\delta^2}{2+\delta}nb\right)
\]
where $b \eqdef \Pr(|\hat\bw^T\bg_i| \ge 10 \sqrt{\log(n(p+q))})$. Moreover, by the standard upper bound on the Gaussian tail probability (Proposition 2.1.2, \cite{vershynin2018high}) we know that $b < e^{-50 \log(n(p+q))}$. Let us set
\[
	\delta \eqdef \frac{1}{bn(p+q)\log n}.
\]
Then by the upper bound on $b$ and the assumption that $p,q = \Omega(\log^2 n /n)$ we know that
\[
	\delta \ge \frac{(n(p+q))^{49}}{\log n} \ge \Omega(\log^{97} n) = \omega(1).
\]
It follows that
\[
	\frac{\delta^2}{2+\delta}nb \ge \Omega(\delta n b) = \Omega\left(\frac{1}{(p+q)\log n}\right) \ge \Omega(\log n).
\]
Therefore, with probability at least $1-o(1)$ we have that
\[
	S \le nb(1+\delta) \le \frac{n}{(n(p+q))^{50}}+ \frac{n}{n(p+q)\log n} = O\left(\frac{n}{n(p+q)\log n}\right).
\]
Apply the concentration result of node degrees, this means that with probability at least $1-o(1)$,
\begin{align*}
&\left|\left\{i \in [n] ~\big|~ |\hat\bw^T\bg_i| \ge 10 \sqrt{\log(n(p+q))} \ \mbox{or} \ \exists j \in N_i \ \mbox{such that} \ |\hat\bw^T\bg_j| \ge 10 \sqrt{\log(n(p+q))} \right\}\right| \\
&\le \ S \cdot \frac{n}{2}(p+q)(1\pm o(1)) = O\left(\frac{n}{n(p+q)\log n}\right) \cdot \frac{n}{2}(p+q)(1\pm o(1)) = O\left(\frac{n}{\log n}\right).
\end{align*}
Therefore we have
\[
	|\calA'| \ge n - O(n/\log n) = n(1-o(1)).
\]
\end{proof}

\begin{claim}
\label{claim:setJ}
With probability at least $1-o(1)$, we have that for all $i \in [n]$, 
\[
	|J_{i,0}| \ge \frac{9}{10}|N_i \cap C_0| ~~ \mbox{and} ~~ |J_{i,1}| \ge \frac{9}{10} |N_i \cap C_1|.
\]
\end{claim}

\begin{proof}
We prove the result for $J_{i,0}$, the result for $J_{i,1}$ follows analogously. First, fix $i \in [n]$. For each $j \in |N_i \cap C_0|$ we have that
\[
	\Pr[|\hat{\ba}_2\bw^T\bg_j| \ge \sqrt{10}] \le \Pr[|\bw^T\bg_j| \ge \sqrt{10}] \le e^{-50}.
\]
Denote $J_{i,0}^c \eqdef (N_i \cap C_0) \setminus J_{i,0}$. We have that
\[
	\Ex[ |J_{i,0}^c|] = \Ex \left[\sum_{j \in N_i \cap C_0} \ind_{\left\{|\hat{\ba}_2\bw^T\bg_j| \ge \sqrt{10} \right\}} \right] \le  e^{-50} |N_i \cap C_0|,
\]
Apply Chernoff's inequality (Theorem 2.3.4 in \cite{vershynin2018high}) we have
\begin{align*}
	\Pr\left[|J_{i,0}^c| \ge \frac{1}{10}|N_i \cap C_0| \right]
	&\le e^{-\Ex [|J_{i,0}^c|]} \left(\frac{e \Ex [|J_{i,0}^c|]}{ |N_i \cap C_0|/10}\right)^{|N_i \cap C_0|/10} \\
	&\le \left(\frac{e e^{-50}|N_i \cap C_0|}{|N_i \cap C_0|/10} \right)^{ |N_i \cap C_0| / 10}  \\ 
	&= \exp\left(-\left(\frac{1}{2} - \frac{\log 10}{10} - \frac{1}{10}\right)|N_i \cap C_0| \right) \\
	&\le \exp\left(-\frac{4}{25} |N_i \cap C_0| \right).
\end{align*}
Apply the union bound we get
\begin{align*}
	\Pr\left[|J_{i,0}| \ge \frac{9}{10} |C_0 \cap N_i|, \forall i \in [n] \right] 
	&\ge 1 - \sum_{i \in [n]} \exp\left(-\frac{4}{25} |N_i \cap C_0|\right) \\
	&\ge \Pr(\calbE_3) \cdot \left(1 - \sum_{i \in [n]} \exp\left(-\frac{4}{25} \frac{n \min(p,q)(1 - o(1))}{2} \right)\right) \\
	&= (1-o(1)) \cdot \left(1 - n \exp\left(-\frac{2n\min(p,q)(1- o(1))}{25}\right)\right) \\
	&= 1 - o(1).
\end{align*}
The second inequality follows because $|N_i \cap C_0| \ge \frac{n}{2}\min(p,q)(1 - o(1))$ under the event $\calbE_3$ (cf. Definition~\ref{def:high_prob_events}) for all $i \in [n]$. The last equality is due to our assumption that $p,q = \Omega(\frac{\log^2 n}{n})$. 
\end{proof}

\begin{claim}
\label{claim:setB}
With probability at least $1-o(1)$, we have that for all $i \in [n]$ and for all $t \in [T]$, 
\[
	|B_{i,0}^t| \le \Ex [|B_{i,0}^t|] + \sqrt{T}|N_i \cap C_0|^{\frac{4}{5}} ~~ \mbox{and} ~~|B_{i,1}^t| \le \Ex [|B_{i,1}^t|] + \sqrt{T}|N_i \cap C_1|^{\frac{4}{5}}.
\]
\end{claim}

\begin{proof}
We prove the result for $B_{i,0}^t$, and the result for $B_{i,1}^t$ follows analogously. First fix $i \in [n]$ and $t \in [T]$. By the additive Chernoff inequality, we have
\[
	\Pr\left(|B_{i,0}^t| \ge \Ex [|B_{i,0}^t|] +  |N_i \cap C_0| \cdot \sqrt{T}|N_i \cap C_0|^{-\frac{1}{5}}\right) \le  e^{-2T|N_i \cap C_0|^{3/5}}.
\]
Taking a union bound over all $i \in [n]$ and $t \in [T]$ we get
\begin{align*}
	&\Pr\left[\bigcup_{i\in [n]}\bigcup_{t\in[T]} \left\{|B_{i,0}^t| \ge \Ex [|B_{i,0}^t|] + \sqrt{T}|N_i \cap C_0|^{\frac{4}{5}}\right\}\right]\\
	\le~& nT\exp\left(-2T\left(\frac{n}{2}\min(p,q)(1 - o(1))\right)^{3/5}\right) + o(1) ~=~ o(1),
\end{align*}
where the last equality follows from Assumption~\ref{ass:p_q} that $p,q = \Omega(\frac{\log^2 n}{n})$, and hence 
\begin{align*}
    nT\exp\left(-2T\left(\frac{n}{2}\min(p,q)(1 - o(1))\right)^{3/5}\right) 
    &= nT \exp\left(-\omega\left(\sqrt{2}T\log n \right)\right) = O\left(n^{-c}\right)
\end{align*}
for some absolute constant $c > 0$. Moreover, we have used degree concentration, which introduced the additional additive $o(1)$ term in the probability upper bound. Therefore we have
\[
	\Pr\left[|B_{i,0}^t| \le \Ex [|B_{i,0}^t|] +  \sqrt{T}|N_i \cap C_0|^{\frac{4}{5}}, \forall i \in [n] ~\forall t\in[T]\right] \ge 1-o(1).
\]
\end{proof}

We start by defining an event $\calbE^{\#}$ which is the intersection of the following events over the randomness of $\bA$ and $\{\eps_i\}_i$ and $\bX_i = (2\eps_i -1)\bmu + \sigma \bg_i$,
\begin{itemize}
    \item $\calbE_1'$ is the event that for each $i \in [n]$, $|C_0\cap N_i|=\frac{n}{2}((1-\eps_i)p+\eps_i q)(1\pm o(1))$ and $|C_1\cap N_i|=\frac{n}{2}((1-\eps_i)q+\eps_i p)(1\pm o(1))$.
    \item $\calbE_2'$ is the event that $|\calA| \ge n - o(\sqrt{n})$.
    \item $\calbE_3'$ is the event that $|J_{i,0}| \ge \frac{9}{10}|N_i \cap C_0|$ and $|J_{i,1}| \ge \frac{9}{10} |N_i \cap C_1|$ for all $i \in [n]$.
    \item $\calbE_4'$ is the event that $|B_{i,0}^t| \le \Ex[ |B_{i,0}^t|] + \sqrt{T}|N_i \cap C_0|^{\frac{4}{5}}$ and $|B_{i,1}^t| \le \Ex [|B_{i,1}^t|] +  \sqrt{T}|N_i \cap C_1|^{\frac{4}{5}}$ for all $i \in [n]$ and for all $t \in [T]$.
\end{itemize}
By Claims~\ref{claim:setA},~\ref{claim:setJ},~\ref{claim:setB}, we get that with probability at least $1-o(1)$, the event $\calbE^{\#} \eqdef \bigcap_{i=1}^4 \calbE_i'$ holds. We will show that under event $\calbE^{\#}$, for all $i \in \calA$, for all $j \in J_{i,c}$ where $c \in \{0,1\}$, we have $\gamma_{ij} = \Theta(1/|N_i|)$. This will prove Theorem~\ref{thm:gamma_hard}.

Fix $i \in \calA$ and some $j \in J_{i,0}$. Let us consider
\begin{align*}
	\gamma_{ij} &= \frac{\exp\left(\LeakyRelu(\ba_1\bw^T\bX_i + \ba_2\bw^T\bX_j + b)\right)}{\sum_{k \in N_i} \exp\left(\LeakyRelu(\ba_1\bw^T\bX_i + \ba_2\bw^T\bX_k + b)\right)}\\
	& = \frac{\exp\left(\sigma RR' \ \LeakyRelu(\kappa_{ij} + \hat\ba_1\hat\bw^T\bg_i + \hat\ba_2\hat\bw^T\bg_j + b')\right)}{\sum_{k \in N_i} \exp\left(\sigma RR' \ \LeakyRelu(\kappa_{ik} + \hat\ba_1\hat\bw^T\bg_i + \hat\ba_2\hat\bw^T\bg_k + b')\right)}\\
	& = \frac{1}{\sum_{k \in N_i} \exp(\Delta_{ik} - \Delta_{ij})}
\end{align*}
where for $l \in N_i$, we denote
\begin{align*}
	\kappa_{il} &\eqdef(2\eps_i-1)\hat{\bw}^T\bmu/\sigma + (2\eps_l-1)\hat{\bw}^T\bmu/\sigma,\\
	\Delta_{il} &\eqdef \sigma R R' \ \LeakyRelu(\kappa_{il} +  \hat{\ba}_1\bw^T\bg_i + \hat{\ba}_2\bw^T\bg_l + b'),
\end{align*}
and $b = \sigma RR' b'$. We will show that 
\[
	\sum_{k \in N_i} \exp(\Delta_{ik} - \Delta_{ij}) = \Theta(|N_i|)
\]
and hence conclude that $\gamma_{ij} = \Theta(1/|N_i|)$. First of all, note that since $\|\bmu\| \le K \sigma$ for some absolute constant $K$, we know that 
\[
	|\kappa_{il}| \le \sqrt{2}K = O(1).
\]
Let us assume that $\hat\ba_1\hat\bw^T\bg_i \ge 0$ and consider the following two cases regarding the magnitude of $\hat\ba_1\hat\bw^T\bg_i$.

\underline{Case 1.} If $\kappa_{ij} + \hat\ba_1\hat\bw^T\bg_i + \hat\ba_2\hat\bw^T\bg_j + b' < 0$, then
\begin{align*}
	\Delta_{ik} - \Delta_{ij}
	&= \sigma RR' \Big(\LeakyRelu(\kappa_{ik} + \hat\ba_1\hat\bw^T\bg_i + \hat\ba_2\hat\bw^T\bg_k + b')\\
	&\qquad - \LeakyRelu(\kappa_{ij} + \hat\ba_1\hat\bw^T\bg_i + \hat\ba_2\hat\bw^T\bg_j + b')\Big) \\
	&= \sigma RR' \Big(\LeakyRelu(\hat\ba_1\hat\bw^T\bg_i + \hat\ba_2\hat\bw^T\bg_k \pm O(1)) \\
	&\qquad - \beta(\kappa_{ij} + \hat\ba_1\hat\bw^T\bg_i + \hat\ba_2\hat\bw^T\bg_j + b')\Big)\\
	&= \sigma RR' \left(\LeakyRelu(\hat\ba_2\hat\bw^T\bg_k \pm O(1)) \pm O(1) \right)\\
	&=\sigma RR' \left(\Theta( \hat\ba_2\hat\bw^T\bg_k) \pm O(1)\right),
\end{align*}
where $\beta$ is the slope of $\LeakyRelu(x)$ for $x < 0$. Here, the second equality follows from $|\kappa_{ik} + b'| \le \sqrt{2}K + |b'| =  O(1)$ and $\kappa_{ij} + \hat\ba_1\hat\bw^T\bg_i + \hat\ba_2\hat\bw^T\bg_j + b' < 0$. The third equality follows from
\begin{itemize}
	\item We have $j \in J_{i,0}$ and hence $|\hat\ba_2\hat\bw^T\bg_j| = O(1)$;
	\item We have $\kappa_{ij} + \hat\ba_1\hat\bw^T\bg_i + \hat\ba_2\hat\bw^T\bg_j + b' < 0$, so $\hat\ba_1\hat\bw^T\bg_i  < |\kappa_{ij}| + |\hat\ba_2\hat\bw^T\bg_j| + |b'| = O(1)$, moreover, because $\hat\ba_1\hat\bw^T\bg_i \ge 0$, we get that $|\hat\ba_1\hat\bw^T\bg_i | = O(1)$;
	\item We have $|\kappa_{ij} + \hat\ba_1\hat\bw^T\bg_i + \hat\ba_2\hat\bw^T\bg_j + b' | \le  | \hat\ba_1\hat\bw^T\bg_i| + |\hat\ba_2\hat\bw^T\bg_j| + |\kappa_{ij} + b'| = O(1) + O(1) + O(1) = O(1)$.
\end{itemize}

\underline{Case 2.} If $\kappa_{ij} + \hat\ba_1\hat\bw^T\bg_i + \hat\ba_2\hat\bw^T\bg_j + b' \ge 0$, then
\begin{align*}
	\Delta_{ik} - \Delta_{ij}
	&= \sigma RR' \Big(\LeakyRelu(\kappa_{ik} + \hat\ba_1\hat\bw^T\bg_i + \hat\ba_2\hat\bw^T\bg_k + b')\\
	&\qquad - \LeakyRelu(\kappa_{ij} + \hat\ba_1\hat\bw^T\bg_i + \hat\ba_2\hat\bw^T\bg_j + b')\Big) \\
	&= \sigma RR' \Big(\LeakyRelu(\kappa_{ik} + \hat\ba_1\hat\bw^T\bg_i + \hat\ba_2\hat\bw^T\bg_k + b')\\
	&\qquad - \kappa_{ij} - \hat\ba_1\hat\bw^T\bg_i -  \hat\ba_2\hat\bw^T\bg_j - b' \Big)\\
	&= \sigma RR' \left(\LeakyRelu(\kappa_{ik} + \hat\ba_1\hat\bw^T\bg_i + \hat\ba_2\hat\bw^T\bg_k + b') - \hat\ba_1\hat\bw^T\bg_i \pm O(1)\right)\\
	&\left\{ \begin{array}{ll}  = \sigma RR'\left(\Theta( \hat\ba_2\hat\bw^T\bg_k) \pm O(1)\right), & \mbox{if} \ k \in J_{i,0} \cup J_{i,1} \\ \le \sigma RR' \left(O( \hat\ba_2\hat\bw^T\bg_k) \pm O(1)\right), & \mbox{otherwise}. \end{array} \right.
\end{align*}

To see the last (in)equality in the above, consider the following cases:

\begin{enumerate}

\item If $k \in J_{i,0} \cup J_{i,1}$, then there are two cases depending on the sign of $\kappa_{ik} + \hat\ba_1\hat\bw^T\bg_i + \hat\ba_2\hat\bw^T\bg_k + b'$.

\begin{itemize}

\item If $\kappa_{ik} + \hat\ba_1\hat\bw^T\bg_i + \hat\ba_2\hat\bw^T\bg_k + b' \ge 0$, then we have that
\begin{align*}
	&\LeakyRelu(\kappa_{ik} + \hat\ba_1\hat\bw^T\bg_i + \hat\ba_2\hat\bw^T\bg_k + b') - \hat\ba_1\hat\bw^T\bg_i \pm O(1) \\
	=~& \kappa_{ik} + \hat\ba_1\hat\bw^T\bg_i + \hat\ba_2\hat\bw^T\bg_k + b' - \hat\ba_1\hat\bw^T\bg_i \pm O(1) \\
	=~& \hat\ba_2\hat\bw^T\bg_k +  \kappa_{ik} + b' \pm O(1)\\
	=~&  \hat\ba_2\hat\bw^T\bg_k \pm O(1).
\end{align*}

\item If $\kappa_{ik} + \hat\ba_1\hat\bw^T\bg_i + \hat\ba_2\hat\bw^T\bg_k + b' < 0$, then because $\hat\ba_1\hat\bw^T\bg_i \ge 0$ and $|\kappa_{ik} + \hat\ba_2\hat\bw^T\bg_k + b'| \le |\kappa_{ik}| +  |\hat\ba_2\hat\bw^T\bg_k| + |b'| = O(1)$, we know that $\hat\ba_1\hat\bw^T\bg_i < |\kappa_{ik}| +  |\hat\ba_2\hat\bw^T\bg_k| + |b'| = O(1)$ and $|\kappa_{ik} + \hat\ba_1\hat\bw^T\bg_i + \hat\ba_2\hat\bw^T\bg_k + b'| = O(1)$. Therefore it follows that
\begin{align*}
	&\LeakyRelu(\kappa_{ik} + \hat\ba_1\hat\bw^T\bg_i + \hat\ba_2\hat\bw^T\bg_k + b') - \hat\ba_1\hat\bw^T\bg_i \pm O(1) \\
	=~& \LeakyRelu(\pm O(1)) - O(1) \pm O(1) \\
	=~& \pm O(1) \\
	=~& \hat\ba_2\hat\bw^T\bg_k \pm O(1)
\end{align*}
where the last equality is due to the fact that $k \in J_{i,0} \cup J_{i,1}$ so $|\hat\ba_2\hat\bw^T\bg_k| = O(1)$.
\end{itemize}

\item If $k \not\in J_{i,0} \cup J_{i,1}$, then there are two cases depending on the sign of $\kappa_{ik} + \hat\ba_1\hat\bw^T\bg_i + \hat\ba_2\hat\bw^T\bg_k + b'$.

\begin{itemize}

\item If $\kappa_{ik} + \hat\ba_1\hat\bw^T\bg_i + \hat\ba_2\hat\bw^T\bg_k + b' \ge 0$, then we have that
\begin{align*}
	&\LeakyRelu(\kappa_{ik} + \hat\ba_1\hat\bw^T\bg_i + \hat\ba_2\hat\bw^T\bg_k + b') - \hat\ba_1\hat\bw^T\bg_i \pm O(1) \\
	=~& \kappa_{ik} + \hat\ba_1\hat\bw^T\bg_i + \hat\ba_2\hat\bw^T\bg_k + b' - \hat\ba_1\hat\bw^T\bg_i \pm O(1) \\
	=~& \hat\ba_2\hat\bw^T\bg_k +  \kappa_{ik} + b' \pm O(1)\\
	=~&  \hat\ba_2\hat\bw^T\bg_k \pm O(1).
\end{align*}

\item If $\kappa_{ik} + \hat\ba_1\hat\bw^T\bg_i + \hat\ba_2\hat\bw^T\bg_k + b' < 0$, then we have that,
\begin{align*}
	&\LeakyRelu(\kappa_{ik} + \hat\ba_1\hat\bw^T\bg_i + \hat\ba_2\hat\bw^T\bg_k + b') - \hat\ba_1\hat\bw^T\bg_i \pm O(1) \\
	=~& \beta\kappa_{ik} + \beta\hat\ba_1\hat\bw^T\bg_i + \beta\hat\ba_2\hat\bw^T\bg_k + \beta b' - \hat\ba_1\hat\bw^T\bg_i \pm O(1) \\
	=~&  \beta\hat\ba_2\hat\bw^T\bg_k - (1-\beta)\hat\ba_1\hat\bw^T\bg_i \pm O(1) \\
	\le~& \beta\hat\ba_2\hat\bw^T\bg_k \pm O(1),
\end{align*}
where $\beta$ is the slope of $\LeakyRelu(\cdot)$.
\end{itemize}
\end{enumerate}

Combining the two cases regarding the magnitude of $\hat\ba_1\hat\bw^T\bg_i$ and our assumption that $\sigma, R, R = O(1)$, so far we have showed that, for any $i$ such that $\hat\ba_1\hat\bw^T\bg_i \ge 0$, for all $j \in J_{i,0}$, we have
\begin{equation}
\label{eq:order_of_exponent}
	\Delta_{ik} - \Delta_{ij} =\left\{ \begin{array}{ll}   \Theta( \hat\ba_2\hat\bw^T\bg_k) \pm O(1), & \mbox{if} \ k \in J_{i,0} \cup J_{i,1} \\  O( \hat\ba_2\hat\bw^T\bg_k) \pm O(1), & \mbox{otherwise}. \end{array} \right.
\end{equation}
By following a similar argument, one can show that Equation~\ref{eq:order_of_exponent} holds for any $i$ such that $\hat\ba_1\hat\bw^T\bg_i < 0$.

Let us now compute 
\[
	\sum_{k \in N_i} \exp(\Delta_{ik} - \Delta_{ij}) = \sum_{k \in N_i \cap C_0}\exp(\Delta_{ik} - \Delta_{ij})  + \sum_{k \in N_i \cap C_1}\exp(\Delta_{ik} - \Delta_{ij})
\]
for some $j \in J_{i,0}$. Let us focus on $\sum_{k \in N_i \cap C_0}\exp(\Delta_{ik} - \Delta_{ij})$ first. We will show that $\Omega(|N_i \cap C_0|) \le \sum_{k \in N_i \cap C_0}\exp(\Delta_{ik} - \Delta_{ij}) \le O(|N_i|)$.

First of all, we have that
\begin{equation}
\label{eq:gamma_lb}
\begin{split}
\sum_{k \in N_i \cap C_0} \exp(\Delta_{ik} - \Delta_{ij})  
&\ge  \sum_{k \in J_{i,0}} \exp(\Delta_{ik} - \Delta_{ij}) = \sum_{k \in J_{i,0}} \exp\left(\Theta(\hat\ba_2\hat\bw^T\bg_k) \pm O(1) \right) \\
&\ge \sum_{k \in J_{i,0}} e^{c_1} = |J_{i,0}| e^{c_1} = \Omega(|N_i \cap C_0|),
\end{split}
\end{equation}
where $c_1$ is an absolute constant (possibly negative). On the other hand, consider the following partition of $N_i \cap C_0$:
\begin{align*}
	P_1 &\eqdef \{k \in N_i \cap C_0 ~|~ \hat{\ba}_2\hat{\bw}^T\bg_k \le 1\}, \\
	P_2 &\eqdef \{k \in N_i \cap C_0 ~|~ \hat{\ba}_2\hat{\bw}^T\bg_k \ge 1\}.
\end{align*}
It is easy to see that
\begin{equation}
\label{eq:gamma_ub1}
	\sum_{k \in P_1} \exp(\Delta_{ik} - \Delta_{ij}) 
	\le \sum_{k \in P_1} \exp\left(O(\hat\ba_2\hat\bw^T\bg_k) \pm O(1)\right) 
	\le \sum_{k \in P_1} e^{c_2} = |P_1|e^{c_2} = O(|N_i \cap C_0|),
\end{equation}
where $c_2$ is an absolute constant. Moreover, because $i \in \calA$ we have that $P_2 \subseteq \bigcup_{t \in [T]} B_{i,0}^t$. It follows that
\begin{equation}
\label{eq:gamma_ub21}
\begin{split}
	\sum_{k \in P_2} \exp(\Delta_{ik} - \Delta_{ij}) 
	&= \sum_{t\in[T]} \sum_{k \in B_{i,0}^t} \exp(\Delta_{ik} - \Delta_{ij}) \\
	&\le \sum_{t\in[T]} \sum_{k \in B_{i,0}^t} \exp \left(O(\hat\ba_2\hat\bw^T\bg_k) \pm O(1)\right)\\
	&\le \sum_{t\in[T]} |B_{i,0}^t| e^{c_3 2^t},
\end{split}
\end{equation}
where $c_3$ is an absolute constant. We can upper bound the above quantity as follows. Under the Event $\calbE^*$, we have that
\[
	|B_{i,0}^t| \le m_t +  \sqrt{T}|N_i \cap C_0|^{\frac{4}{5}}, \ \mbox{for all} \ t \in [T],
\]
where
\begin{align*}
	m_t\eqdef \Ex[ |B_{i,0}^t|] &= \sum_{k\in N_i \cap C_0} \Pr(2^{t-1} \le \hat{\ba}_2\hat{\bw}^T\bg_k \le 2^t) \le \sum_{k\in N_i \cap C_0} \Pr[\hat{\ba}_2\hat{\bw}^T\bg_k  \ge 2^{t-1}] \\ &\le  \sum_{k\in N_i \cap C_0}\Pr[\hat{\bw}^T\bg_k  \ge 2^{t-1}] \le  |N_i \cap C_0|e^{-2^{2t-3}}. 
\end{align*}
It follows that
\begin{equation}
\label{eq:gamma_ub22}
\begin{split}
	\sum_{t\in[T]} |B_{i,0}^t| e^{c_3 2^t} 
	&\le \sum_{t\in[T]}  \left( |N_i \cap C_0|e^{-2^{2t-3}} + \sqrt{T}|N_i \cap C_0|^{\frac{4}{5}} \right)e^{c_3 2^t} \\
	&\le |N_i \cap C_0| \sum_{t=1}^{\infty} e^{-2^{2t-3}}e^{c_3 2^t} + \sum_{t\in[T]}  \sqrt{T}|N_i \cap C_0|^{\frac{4}{5}} e^{c_3 2^T}\\
	&\le c_4 |N_i \cap C_0| + o(|N_i|) \\
	&\le O(|N_i|),
\end{split}
\end{equation}
where $c_4$ is an absolute constant. The third inequality in the above follows from
\begin{itemize}
\item The series $\sum_{t=1}^{\infty} e^{-2^{2t-3}}e^{c_3 2^t}$ converges absolutely for any constant $c_3$;
\item The sum $\sum_{t\in[T]}  \sqrt{T}|N_i \cap C_0|^{\frac{4}{5}} e^{c_3 2^T} = T^{\frac{3}{2}} |N_i \cap C_0|^{\frac{4}{5}} e^{c_3 2^T} = o(|N_i|)$ because
\begin{align*}
	\log\left(T^{\frac{3}{2}}e^{c_3 2^T}\right) 
	&= \frac{3}{2} \log \left\lceil\log_2\left(10\sqrt{\log (n(p+q))}\right)\right\rceil + c_3 2^{\left\lceil\log_2\left(10\sqrt{\log (n(p+q))}\right)\right\rceil}\\
	&\le  \frac{3}{2} \log \left\lceil\log_2\left(10\sqrt{\log (n(p+q))}\right)\right\rceil  + 20c_3\sqrt{\log (n(p+q))}\\
	&\le O\left(\frac{1}{c}\log(n(p+q))\right),
\end{align*}
for any $c > 0$. In particular, by picking $c > 5$ we see that $T^{\frac{3}{2}}e^{c_3 2^T} \le O((n(p+q))^{\frac{1}{c}}) \le o(|N_i|^{\frac{1}{5}})$, and hence we get $T^{\frac{3}{2}}e^{c_3 2^T} |N_i \cap C_0|^{\frac{4}{5}} \le |N_i|^{\frac{4}{5}} \cdot o(|N_i|^{\frac{1}{5}}) = o(|N_i|)$.
\end{itemize}
Combining Equations \ref{eq:gamma_ub21} and \ref{eq:gamma_ub22} we get
\begin{equation}
\label{eq:gamma_ub2}
	\sum_{k \in P_2} \exp(\Delta_{ik} - \Delta_{ij}) \le O(|N_i|),
\end{equation}
and combining Equations \ref{eq:gamma_ub1} and \ref{eq:gamma_ub2} we get
\begin{equation}
\label{eq:gamma_ub}
	\sum_{k \in N_i \cap C_0} \exp(\Delta_{ik} - \Delta_{ij}) =  \sum_{k \in P_1} \exp(\Delta_{ik} - \Delta_{ij}) + \sum_{k \in P_1} \exp(\Delta_{ik} - \Delta_{ij})\le O(|N_i|).
\end{equation}
Now, by Equations \ref{eq:gamma_lb} and \ref{eq:gamma_ub} we get
\begin{equation}
\label{eq:gamma_lbub_c0}
	\Omega(|N_i \cap C_0|) \le \sum_{k \in N_i \cap C_0} \exp(\Delta_{ik} - \Delta_{ij}) \le O(|N_i|).
\end{equation}
It turns out that repeating the same argument for $ \sum_{k \in N_i \cap C_1} \exp(\Delta_{ik} - \Delta_{ij})$ yields
\begin{equation}
\label{eq:gamma_lbub_c1}
	\Omega(|N_i \cap C_1|) \le \sum_{k \in N_i \cap C_1} \exp(\Delta_{ik} - \Delta_{ij}) \le O(|N_i|).
\end{equation}
Finally, Equations \ref{eq:gamma_lbub_c0} and \ref{eq:gamma_lbub_c1} give us
\[
	 \sum_{k \in N_i} \exp(\Delta_{ik} - \Delta_{ij}) = \Theta(|N_i|),
\]
which readily implies
\[
	\gamma_{ij} = \frac{1}{\sum_{k \in N_i} \exp(\Delta_{ik} - \Delta_{ij})} = \Theta(1/|N_i|)
\]
as required. We have showed that for all $i \in \calA$ and for all $j \in J_{i,0}$, $\gamma_{ij} = \Theta(1/|N_i|)$. Repeating the same argument we get that the same result holds for all $i \in \calA$ and for all $j \in J_{i,1}$, too. Hence, by Claims \ref{claim:setA} and \ref{claim:setJ} about the cardinalities of $\calA$, $J_{i,0}$ and $J_{i,1}$ we have thus proved Theorem~\ref{thm:gamma_hard}.

\subsection{Proof of Proposition~\ref{thm:good_psi_negative}}\label{subsec:good_psi_negative_proof}
We restate Proposition~\ref{thm:good_psi_negative} for convenience.
\begin{proposition}
Suppose that $p,q$ satisfy Assumption~\ref{ass:p_q} and that $p,q$ are bounded away from 1. For every $\epsilon > 0$, there are absolute constants $M, M' = O(\sqrt{\epsilon})$ such that, with probability at least $1-o(1)$ over the data $(\bX,\bA) \sim \CSBM(n,p,q,\bmu,\sigma^2)$, using the graph attention convolution in \eqref{eq:gat_output} and the attention architecture $\tilde{\Psi}$ in \eqref{eq:good_psi}, the model misclassifies at least $\Omega(n^{1-\epsilon})$ nodes for any $\bw$ such that $\|\bw\| = 1$, if
\begin{enumerate}
    \item $t = O(1)$ and $\|\bmu\| \le M \sigma \sqrt{\frac{\log n}{n(p+q)}(1-\max(p,q))}  \frac{p+q}{|p-q|}$;
    \item $t = \omega(1)$ and $\|\bmu\| \le M' \sigma \sqrt{\frac{\log n}{n(p+q)}(1-\max(p,q))}$.
\end{enumerate}
\end{proposition}

We start with part 1 of the proposition. Let us assume that $p \ge q$. The result when $p < q$ follows analogously. We will condition on the events $\calbE_1,\calbE_2,\calbE_3$ defined in Definition~\ref{def:high_prob_events}. These events are concerned with the concentration of class sizes $|C_0|$ and $|C_1|$ and the concentration of the number of intra-class and inter-class edges, i.e. $|C_0 \cap N_i|$ and $C_1\cap N_i|$ for all $i$. By Lemma~\ref{lem:calE*}, the probability that these events hold simultaneously is at least $1- o(1)$. Fix any $\bw \in \R^d$ such that $\|\bw\|=1$. Without loss of generality, assume that $\bw^T\bmu > 0$. Because $t=O(1)$, by the definition of $\tilde{\Psi}$ in \eqref{eq:good_psi} and the attention coefficients in \eqref{def:attention_coeff} we have that
\begin{equation}\label{eq:weak_psi_gammas}
	\gamma_{ij} = \left\{\begin{array}{ll} \frac{c_1}{n(p+q)}(1 \pm o(1)), & \mbox{if $(i,j)$ is an intra-class edge}, \\ \frac{c_2}{n(p+q)}(1 \pm o(1)), & \mbox{if $(i,j)$ is an inter-class edge}, \end{array}\right.
\end{equation}
for some positive constants $c_1 \ge 1$ and $c_2 \le 1$. Let us write $\bX_i = (2\eps_i - 1)\bmu + \sigma \bg_i$ where $\bg_i \sim N(0, \bI)$, $\eps_i = 0$ if $i \in C_0$ and $\eps_i = 1$ if $i \in C_1$. Let us consider the classification of nodes in $C_0$ based on the model output $\sum_{j \in N_i}\gamma_{ij}\bw^TX_j$ for node $i$. Note that, depending on whether $i \in C_0$ or $i \in C_1$, the expectation of the model output is symmetric around 0, and therefore we consider the decision boundary at 0. Let $0 < \epsilon < 1$ and fix any partition of $C_0$ into a set of disjoint subsets, 
\[
	C_0 = \bigcup_{h=1}^{\ell}C_0^{(h)}
\]
such that $\ell = |C_0|^{1-\epsilon}$ and $|C_0^{(h)}| = |C_0|^\epsilon$ for every $h = 1,2,\ldots, \ell$. In what follows we will consider the classification of nodes in each $C_0^{(h)}$ separately. We will show that, with probability at least $1-o(1)$, for each one of more than half of the subsets $C_0^{(h)}$ where $h = 1,2,\ldots,\ell$, the model is going to misclassify at least one node, and consequently, the model is going to misclassify at least $\ell/2 = |C_0|^{1-\epsilon}/2=\Omega(n^{1-\epsilon})$ nodes, giving the required result on misclassification rate.

Fix any $h' \in \{1,2,\ldots, \ell\}$. Using \eqref{eq:weak_psi_gammas} we get that, for large enough $n$, the event that the model correctly classifies all nodes in $C_0^{(h)}$ satisfies
\[
\begin{split}
	&\Bigg\{\max_{i \in C_0^{(h')}} \sum_{j \in N_i} \gamma_{ij}\bw^T\bX_j < 0\Bigg\} \\
	=~&\Bigg\{\max_{i \in C_0^{(h')}}\bigg(\sum_{j \in N_i \cap C_1}\gamma_{ij} - \sum_{j \in N_i \cap C_0}\gamma_{ij}\bigg)\bw^T\bmu + \sigma\sum_{j \in N_i} \gamma_{ij}\bw^T\bg_j < 0\Bigg\}\\
	\subseteq~& \Bigg\{c_3\left(\frac{q-p}{p+q}\right)\bw^T\bmu +\sigma \max_{i \in C_0^{(h')}} \sum_{j \in N_i} \gamma_{ij}\bw^T\bg_j <0\Bigg\}
\end{split}
\]
for some absolute constant $c_3 > 0$, and hence the probability that the model correctly classifies all nodes in $C_0^{(h')}$ satisfies, for large enough $n$,
\begin{align*}
	\Pr\Bigg(\max_{i \in C_0^{(h')}} \sum_{j \in N_i} \gamma_{ij}\bw^T\bX_j < 0\Bigg) 
	&\le \Pr\Bigg(\max_{i \in C_0^{(h')}} \sum_{j \in N_i} \bw^T\bg_j < c_3\left(\frac{p-q}{p+q}\right)\frac{|\bw^T\bmu|}{\sigma}\Bigg)\\
	&\le \Pr\Bigg(\max_{i \in C_0^{(h')}} \sum_{j \in N_i} \bw^T\bg_j < \tilde{M} \sqrt{\frac{\log n }{n(p+q)}(1-\max(p,q))}\Bigg)
\end{align*}
where the last inequality follows from our assumption on $\|\bmu\|$ and we denote $\tilde{M} \eqdef c_3 M> 0$. Now we will use Sudakov's minoration inequality~\cite{vershynin2018high} to obtain a lower bound on the expected maximum, and then apply Borell's inequality to upper bound the above probability. In order to apply Sudakov's result we will need to define a metric over the node index set $[n]$. Let $\bz_i \eqdef  \sum_{j \in N_i} \bw^T\bg_j$. For $i,j \in C_0$, $i\neq j$, consider the canonical metric $d_\circ(i,j)$ given by
\begin{align*}
	d_\circ(i,j)^2 
	&\eqdef \Ex[(\bz_i - \bz_j)^2] \\
	&= \sum_{k \in N_i} \gamma_{ik}^2 + \sum_{k \in N_j} \gamma_{jk}^2 - 2\sum_{k \in N_i \cap N_j} \gamma_{ik}\gamma_{jk} \\
	&\ge c_4\sum_{k \in J_{ij}} \frac{1}{n^2(p+q)^2}\\
	&= \frac{c_4|J_{ij}|}{n^2(p+q)^2},
\end{align*}
where $J_{ij} \eqdef (N_i \cup N_j) \backslash (N_i \cap N_j)$ is the symmetric difference of the neighbors of $i$ and $j$, $c_4 > 0$ is an absolute constant, and the inequality is due to \eqref{eq:weak_psi_gammas}. We lower bound $|J_{ij}|$ as follows. For $i, j\in C_0$, $i \neq j$, and a node $k\in [n]$, the probability that $k$ is a neighbor of exactly one of $i$ and $j$ is $2p(1-p)$ if $k\in C_0$ and $2q(1-q)$ if $k\in C_1$. Therefore we have $\Ex [|J_{ij}|] = n(p(1-p)+q(1-q))$. It follows from the multiplicative Chernoff bound that for any $0<\delta<1$,
 \[
 	\Pr[|J_{ij}| < \Ex[|J_{ij}|](1-\delta)]\le \exp(-\delta^2\Ex[|J_{ij}|]/3).
\]
Choose
\[
	\delta=3\sqrt{\frac{\log n}{\Ex[|J_{ij}|]}} = 3\sqrt{\frac{\log n}{n(p(1-p)+q(1-q))}} = o(1),
\]
where the last equality follows from $n(p(1-p)+q(1-q)) = \Omega(\log^2n)$, which is in turn due to the assumptions that $p,q =\Omega(\log^2 n /n)$ and $p,q$ are bounded away from 1. Apply a union bound over all $i,j \in C_0$, we get that with probability at least $1-o(1)$, the size of $J_{ij}$ satisfies 
\begin{equation}\label{eq:uncommon_neighbors_lb}
	|J_{ij}| \ge n(p(1-p)+q(1-q))(1-o(1)), \forall i,j\in C_0.
\end{equation}
Therefore it follows that, for large enough $n$,
\[
	d_\circ(i,j) \ge \sqrt{\frac{c_4|J_{ij}|}{n^2(p+q)^2}} = \sqrt{\frac{c_4n(p(1-p)+q(1-q))(1-o(1))}{n^2(p+q)^2}} = \Omega\left(\sqrt{\frac{1-\max(p,q)}{n(p+q)}}\right).
\]
We condition on the event that the inequality~\eqref{eq:uncommon_neighbors_lb} holds for all $i,j \in C_0$, which happens with probability at least $1-o(1)$. Apply Sudakov's minoration with metric $d_\circ(i,j)$, we get that for large enough $n$, for all $h = 1,2,\ldots,\ell$,
 \begin{equation}\label{eq:sudakov_max_lb}
 	\Ex\left[\max_{i\in C_0^{(h)}}\sum_{j\in N_i}\gamma_{ij} \bw^T\bg_j\right] \ge c_5 d_\circ(i,j)\sqrt{\log |C_0^{(h)}|} \ge c_6\sqrt{\epsilon}\sqrt{\frac{\log n}{n(p+q)}(1-\max(p,q))}
\end{equation}
for some absolute constants $c_5,c_6>0$. The last inequality in the above follows from $|C_0^{(h)}| = |C_0|^\epsilon = \Omega(n^\epsilon)$. In addition, note that since by assumption $\Psi$ is independent from the node features, using \eqref{eq:weak_psi_gammas} we have that $\sum_{j \in N_i} \gamma_{ij} \bw^T\bg_j$ is Gaussian with variance $O(\frac{1}{n(p+q)})$. We use Borell's inequality (\cite{adler2007random} Chapter 2) to get that for any $t>0$ and large enough $n$,
 \[
 	\Pr\left[\max_{i\in C_0^{(h')}}\sum_{j\in N_i}\gamma_{ij} \bw^T\bg_j<\Ex\left[\max_{i\in C_0^{(h')}}\sum_{j\in N_i}\gamma_{ij}\bw^T\bg_j\right]-t\right]\le \exp(-c_7t^2n(p+q)).
\] 
for some absolute constant $c_7 > 0$. By the lower bound of the expectation \eqref{eq:sudakov_max_lb}, we get that 
 \[
 	\Pr\left[\max_{i\in C_0^{(h')}}\sum_{j\in N_i}\gamma_{ij} \bw^T\bg_j<c_6\sqrt{\epsilon}\sqrt{\frac{\log n}{n(p+q)}(1-\max(p,q))}-t\right]\le \exp(-c_7t^2n(p+q)).
\]
Now, let $M>0$ be any constant that also satisfies $M < c_6\sqrt{\epsilon}/c_3$. Recall that we defined $\tilde{M} = c_3M$, and hence $\tilde{M} < c_6\sqrt{\epsilon}$. Set
\begin{equation}\label{eq:t_cond}
	t=(c_6\sqrt{\epsilon} - \tilde{M})\sqrt{\frac{\log n}{n(p+q)}(1-\max(p,q))}=\Omega\left(\sqrt{\frac{\log n}{n(p+q)}(1-\max(p,q))}\right),
\end{equation}
then combine with the events we have conditioned so far we get
\[
	\Pr\left[\max_{i\in C_0^{(h')}}\sum_{j\in N _i}\gamma_{ij} \bw^T\bg_j\le \tilde{M}\sqrt{\frac{\log n}{n(p+q)}(1-\max(p,q))}\right]=o(1).
\]
Recall that the above probability is the probability of correctly classifying all nodes in $C_0^{(h')}$. Since our choice of $h'$ was arbitrary, this applies to every $h \in \{1,2,\ldots,h\}$. Let $I_0^{(h)}$ denote the indicator variable of the event that there is at least one node in $C_0^{(h)}$ that is misclassified, then
\[
	\Ex\left[\sum_{h=1}^{\ell}I_0^{(h)}\right] = \ell \cdot \Ex[I_0^{(1)}] = \ell \cdot (1-o(1)) = \ell - o(\ell).
\]
Apply the reverse Markov inequality, we get
\[
	\Pr\left[\sum_{h=1}^{\ell}I_0^{(h)} \le \frac{1}{2}\ell\right] \le \frac{\ell - \Ex\left[\sum_{h=1}^{\ell}I_0^{(h)}\right]}{\ell - \frac{1}{2}\ell} = \frac{o(\ell)}{\frac{1}{2}\ell} = o(1).
\]
Therefore, with probability at least $1-o(1)$, we have $\sum_{h=1}^{\ell}I_0^{(h)} \ge \frac{1}{2}\ell = \Omega(n^{1-\epsilon})$. This implies the required result that the model misclassifies at least $\Omega(n^{1-\epsilon})$ nodes.

The proof of part 2 is similar to the proof of part 1. Let us assume that $p \ge q$ since the result when $p < q$ can be proved analogously. We condition on the events $\calbE_1,\calbE_2,\calbE_3$ defined in Definition~\ref{def:high_prob_events} which simultaneous hold with probability at least $1- o(1)$ by Lemma~\ref{lem:calE*}. Fix any $\bw \in \R^d$ such that $\|\bw\|=1$. Because $t = \omega(1)$, by the definition of $\tilde{\Psi}$ in \eqref{eq:good_psi} and the attention coefficients in \eqref{def:attention_coeff} we have that
\begin{equation}\label{eq:strong_psi_gammas}
	\gamma_{ij} = \left\{\begin{array}{ll} \frac{2}{np}(1 \pm o(1)), & \mbox{if $(i,j)$ is an intra-class edge}, \\ o(\frac{1}{n(p+q)}), & \mbox{if $(i,j)$ is an inter-class edge}. \end{array}\right.
\end{equation}
Write $\bX_i = (2\eps_i - 1)\bmu + \sigma \bg_i$ where $\bg_i \sim N(0, \bI)$, $\eps_i = 0$ if $i \in C_0$ and $\eps_i = 1$ if $i \in C_1$. We consider the classification of nodes in $C_0$ based on the model output $\sum_{j \in N_i}\gamma_{ij}\bw^TX_j$ for node $i$ and the decision boundary at 0. As before, let $0 < \epsilon < 1$ and fix any partition of $C_0$ into a set of disjoint subsets $C_0 = \bigcup_{h=1}^{\ell}C_0^{(h)}$ such that $\ell = |C_0|^{1-\epsilon}$ and $|C_0^{(h)}| = |C_0|^\epsilon$ for every $h = 1,2,\ldots, \ell$. We proceed to show that, with high probability, for each one of more than half of the subsets $C_0^{(h)}$ where $h = 1,2,\ldots,\ell$, the model is going to misclassify at least one node, and consequently, the model is going to misclassify at least $\ell/2 =\Omega(n^{1-\epsilon})$ nodes as required. Fix any  $h' \in \{1,2,\ldots, \ell\}$. Using \eqref{eq:strong_psi_gammas} we get that, for large enough $n$, the event that the model correctly classifies all nodes in $C_0^{(h')}$ satisfies
\[
\begin{split}
	\Bigg\{\max_{i \in C_0^{(h')}} \sum_{j \in N_i} \gamma_{ij}\bw^T\bX_j < 0\Bigg\} 
	\subseteq \Bigg\{c_1\bw^T\bmu + \sigma\max_{i \in C_0^{(h')}}\sum_{j \in N_i} \gamma_{ij}\bw^T\bg_j < 0\Bigg\}
\end{split}
\]
for some absolute constant $c_1 > 0$, and hence the probability that the model classifies all nodes in $C_0^{h'}$ correctly satisfies, for large enough $n$,
\begin{align*}
	\Pr\Bigg(\max_{i \in C_0^{(h')}} \sum_{j \in N_i} \gamma_{ij}\bw^T\bX_j < 0\Bigg) 
	&\le \Pr\Bigg(\max_{i \in C_0^{(h')}} \sum_{j \in N_i} \bw^T\bg_j < c_1\frac{|\bw^T\bmu|}{\sigma}\Bigg)\\
	&\le \Pr\Bigg(\max_{i \in C_0^{(h')}} \sum_{j \in N_i} \bw^T\bg_j < \tilde{M} \sqrt{\frac{\log n }{n(p+q)}(1-\max(p,q))}\Bigg)
\end{align*}
where the last inequality follows from our assumption on $\|\bmu\|$ and we denote $\tilde{M} \eqdef c_1M' > 0$. The rest of the proof of part 2 proceeds as the proof of part 1.

\subsection{Proof of Theorem~\ref{thm:gat_linear}}

We restate Theorem~\ref{thm:gat_linear} for convenience. Recall that we write $\|\bmu\| = \kappa\sigma$ for some $\kappa > 0$.
\begin{theorem*}
With probability at least $1-o(1)$ over the data $(\bX,\bA) \sim \CSBM(n,p,q,\bmu,\sigma^2)$, using the two-layer MLP attention architecture $\Psi$ given in \eqref{eq:psi_ansatz} and \eqref{eq:psi_ansatz_parameters} with $R = \Omega(n \log^2n/\sigma)$, the graph attention convolution output satisfies
\begin{align*}
&h_i' = \sum_{j\in N_i}\gamma_{ij}\tilde\bw^T\bX_j > 0 \; \mbox{if and only if} \; \tilde\bw^T\bX_i > 0, \; \forall i \in[n],\\
&h_i' = \sum_{j\in N_i}\gamma_{ij}\tilde\bw^T\bX_j < 0 \; \mbox{if and only if} \; \tilde\bw^T\bX_i < 0, \; \forall i \in[n].
\end{align*}
\end{theorem*}

We will condition on the following two events concerning the values of $\tilde\bw^T\bX_i$ for $i \in [n]$. Both events hold with probability at least $1-o(1)$. First, for $\epsilon>0$ consider the event
\[
    \left\{\tilde\bw^T \bX_i \not\in [-\epsilon\sigma\max\{1,\kappa\},\epsilon\sigma\max\{1,\kappa\}] \ \mbox{for all} \ i \in [n] \right\}.
\]
Since $\tilde\bw = \bmu/\|\bmu\|$ and $\|\bmu\| = \kappa\sigma$, each $\tilde\bw^T \bX_i$ follows a normal distribution with mean $\epsilon_i\kappa\sigma$ and variance $\sigma^2$ (recall that $\epsilon_i \in \{0,1\}$ generates the node memberships in CSBM), we have that
\begin{align*}
    \Pr\left(-\epsilon\sigma\max\{1,\kappa\}\le\tilde\bw^T \bX_i \le\epsilon\sigma\max\{1,\kappa\}\right)
    &=\int_{\kappa-\epsilon\max\{1,\kappa\}}^{\kappa+\epsilon\max\{1,\kappa\}}\frac{1}{\sqrt{2\pi}}e^{-x^2/2}dx \\
    &\le \left\{\begin{array}{ll}\sqrt{\frac{2}{\pi}}\epsilon, & \mbox{if} \ \kappa \le 1, \\ \sqrt{\frac{2}{\pi}}e^{-\kappa^2(1-\epsilon)^2/2}\epsilon\kappa, & \mbox{if} \ \kappa > 1. \end{array}\right.
\end{align*}
Let us pick $\epsilon$ such that
\begin{align*}
    \epsilon = \frac{1}{n \sqrt{\log n}} \;\ \mbox{if} \;\ \kappa \le 2\sqrt{\log n}, \quad \mbox{and} \quad
    \epsilon = \frac{1}{4} \;\ \mbox{if} \;\ \kappa > 2\sqrt{\log n}.
\end{align*}
This gives that
\[
    \Pr\left(\tilde\bw^T \bX_i \not\in [-\epsilon\sigma\max\{1,\kappa\},\epsilon\sigma\max\{1,\kappa\}] \ \mbox{for all} \ i \in [n]\right) \ge 1 - n\cdot o(1/n) = 1 - o(1).
\]
Second, consider the event that $|\tilde\bw^T\bX_i - \Ex[\tilde\bw^T\bX_i]| \le 10\sigma\sqrt{\log n}$. By Lemma~\ref{lem:calE*} we know that it holds with probability at least $1-o(1)$. Under these two events, we have that 
\begin{align*}
    \epsilon\sigma\max\{1,\kappa\} \le \tilde\bw^T\bX_i \le \kappa\sigma+10\sigma\sqrt{\log n}\quad&\mbox{if} \quad \tilde\bw^T\bX_i  > 0,\\
     -\kappa\sigma-10\sigma\sqrt{\log n} \le \tilde\bw^T\bX_i \le -\epsilon\sigma\max\{1,\kappa\}\quad &\mbox{if} \quad \tilde\bw^T\bX_i  < 0.
\end{align*}
Recall from \eqref{eq:psi_simplified} that
\[
    \Psi(\bX_i,\bX_j) =
    \left\{
    \begin{array}{ll}
    -2R(1-\beta)\tilde\bw^T\bX_i, & \mbox{if}~ \tilde\bw^T\bX_j \le -\left|\tilde\bw^T\bX_i\right|,\\
    \hspace{3mm}2R(1-\beta)\sign(\tilde\bw^T\bX_i)\tilde\bw^T\bX_j, & \mbox{if}~ -\left|\tilde\bw^T\bX_i\right| < \tilde\bw^T\bX_j < \left|\tilde\bw^T\bX_i\right|,\\
    \hspace{3mm}2R(1-\beta)\tilde\bw^T\bX_i, & \mbox{if}~ \tilde\bw^T\bX_j \ge \left|\tilde\bw^T\bX_i\right|.
    \end{array}
    \right.
\]
Consider an arbitrary node $i$ where $\tilde\bw^T\bX_{i} > 0$. We will show that $h_i = \sum_{j \in N_i}\gamma_{ij}\tilde\bw^T\bX_j > 0$. By the definition of attention coefficients in \eqref{def:attention_coeff}, we know that
\[
    \sum_{j \in N_i}\gamma_{ij}\tilde\bw^T\bX_j > 0 \iff \sum_{j \in N_i}\exp(\Psi(\bX_i,\bX_j))\tilde\bw^T\bX_j > 0.
\]
 Using the above expression for $\Psi$, we have that
\begin{equation}\label{eq:gat_conv_lb}
\begin{split}
    &\sum_{j \in N_i}\exp(\Psi(\bX_i,\bX_j))\tilde\bw^T\bX_j\\
    =~&\exp(\Psi(\bX_i,\bX_i))\tilde\bw^T\bX_i + \sum_{\substack{j \in N_i\\j\neq i}}\exp(\Psi(\bX_i,\bX_j))\tilde\bw^T\bX_j\\
    \ge~&\exp(\Psi(\bX_i,\bX_i))\tilde\bw^T\bX_i - (n-1)\exp(\Psi(\bX_i,-\bX_i))(\kappa\sigma+10\sigma\sqrt{\log n})\\
    \ge~&e^{2R(1-\beta)\epsilon\sigma\max\{1,\kappa\}}  \epsilon\sigma\max\{1,\kappa\} - (n-1)e^{-2R(1-\beta)\epsilon\sigma\max\{1,\kappa\}}(\kappa\sigma+10\sigma\sqrt{\log n}),
\end{split}
\end{equation}
where the second last inequality follows from the event that $|\tilde\bw^T\bX_j| \le \kappa\sigma + 10\sigma\sqrt{\log n}$ for all $j$, and $|N_i| \le n$, and the last inequality follows from the fact that the function
\[
    f(x) = xe^{2R(1-\beta)x} - (n-1)(\kappa\sigma+10\sigma\sqrt{\log n})e^{-2R(1-\beta)x}
\]
is increasing with respect to $x$ for $x\ge0$. To see that $h_i' = \sum_{j \in N_i}\exp(\Psi(\bX_i,\bX_j))\tilde\bw^T\bX_j > 0$, consider the following cases separately.
\begin{itemize}
\item If $\kappa \le 2\sqrt{\log n}$, recall that we picked $\epsilon = 1/(n \log n)$, then because $R = \Omega(n \log^2n/\sigma)$, for large enough $n$ we have that 
\begin{align*}
    &4R(1-\beta)\sigma > \frac{n\sqrt{\log n}}{\max\{1,\kappa\}}\left(\log n +  \log\left(\frac{n\sqrt{\log n}}{\max\{1,\kappa\}}\right) + \log(\kappa + 10\sqrt{\log n})\right)\\
    \iff & 4R(1-\beta)\sigma > \frac{1}{\epsilon\max\{1,\kappa\}}\left(\log n +\log\left(\frac{1}{\epsilon\max\{1,\kappa\}}\right) + \log(\kappa + 10\sqrt{\log n})\right)\\
    \iff & e^{4R(1-\beta)\sigma\epsilon\max\{1,\kappa\}} > \frac{n(\kappa + 10\sqrt{\log n})}{\epsilon\max\{1,\kappa\}}\\
    \iff & e^{2R(1-\beta)\sigma\epsilon\max\{1,\kappa\}} \sigma\epsilon\max\{1,\kappa\}> ne^{-2R(1-\beta)\sigma\epsilon\max\{1,\kappa\}}(\kappa\sigma + 10\sigma\sqrt{\log n}).
\end{align*}
By \eqref{eq:gat_conv_lb} this means that $h_i' > 0$.

\item If $\kappa > 2\sqrt{\log n}$, then $\epsilon = 1/4$. Because $R = \Omega(n \log^2n/\sigma)$, for large enough $n$ we have
\begin{align*}
    &4R(1-\beta)\sigma > \frac{\log n + \log(\kappa+10\sqrt{\log n}) - \log(\kappa/4)}{\kappa/4}\\
    \iff &4R(1-\beta)\sigma  > \frac{\log n + \log(\kappa+10\sqrt{\log n}) - \log(\epsilon\kappa)}{\epsilon\kappa}\\
    \iff &e^{4R(1-\beta)\sigma\epsilon\kappa} > \frac{n(\kappa+10\sqrt{\log n})}{\epsilon\kappa}\\
    \iff &e^{2R(1-\beta)\sigma\epsilon\kappa} \sigma\epsilon\kappa > n(\kappa\sigma+10\sigma\sqrt{\log n})
\end{align*}
By \eqref{eq:gat_conv_lb} this means that $h_i' > 0$.
\end{itemize}
This shows that $h_i' > 0$ whenever $\tilde\bw^T\bX_i > 0$. Similarly, one easily gets that $h_i' < 0$ whenever $\tilde\bw^T\bX_i < 0$. Therefore the proof is complete.

\subsection{Proof of Corollary~\ref{cor:gat_linear_recovery}}\label{subsec:cor15proof}

We restate Corollary~\ref{cor:gat_linear_recovery} for convenience. Recall that we write $\|\bmu\| = \kappa\sigma$ for some $\kappa > 0$.
\begin{corollary*}
With probability at least $1-o(1)$ over the data $(\bX,\bA) \sim \CSBM(n,p,q,\bmu,\sigma^2)$, using the two-layer MLP attention architecture $\Psi$ given in \eqref{eq:psi_ansatz} and \eqref{eq:psi_ansatz_parameters} with $R = \Omega(n \log^2n\sigma)$, one has that
\begin{itemize}
    \item (Perfect classification) If $\kappa \ge \sqrt{2\log n}$ then all nodes are correctly classified;
    \item (Almost perfect classification) If $\kappa = \omega(1)$ then at least $1-o(1)$ fraction of all nodes are correctly classified;
    \item (Partial classification) If $\kappa = O(1)$ then at least $\Phi(\kappa)-o(1)$ fraction of all nodes are correctly classified.
\end{itemize}
\end{corollary*}

Fix $i\in [n]$ and write $\tilde\bw^T\bX_i=(2\eps_i-1)\|\bmu\|+\sigma/\|\bmu\| \cdot \bmu^T\bg_i$ where $\bg_i\sim N(0,\bI)$, where we recall that $\epsilon_i\in\{0,1\}$ defines the class membership of node $i$. We have that $\tilde\bw^T\bX_i$ follows a normal distribution with mean $(2\eps_i-1)\kappa\sigma$ and standard deviation $\sigma$. Part 1 of the corollary is exactly Proposition~\ref{prop:linear_easy} whose proof is given in the main text. We consider the other cases for $\kappa$. For both classes $\epsilon_i = 0$ and $\epsilon_i = 1$, the probability of correct classification is $\Phi(\kappa)$ (using 0 as the decision boundary). Therefore, by applying additive Chernoff bound, we have that
\begin{align*}
    &\Prx\bigg[\mbox{At most $n\Phi(\kappa)-\sqrt{n \log n}$ nodes are correctly classified}\bigg]\\
    =~&\Prx\left[\sum_{i\in [n]}\indi_{\{\text{node $i$ is correctly classified}\}}\le n\Phi(\kappa)-\sqrt{n \log n}\right]\le \frac{2}{n}.
\end{align*}
The proof is complete by noticing that $n\Phi(\kappa)-\sqrt{n \log n} = n(1-o(1))$ when $\kappa = \omega(1)$ and $n\Phi(\kappa)-\sqrt{n \log n} = n(\Phi(\kappa)-o(1))$ when $\kappa = O(1)$.

\bibliography{references}

\begin{thebibliography}{10}

\bibitem{Abbe2018}
E.~Abbe.
\newblock Community detection and stochastic block models: Recent developments.
\newblock {\em Journal of Machine Learning Research}, 18:1--86, 2018.

\bibitem{adler2007random}
Robert~J Adler, Jonathan~E Taylor, et~al.
\newblock {\em Random fields and geometry}, volume~80.
\newblock Springer, 2007.

\bibitem{anderson1962introduction}
T.W. Anderson.
\newblock {\em An introduction to multivariate statistical analysis}.
\newblock John Wiley \& Sons, 2003.

\bibitem{AFTPPVLLQS}
A.~Athreya, D.~E. Fishkind, M.~Tang, C.~E. Priebe, Y.~Park, J.~T. Vogelstein,
  K.~Levin, V.~Lyzinski, Y.~Qin, and D.~L. Sussman.
\newblock Statistical inference on random dot product graphs: A survey.
\newblock {\em Journal of Machine Learning Research}, 18(226):1--92, 2018.

\bibitem{AT16}
J.~Atwood and D.~Towsley.
\newblock Diffusion-convolutional neural networks.
\newblock In {\em Advances in Neural Information Processing Systems (NeurIPS)},
  page 2001–2009, 2016.

\bibitem{BCB15}
D.~Bahdanau, K.~H. Cho, , and Y.~Bengio.
\newblock Neural machine translation by jointly learning to align and
  translate.
\newblock In {\em International Conference on Learning Representations (ICLR)},
  2015.

\bibitem{BFJ2021}
A.~Baranwal, K.~Fountoulakis, and A.~Jagannath.
\newblock Graph convolution for semi-supervised classification: Improved linear
  separability and out-of-distribution generalization.
\newblock In {\em Proceedings of the 38th International Conference on Machine
  Learning (ICML)}, volume 139, pages 684--693, 2021.

\bibitem{BVR17}
N.~Binkiewicz, J.~T. Vogelstein, and K.~Rohe.
\newblock Covariate-assisted spectral clustering.
\newblock {\em Biometrika}, 104:361--377, 2017.

\bibitem{bodnar2022neural}
Cristian Bodnar, Francesco Di~Giovanni, Benjamin Chamberlain, Pietro Lio, and
  Michael Bronstein.
\newblock Neural sheaf diffusion: A topological perspective on heterophily and
  oversmoothing in gnns.
\newblock {\em Advances in Neural Information Processing Systems (NeurIPS)},
  35:18527--18541, 2022.

\bibitem{BL18}
X.~Bresson and T.~Laurent.
\newblock Residual gated graph convnets.
\newblock In {\em arXiv:1711.07553}, 2018.

\bibitem{BAY21}
S.~Brody, U.~Alon, and E.~Yahav.
\newblock How attentive are graph attention networks.
\newblock In {\em International Conference on Learning Representations (ICLR)},
  2022.

\bibitem{bronstein2021geometric}
Michael~M Bronstein, Joan Bruna, Taco Cohen, and Petar Veli{\v{c}}kovi{\'c}.
\newblock Geometric deep learning: Grids, groups, graphs, geodesics, and
  gauges.
\newblock {\em arXiv preprint arXiv:2104.13478}, 2021.

\bibitem{BZSL14}
J.~Bruna, W.~Zaremba, A.~Szlam, and Y.~LeCun.
\newblock Spectral networks and locally connected networks on graphs.
\newblock In {\em International Conference on Learning Representations (ICLR)},
  2014.

\bibitem{CLB19}
Z.~Chen, L.~Li, and J.~Bruna.
\newblock Supervised community detection with line graph neural networks.
\newblock In {\em International Conference on Learning Representations (ICLR)},
  2019.

\bibitem{Chien:2020:joint}
E.~Chien, J.~Peng, P.~Li, and O.~Milenkovic.
\newblock Adaptive universal generalized pagerank graph neural network.
\newblock In {\em International Conference on Learning Representations (ICLR)},
  2021.

\bibitem{DBV16}
M.~Defferrard, X.~Bresson, and P.~Vandergheynst.
\newblock Convolutional neural networks on graphs with fast localized spectral
  filtering.
\newblock In {\em Advances in Neural Information Processing Systems (NeurIPS)},
  page 3844–3852, 2016.

\bibitem{DSM18}
Y.~Deshpande, A.~Montanari S.~Sen, and E.~Mossel.
\newblock Contextual stochastic block models.
\newblock In {\em Advances in Neural Information Processing Systems (NeurIPS)},
  2018.

\bibitem{DMAGHAA15}
D.~Duvenaud, D.~Maclaurin, J.~Aguilera-Iparraguirre, R.~G\'{o}mez-Bombarelli,
  T.~Hirzel, A.~Aspuru-Guzik, and R.~P. Adams.
\newblock Convolutional networks on graphs for learning molecular fingerprints.
\newblock In {\em Advances in Neural Information Processing Systems (NeurIPS)},
  volume~45, page 2224–2232, 2015.

\bibitem{FL2019}
M.~Fey and J.~E. Lenssen.
\newblock Fast graph representation learning with {PyTorch Geometric}.
\newblock In {\em ICLR Workshop on Representation Learning on Graphs and
  Manifolds}, 2019.

\bibitem{GJJ20}
V.~Garg, S.~Jegelka, and T.~Jaakkola.
\newblock Generalization and representational limits of graph neural networks.
\newblock In {\em Advances in Neural Information Processing Systems (NeurIPS)},
  volume 119, pages 3419--3430, 2020.

\bibitem{GMS05}
M.~Gori, G.~Monfardini, and F.~Scarselli.
\newblock A new model for learning in graph domains.
\newblock In {\em IEEE International Joint Conference on Neural Networks
  (IJCNN)}, 2005.

\bibitem{HYL17}
W.~L. Hamilton, R.~Ying, and J.~Leskovec.
\newblock Inductive representation learning on large graphs.
\newblock In {\em Advances in Neural Information Processing Systems (NeurIPS)},
  pages 1025--1035, 2017.

\bibitem{HBL15}
M.~Henaff, J.~Bruna, and Y.~LeCun.
\newblock Deep convolutional networks on graph-structured data.
\newblock In {\em arXiv:1506.05163}, 2015.

\bibitem{hou2019measuring}
Y.~Hou, J.~Zhang, J.~Cheng, K.~Ma, R.~T.~B. Ma, H.~Chen, and M.-C. Yang.
\newblock Measuring and improving the use of graph information in graph neural
  networks.
\newblock In {\em International Conference on Learning Representations (ICLR)},
  2019.

\bibitem{HFZDRLCL20}
W.~Hu, M.~Fey, M.~Zitnik, Y.~Dong, H.~Ren, B.~Liu, M.~Catasta, and J.~Leskovec.
\newblock Open graph benchmark: datasets for machine learning on graphs.
\newblock In {\em Advances in Neural Information Processing Systems (NeurIPS)},
  2020.

\bibitem{J22}
S.~Jegelka.
\newblock Theory of graph neural networks: Representation and learning.
\newblock In {\em arXiv:2204.07697}, 2022.

\bibitem{KBV21}
N.~Keriven, A.~Bietti, and S.~Vaiter.
\newblock On the universality of graph neural networks on large random graphs.
\newblock In {\em Advances in Neural Information Processing Systems (NeurIPS)},
  2021.

\bibitem{kipf:gcn}
T.~N. Kipf and M.~Welling.
\newblock Semi-supervised classification with graph convolutional networks.
\newblock In {\em International Conference on Learning Representations (ICLR)},
  2017.

\bibitem{KTA19}
B.~Knyazev, G.~W. Taylor, and M.~Amer.
\newblock Understanding attention and generalization in graph neural networks.
\newblock In {\em Advances in Neural Information Processing Systems (NeurIPS)},
  pages 4202--4212, 2019.

\bibitem{LRKAK19}
B.~J. Lee, R.~A. Rossi, S.~Kim, K.~N. Ahmed, and E.~Koh.
\newblock Attention models in graphs: A survey.
\newblock {\em ACM Transactions on Knowledge Discovery from Data (TKDD)}, 2019.

\bibitem{levie2018cayleynets}
Ron Levie, Federico Monti, Xavier Bresson, and Michael~M Bronstein.
\newblock Cayleynets: Graph convolutional neural networks with complex rational
  spectral filters.
\newblock {\em IEEE Transactions on Signal Processing}, 67(1):97--109, 2018.

\bibitem{LZBT16}
Y.~Li, R.~Zemel, M.~Brockschmidt, and D.~Tarlow.
\newblock Gated graph sequence neural networks.
\newblock In {\em International Conference on Learning Representations (ICLR)},
  2016.

\bibitem{lim2021large}
Derek Lim, Felix Hohne, Xiuyu Li, Sijia~Linda Huang, Vaishnavi Gupta, Omkar
  Bhalerao, and Ser~Nam Lim.
\newblock Large scale learning on non-homophilous graphs: New benchmarks and
  strong simple methods.
\newblock In {\em Advances in Neural Information Processing Systems (NeurIPS)},
  volume~34, pages 20887--20902, 2021.

\bibitem{A2020}
A.~Loukas.
\newblock How hard is to distinguish graphs with graph neural networks?
\newblock In {\em Advances in Neural Information Processing Systems (NeurIPS)},
  2020.

\bibitem{ALoukas2020}
A.~Loukas.
\newblock What graph neural networks cannot learn: Depth vs width.
\newblock In {\em International Conference on Learning Representations (ICLR)},
  2020.

\bibitem{luan2022revisiting}
Sitao Luan, Chenqing Hua, Qincheng Lu, Jiaqi Zhu, Mingde Zhao, Shuyuan Zhang,
  Xiao-Wen Chang, and Doina Precup.
\newblock Revisiting heterophily for graph neural networks.
\newblock In {\em Advances in Neural Information Processing Systems (NeurIPS)},
  volume~35, pages 1362--1375, 2022.

\bibitem{MLLK22}
S.~Maskey, R.~Levie, Y.~Lee, and G.~Kutyniok.
\newblock Generalization analysis of message passing neural networks on large
  random graphs.
\newblock In {\em Advances in Neural Information Processing Systems (NeurIPS)},
  2022.

\bibitem{minsky1969perceptron}
M.~Minsky and S.~Papert.
\newblock Perceptron: an introduction to computational geometry, 1969.

\bibitem{moore2017csphysics}
C.~Moore.
\newblock The computer science and physics of community detection: Landscapes,
  phase transitions, and hardness.
\newblock {\em Bulletin of The European Association for Theoretical Computer
  Science}, 1(121), 2017.

\bibitem{graphworld2022}
John Palowitch, Anton Tsitsulin, Brandon Mayer, and Bryan Perozzi.
\newblock Graphworld: Fake graphs bring real insights for gnns.
\newblock In {\em ACM SIGKDD Conference on Knowledge Discovery and Data Mining
  (KDD)}, pages 3691--3701, 2022.

\bibitem{PBL20}
O.~Puny, H.~Ben-Hamu, and Y.~Lipman.
\newblock Global attention improves graph networks generalization.
\newblock In {\em arXiv:2006.07846}, 2020.

\bibitem{scarselli:gnn}
F.~Scarselli, M.~Gori, A.~C. Tsoi, M.~Hagenbuchner, and G.~Monfardini.
\newblock The graph neural network model.
\newblock {\em IEEE Transactions on Neural Networks}, 20(1), 2009.

\bibitem{VSPUJGKP17}
A.~Vaswani, N.~Shazeer, N.~Parmar, J.~Uszkoreit, L.~Jones, A.~N. Gomez,
  L.~Kaiser, and I.~Polosukhin.
\newblock Attention is all you need.
\newblock In {\em Advances in Neural Information Processing Systems (NeurIPS)},
  page 6000–6010, 2017.

\bibitem{Velickovic2018GraphAN}
P.~Velickovic, G.~Cucurull, A.~Casanova, A.~Romero, P.~Li{\`o}, and Y.~Bengio.
\newblock Graph attention networks.
\newblock In {\em International Conference on Learning Representations (ICLR)},
  2018.

\bibitem{vershynin2018high}
R.~Vershynin.
\newblock {\em High-dimensional probability: An introduction with applications
  in data science}, volume~47.
\newblock Cambridge university press, 2018.

\bibitem{wang2019dgl}
M.~Wang, D.~Zheng, Z.~Ye, Q.~Gan, M.~Li, X.~Song, J.~Zhou, C.~Ma, L.~Yu,
  Y.~Gai, T.~Xiao, T.~He, G.~Karypis, J.~Li, and Z.~Zhang.
\newblock Deep graph library: A graph-centric, highly-performant package for
  graph neural networks.
\newblock {\em arXiv preprint arXiv:1909.01315}, 2019.

\bibitem{XHLJ19}
K.~Xu, W.~Hu, J.~Leskovec, and S.~Jegelka.
\newblock How powerful are graph neural networks?
\newblock In {\em International Conference on Learning Representations (ICLR)},
  2019.

\bibitem{yan2022two}
Yujun Yan, Milad Hashemi, Kevin Swersky, Yaoqing Yang, and Danai Koutra.
\newblock Two sides of the same coin: Heterophily and oversmoothing in graph
  convolutional neural networks.
\newblock In {\em IEEE International Conference on Data Mining (ICDM)}, pages
  1287--1292. IEEE, 2022.

\bibitem{Zhu:2020:generalizing}
J.~Zhu, Y.~Yan, L.~Zhao, M.~Heimann, L.~Akoglu, and D.~Koutra.
\newblock Beyond homophily in graph neural networks: Current limitations and
  effective designs.
\newblock In {\em Advances in Neural Information Processing Systems (NeurIPS)},
  2020.

\end{thebibliography}
\bibliographystyle{plain}

\end{document}